\pdfoutput=1
\documentclass[11pt]{article}
\usepackage[letterpaper,margin=1in]{geometry}
\usepackage[normalem]{ulem}
\usepackage{newtxtext}
\usepackage{titling}
\usepackage[numbers,sort&compress]{natbib}
\usepackage{wrapfig}
\usepackage[utf8]{inputenc}
\usepackage[T1]{fontenc}
\usepackage{url}
\usepackage{algorithm}
\usepackage{algorithmic}
\usepackage{booktabs}
\usepackage{amsfonts}
\usepackage{nicefrac}
\usepackage{microtype}
\usepackage[table]{xcolor}
\definecolor{authorcolor}{HTML}{3559A7}
\definecolor{affiliationcolor}{HTML}{C21875}
\usepackage{amsmath,amssymb,amsthm,mathtools,bm}
\usepackage{graphicx}
\usepackage{capt-of}
\usepackage{placeins}
\usepackage{float}
\usepackage{multirow,array}
\usepackage{xspace}
\usepackage{bigdelim}
\usepackage[
  colorlinks=true,
  citecolor=blue,
  linkcolor=black,
  urlcolor=blue
]{hyperref}
\hypersetup{
  pdftitle={Provable Pruning for Efficient 3D Gaussian Splatting via Coresets},
  pdfauthor={Waseem Mousa and Alaa Maalouf},
  pdfsubject={Provable pruning for efficient 3D Gaussian Splatting via coresets},
  pdfkeywords={3D Gaussian Splatting, coresets, pruning, compression, novel-view synthesis}
}
\usepackage[nameinlink,noabbrev]{cleveref}

\newtheorem{theorem}{Theorem}
\newtheorem{lemma}{Lemma}
\newtheorem{proposition}{Proposition}
\newtheorem{corollary}{Corollary}
\theoremstyle{definition}
\newtheorem{definition}{Definition}
\newtheorem{assumption}{Assumption}
\newtheorem{remark}{Remark}

\crefname{assumption}{Assumption}{Assumptions}
\Crefname{assumption}{Assumption}{Assumptions}

\AddToHook{env/theorem/begin}{\crefalias{section}{theorem}}
\AddToHook{env/lemma/begin}{\crefalias{section}{lemma}}
\AddToHook{env/proposition/begin}{\crefalias{section}{proposition}}
\AddToHook{env/corollary/begin}{\crefalias{section}{corollary}}
\AddToHook{env/definition/begin}{\crefalias{section}{definition}}
\AddToHook{env/assumption/begin}{\crefalias{section}{assumption}}
\AddToHook{env/remark/begin}{\crefalias{section}{remark}}

\newcommand{\RR}{\mathbb{R}}
\newcommand{\cQ}{\mathcal{Q}}
\newcommand{\cQf}{\mathcal{Q}'}
\newcommand{\eps}{\varepsilon}
\newcommand{\Abs}[1]{\left|#1\right|}

\newcommand{\method}{our method\xspace}

\newcommand{\projectrepo}{waseem-m/3dgs_provable_coresets}
\newcommand{\projecturl}{https://github.com/\projectrepo}
\newcommand{\projectlink}{\href{\projecturl}{\textcolor{affiliationcolor}{\nolinkurl{github.com/\projectrepo}}}}

\title{\bfseries Provable Pruning for Efficient 3D Gaussian Splatting via Coresets}
\author{%
\textcolor{authorcolor}{Waseem Mousa$^{1}$ \qquad Alaa Maalouf$^{1}$}\\[0.45em]
\textcolor{affiliationcolor}{\normalsize $^{1}$Department of Computer Science, University of Haifa}\\[0.35em]
{\small\textbf{Project page and open-source code:} \projectlink}
}
\date{}

\begin{document}

\maketitle
\thispagestyle{plain}

\begin{abstract}
3D Gaussian Splatting (3DGS) enables high-quality real-time novel-view synthesis, but practical scenes often contain millions of Gaussians, making compression essential for deployment on limited hardware. Existing reduction methods are effective but mostly heuristic: they provide no multiplicative approximation guarantee for the rendered objective, and thus rely heavily on costly post-pruning finetuning to recover quality.
We ask a basic question: \emph{can a 3DGS scene be provably replaced by a much smaller weighted subset (coreset) while preserving the objective of interest?} 
We first show that, in the unrestricted setting, no non-trivial multiplicative 3DGS coreset exists. 
We then show that multiplicative guarantees are not impossible, but resolution-dependent. For a prescribed rendering resolution, such as representative views or grids of views/rays, we provide the first weighted coreset construction theorem for 3DGS. The construction samples Gaussians by sensitivity: provable importance scores measuring each Gaussian’s role in the full-scene objective. Finally, under explicit validity and log-transmittance stability assumptions, we turn this objective guarantee into a rendering guarantee.
Empirically, our method is strongest where deployment needs it most: aggressive compression with no or minimal recovery compute. In prune-only and very short finetuning regimes, it achieves state-of-the-art performance, showing that principled importance estimation can be both theoretically meaningful and practically useful.
Open-source code is available at \projectlink.
\end{abstract}

\begin{figure}[H]
    \centering
    \includegraphics[width=0.99\textwidth]{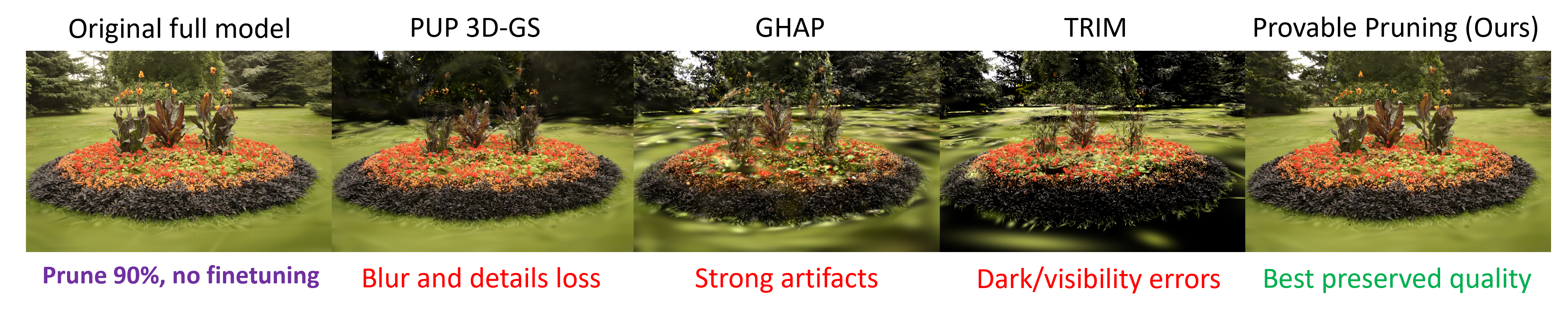}
    \caption{Our method prunes a pretrained 3D Gaussian scene while provably preserving the quality on a set of queries. In the regime where no finetuning is used, our approach preserves rendering quality better than competing pruning heuristics, especially at aggressive pruning ratios. }
    \label{fig:teaser}
\end{figure}

\section{Introduction}
 
3D Gaussian Splatting (3DGS) has become a central representation for real-time novel-view synthesis, combining explicit Gaussian scene models with visibility-aware differentiable splatting and strong fidelity--speed tradeoffs \citep{kerbl2023,yu2023mipsplatting,lu2023scaffoldgs,huang2024twodgs,liang2024analytic,kheradmand2024mcmc,bao2024survey,bagdasarian2025survey}. Yet, high-quality reconstructions often require large Gaussian sets, creating significant memory, storage, bandwidth, and runtime costs, especially on resource-limited hardware. This has motivated growing work on 3DGS pruning and global compaction \citep{fang2024minisplatting,mallick2024taming3dgs,fan2023lightgaussian,girish2023eagles,lee2023compact3dgrf,xie2024mesongs,ali2024trimming,hanson2025pup,wang2025ghap,zhang2025gaussianspa,chen2025pcgs,lee2025omg,wang2025nsvq,ali2025compressionsurvey}.

\textbf{The gap.} While these methods demonstrate that strong compression is often possible, they leave two key bottlenecks. First, the compressed Gaussian set is often followed by a costly recovery stage, requiring substantial hardware resources and many fine-tuning iterations. Second, existing approaches are largely heuristic, providing no provable guarantees on approximation error or subset size.

\textbf{Our approach: Coresets for 3DGS. } To this end, we study 3DGS pruning through the lens of coresets: small weighted subsets that approximately preserve a target objective over a prescribed family of queries. This leads to our central question: \emph{\uline{Can a full 3DGS scene be provably replaced by a much smaller weighted subset of the gaussian while preserving the target rendering objective, and under what assumptions is such a guarantee possible?}} This question is especially important in low-compute settings. When only a prune-only evaluation or a very short recovery phase is feasible, the selected subset directly determines performance. A principled compression method should therefore identify which Gaussians are most important for preserving the rendered objective, while providing explicit control over the approximation error, not requiring intense fine-tuning steps.

\textbf{The challenge. }While coreset construction frameworks provide such guarantees for many objectives~\citep{feldman2011,braverman2016,phillips2016,feldman2020survey}, including subset-selection approaches for neural-network compression \citep{baykal2019,Tukan2022provable,liebenwein2020}, the main obstacle in 3DGS is visibility: front-to-back compositing couples Gaussians through transmittance, so changing or removing one early Gaussian can affect the contribution of many later Gaussians. We therefore ask whether 3DGS admits any non-trivial multiplicative coreset guarantee at all, and if not in full generality, what restrictions recover a useful theorem in practice.

\textbf{A glimpse into our findings.} Our answer is twofold: 
In the unrestricted setting, no general multiplicative coreset theorem is possible: if the query family is rich enough to isolate individual Gaussians, then every Gaussian can become indispensable. 
However, this unrestricted view is stronger than what is needed for rendering. In practice, approximation is required only over a finite set of queries (views) induced by the desired rendering resolution, e.g., desired grid/set of views or rays. Under this finite-query formulation, together with a fixed full-scene objective, frozen itemwise contributions, and stable transmittance/occlusion, we prove a weighted coreset construction theorem. 
The resulting Gaussian subset size scales only logarithmically with the number of representative queries/resolution. 
The construction is based on assigning each Gaussian a sensitivity scores, which quantify the importance of it with respect to the objective and guide the sampling procedure. Thus, our main theoretical message is: \emph{non-trivial 3DGS coresets are impossible in full generality, but provable pruning is possible under practically meaningful restrictions.}

\noindent\textbf{Our Contribution.}
This work develops a theory-guided framework for 3DGS compression/pruning and evaluates it in low-compute deployment regimes. Our main contributions are:
\begin{enumerate}
    \item \textbf{Impossibility in full generality.}
    We prove that no non-trivial multiplicative coreset guarantee can hold for unrestricted 3DGS rendering: when the query family is rich enough to isolate individual Gaussians, every Gaussian may become necessary; see \cref{thm:no-general-main}.

    \item \textbf{Resolution-dependent provable 3DGS pruning (coreset).}
    First, for a desired rendering resolution (a given set or a grid of views/rays), and a fixed full-scene reference decomposition, we provide the first provable pruning (coreset construction) algorithm for 3DGS; the remaining Gaussian subset size scales logarithmically by the size of the desired views set (resolution); see \cref{sec:finite-th-main,thm:finite-th-main}.
    Then, we prove that transfer from a fixed-objective guarantee for true 3DGS re-rendering is possible under explicit validity and log-transmittance stability assumptions; see \cref{sec:rendered-main,thm:rendered-main}.

    \item \textbf{Extension beyond finite resolutions (query sets).}
    We extend the representative-query theorem from finite query families to compact query regions using representative covers and Lipschitz control; see \cref{sec:compact-region-main,thm:compact-region-main}.

    \item \textbf{Open-source implementation.}
    We release the sensitivity computation, coreset construction, pruning, and evaluation pipeline at \projectlink.

\end{enumerate}
Extensive experiments indicate that our theoretically inspired method delivers strong performance across the evaluated settings, with particularly clear gains in prune-only/no-finetuning evaluation and very short finetuning budgets, where costly recovery is not feasible; see \cref{sec:practical-extensions,sec:experiments}.

\section{Related Work}

\noindent\textbf{3D Gaussian Splatting.}
3DGS introduced an explicit radiance-field representation built from anisotropic Gaussians, interleaved optimization and density control, and visibility-aware splatting, enabling fast training and real-time rendering \citep{kerbl2023}.
A large follow-up literature refined this framework in several directions, including anti-aliasing and scale robustness \citep{yu2023mipsplatting,liang2024analytic}, structured or view-adaptive parameterizations \citep{lu2023scaffoldgs}, stronger geometric fidelity \citep{huang2024twodgs}, and alternative optimization or densification dynamics \citep{kheradmand2024mcmc,zhang2024pixel,ye2024absgs}.
Broader surveys document how quickly 3DGS has expanded across optimization, systems, and applications \citep{bao2024survey,bagdasarian2025survey}.
For our purposes, the key point is that once 3DGS became a strong rendering primitive, Gaussian-set size and cost became central research problems \citep{fang2024minisplatting,mallick2024taming3dgs}.

\noindent\textbf{3DGS compaction and pruning.} A growing literature targets the memory, storage, and runtime costs of large 3D scenes, including budgeted training or densification methods~\citep{fang2024minisplatting,mallick2024taming3dgs}. Others reduce model size through pruning, compact parameterizations, or global reduction~\citep{fan2023lightgaussian,girish2023eagles,lee2023compact3dgrf,xie2024mesongs,ali2024trimming,hanson2025pup,wang2025ghap,zhang2025gaussianspa}. A third group emphasizes coding and storage efficiency through progressive or quantized representations~\citep{chen2025pcgs,lee2025omg,wang2025nsvq}. Compression surveys now provide a broader taxonomy of this landscape \citep{bagdasarian2025survey,ali2025compressionsurvey}. %
Pruning~\citep{ali2024trimming,hanson2025pup} is especially attractive because it is post hoc and representation-preserving: it outputs a smaller standard 3DGS scene that can be rendered by existing pipelines, without specialized runtime libraries or separate decoding, while directly reducing memory and rendering cost.

\noindent\textbf{Coresets and provable pruning.} 
Coresets are small weighted subsets that approximate an objective over a query family~\citep{agarwal2005geometric,badoiu2003smaller,kumar2005minimum,maalouf2019fast,maalouf2023autocoreset}. 
Coresets have been applied to many problems, including regression, subspace approximation, $k$-means, neural-network training, test-time active fine tuning, and domain adaptations~\citep{maalouf2020tight,bachemkmeancset,bachem2017practical,jubran2020sets,tukan2023provable,sener2018active,mirzasoleiman2020coresets,yang2023sustainable,rosman2017hybrid,jubran2021overview,hulkund2025datas,khamis2026efficient}. 
Sensitivity sampling and broader coreset frameworks are standard routes to obtaining such coresets~\citep{feldman2011,maalouf2021unified,phillips2016,tukan2020coresets}. Specifically, provable subset-selection results were developed for neural-network pruning, including data-dependent and data-independent neuron pruning~\citep{baykal2019,mussay2020,liebenwein2020,Tukan2022provable}. 
These results are closest in spirit to our goal, but they assume objectives with a fixed itemwise decomposition once the input is fixed.

\section{3DGS Pruning via Coresets}
\label{sec:problem-setup}

\textbf{3DGS scene. } 3DGS represents a scene by a set of Gaussians. Let $N>0$ be an integer, a 3DGS scene is a set $G=\{g_1,\dots,g_N\}$ of $N$ Gaussians, where for every $i\in[N]$: $g_i=(\mu_i,\Sigma_i,\alpha_i,c_i)$ such that $\mu_i\in\RR^3$ is the mean, $\Sigma_i\in\RR^{3\times 3}$ is the positive definite covariance matrix, $\alpha_i\in[0,1]$ is the opacity, and $c_i$ is the view-dependent intensity of a color channel. %

\textbf{Our goal. } For a given query (camera pose and a pixel to render), to render the desired pixel, Gaussians are composited through front-to-back order: each Gaussian contributes according to its
projected footprint, color, and opacity, scaled by the transmittance left by earlier Gaussians.  
We seek a multiplicative approximation guarantee over a family of rendering queries. To mathematically formulate this in a standard relative-error form, each query (we wish to approximate) specifies the
camera parameters (pose) together with a specific output channel in a pixel to approximate, producing a nonnegative
\emph{scalar} objective for one RGB channel. 
Thus, a guarantee over all queries implies approximation for all selected views, pixels, and channels, and therefore for the full RGB
render.

\textbf{Rendering query.} Formally, let $\Lambda=\{R,G,B\}$ denote the output channels. A rendering query is $q=(C,x,\lambda)\in\cQ$, where $C = (R,t)$ are the camera extrinsics\footnote{For simplicity, we assume fixed camera intrinsics and fixed image domain throughout, and omit them from the notation.} with
$R\in SO(3)$ and $t\in\RR^3$, defining the world-to-camera map $p\mapsto Rp+t$, $x$ is an image-plane location (a pixel), and $\lambda\in\Lambda$ is an output channel. Thus, a query $q$ corresponds to one color channel of one pixel from a camera pose.

\textbf{Scalar rendering function.} For $g_i\in G$, let $k(g_i,q)$ be the projected Gaussian-kernel value of Gaussian $g_i$ at pixel $x$ under the camera pose $C$, let $\rho(g_i,q):=\alpha_i k(g_i,q)$ be its attenuation factor, and let $\prec_{G,q}$ denote the front-to-back order induced by the view depths of the Gaussians in $G$ under pose $C$. The full-scene prefix transmittance of Gaussian $g_i$ is $T(G,g_i,q):=\prod_{j\prec_{G,q} i}(1-\rho(g_j,q))$, and its full-scene contribution is $a(G,g_i,q):=c_i(q)\rho(g_i,q)T(G,g_i,q)$. The resulting per-channel scalar rendered value of scene $G$ is $A(G,q):=\sum_{g_i \in G} a(G,g_i,q)$. A note is given here\footnote{Note that the quantities $k(g_i,q)$, $\rho(g_i,q)$, $T(G,g_i,q)$, and the ordering $\prec_{G,q}$, depend only on the camera extrinsics $C$, the pixel $x$, and the Gaussian model G, and are invariant to the output channel $\lambda$.}.

\textbf{A pruned 3DGS scene.} A weighted pruned scene is represented by a vector $w\in[0,\infty)^N$, where $w_i=0$ removes Gaussian $g_i$ and $w_i>0$ reweights it. $\|w\|_0$ is the number of nonzero entries of $w$ (number of retained Gaussians). Its true reduced-scene render is $A_w(G,q):=\sum_{g_i \in G} w_i c_i(q)\rho(g_i,q)T_w(G,g_i,q)$, where $T_w(G,g_i,q):=\prod_{j \prec_{G,q} i}(1-w_j\rho(g_j,q))$. Note that, in this paper, a 3DGS coreset is a weighted pruned scene whose induced objective approximates the full-scene objective over a family of queries/views. %
\begin{definition}[Multiplicative $\eps$-coreset for 3DGS]\label{def:eps-coreset-main}
Let $N>0$,  $G=\{g_1,\dots,g_N\}$ be a 3DGS scene, and let $\mathcal U\subseteq \cQ$ be a family of rendering queries. Let $w\in[0,\infty)^N$ represent a weighted pruned scene. 
Recall $A$, the nonnegative scalar objective of the full scene (per-channel scalar rendering function), and let $A_w$ be the corresponding objective induced by $w$. We call $w$ an \emph{$\eps$-coreset} for $G$ on $\mathcal U$ if
$
\Abs{A(G,q)-A_w(G,q)}\le \eps A(G,q)\text{, for every } q\in\mathcal U.
$
\end{definition}
\subsection{The Unrestricted Case: No General 3DGS Coreset}
We first ask whether unrestricted 3DGS rendering admits a non-trivial multiplicative coreset guarantee. %

\begin{theorem}[General-case impossibility]\label{thm:no-general-main}
For every $N\ge 1$ and every $0<\eps<1$, there exist a 3DGS scene $G=\{g_1,\dots,g_N\}$ and a query family $\cQ$ such that no weighted reduced scene specified $w\in[0,\infty)^N$ with $\|w\|_0<N$ is an $\eps$-coreset for the true rendered objective $A_w$ on $\cQ$.
\end{theorem}

\begin{proof}[Proof overview]
We construct a counterexample. Let $G=\{g_1,\dots,g_N\}$ be any scene for which the query family is rich enough to isolate each Gaussian individually: for every $i\in[N]$, there exists a query $q_i\in\cQ$ such that Gaussian $g_i$ is the unique positive contributor at $q_i$. Then $A(G,q_i)=a(G,g_i,q_i)>0$ and all other Gaussians contribute zero at $q_i$. Now let $w\in[0,\infty)^N$ satisfy $\|w\|_0<N$. Some index $i\in[N]$ must satisfy $w_i=0$. At the corresponding isolating query $q_i$, the reduced scene therefore satisfies $A_w(G,q_i)=0$, while the full scene satisfies $A(G,q_i)>0$. Hence $\Abs{A(G,q_i)-A_w(G,q_i)} = A(G,q_i)$, so the relative error at $q_i$ is exactly $1$. Therefore $w$ cannot be an $\eps$-coreset for any $0<\eps<1$. The full proof, including an explicit footprint-isolation construction for finite-support rasterized 3DGS splats, is given in \cref{sec:appendix-no-general}.
\end{proof}
The interpretation: if the query family is rich enough to isolate individual Gaussians, then every Gaussian is indispensable. A coreset theorem must restrict the query family, the objective, or both.
\subsection{Coresets for 3DGS Under Representative Queries}
\label{sec:finite-th-main}
\begin{algorithm}[t]
\caption{Sensitivity-based 3DGS coreset pruning}
\label{alg:sensitivity-3dgs}
\begin{algorithmic}[1]
\REQUIRE Pretrained 3DGS scene $G=\{g_i\}_{i=1}^n$, and a budget (pruned model size) $m$
\ENSURE Pruned weighted 3DGS scene $G'$

\STATE Build finite queries $\cQf$ from representative cameras, pixels/tiles, and RGB channels.
\STATE For each $g_i\in G$ and $q\in \cQf$, compute
$
\mathrm{imp}(g_i,q):=a(G,g_i,q)/A(G,q).
$
\STATE Set
$
s(g_i):=\max_{q\in \cQf}\mathrm{imp}(g_i,q),\quad
S:=\sum_{j=1}^n s(g_j),\quad
p(g_i):=s(g_i)/S.
$
\STATE Sample $m$ Gaussians i.i.d. from $G$ according to $p$; let $N_i$ be the number of times $g_i$ is sampled.
\STATE Assign weight
$
w_i:=N_i/(m\,p(g_i))
$
for every sampled Gaussian; $w=(w_1,\cdots,w_n)$
\STATE \textbf{return} $w$
\end{algorithmic}
\end{algorithm}
The unrestricted impossibility result shows that non-trivial guarantees require structure. We therefore restrict attention to a finite representative query family induced by the target views together with a nonnegative scalar objective and a stability condition on prefix transmittance perturbations. We then define the sampling rule and state the fixed-objective and rendered guarantees.

\textbf{Setting: Finite query set} Let $\cQf\subseteq \cQ$ be a finite family of representative queries: guarantees are required over a controlled set of camera poses and image samples, not over a continuum of all possible cameras and rays. In 3DGS terms, $\cQf$ defines the view-space discretization of the guarantee: its density determines how finely the approximation covers camera translations and rotations, while its image samples define the pixels, tiles, or channels being scored.

\textbf{Our approach: Gaussian sensitivity. } We seek to quantify the importance of each Gaussian.  
Intuitively, a Gaussian may contribute very differently across queries: it may be highly influential for some views, rays, or pixels, and nearly irrelevant for others. We therefore first define the contribution of each Gaussian to each query. We then define its global importance, or sensitivity, as the maximum relative contribution over all queries. 
This identifies Gaussians whose removal could cause large approximation errors, and distinguishes them from Gaussians that are consistently negligible.
First, we disregard queries where $A(G,q) = 0$ since no Gaussian is required for them.
For each Gaussian $g_i \in G$ and each query $q \in \cQf$, we introduce the \emph{importance} of $g_i$  with respect to $q$: $
imp(g_i, q):=\frac{a(G,g_i,q)}{A(G,q)}$, We then define the Gaussian's \emph{sensitivity} score as: \begin{equation}\label{eq:sen}
s(g_i):=\max_{q\in\cQf}imp(g_i,q) = \max_{q\in\cQf} {a(G,g_i,q)}/{A(G,q)},
\end{equation}
and the total sensitivity of the scene is $ S:=\sum_{g_i \in G} s(g_i)$. 
Thus $s(g_i)$ is the largest normalized contribution of Gaussian $i$ over the representative family. %

\noindent\textbf{Pruning pipeline: 3DGS Coreset via Sensitivity sampling.} For every Gaussian $g_i$, define $p(g_i):=s(g_i)/S$. For a target sample size $m$, sample (i.i.d) $m$ Gaussian $G'=\{g'_1,\dots,g'_m\}$  from the scene $G$, where each $g'_i$ is sampled with probability $p(g'_i)$. Define $n_i$ to be the number of times $g_i$ was sampled. The output weight vector $w\in[0,\infty)^N$ is then
$w_i:= \dfrac{n_i}{m p(g_i)}$. Equivalently, each sampled occurrence of Gaussian $g_i$ contributes an inverse-probability weight $(m p(g_i))^{-1}$. The resulting weighted reduced scene is supported on at most $m$ Gaussians and is unbiased for the fixed-objective relaxation on every fixed query. Algorithm~\ref{alg:sensitivity-3dgs} summarizes the pipeline.%

\noindent\textbf{Relaxed objective.} For the coreset theorem, we also introduce the auxiliary fixed-objective relaxation $A_w^{\mathrm{th}}(G, q):=\sum_{g_i \in G} w_i a(G,g_i,q)$. Here each itemwise term $a(G,g_i,q)$ is frozen under the transmittance of the full scene $G$. This is a deliberate theoretical relaxation: once the subset changes, the true rendered objective $A_w$ no longer decomposes into fixed per-Gaussian terms, whereas $A_w^{\mathrm{th}}$ does. The later rendered theorem shows how a guarantee for $A_w^{\mathrm{th}}$ can be transferred back to $A_w$ under explicit transmittance-stability assumptions.

\begin{theorem}[Query size dependent 3DGS coreset]\label{thm:finite-th-main} Let $G$ be a 3DGS scene, and let $\cQf$ be a desired finite set of queries. Let $\eps_c, \delta \in(0, 1)$, and let $w$ be the weight vector produced by the pruning pipeline above (e.g., the output of a call to Alg.~\ref{alg:sensitivity-3dgs} with $G$ and $m$ as inputs).  If
$
m\ge \frac{3S}{\eps_c^2}\log\!\left(\frac{2|\cQf|}{\delta}\right), 
$
then, with probability at least $1-\delta$, $w$ is an $\eps_c$-coreset for the fixed-objective relaxation $A^{\mathrm{th}}$ on $\cQf$, i.e., $\Abs{A(G,q)-A_w^{\mathrm{th}}(G,q)}\le \eps_c A(G,q)$ for all $q\in\cQf$. Moreover, inherently $\|w\|_0\le m$.
\end{theorem}
\begin{proof}[Proof overview] For any fixed $q\in\cQf$, the estimator $A_w^{\mathrm{th}}(G,q)$ is an average of i.i.d. nonnegative random variables with mean $A(G,q)$. The sensitivity bound implies that each sampled term is at most $S A(G,q)$ after normalization. A standard multiplicative concentration bound therefore gives a relative-error guarantee for one query, and a union bound over the finite family $\cQf$ yields the simultaneous statement. The full proof is given in \cref{sec:appendix-finite-proof}.
\end{proof}

The fixed-objective theorem is not yet a theorem about true re-rendering. The reason is structural: once the reduced scene changes, the prefix transmittances change as well, so later Gaussian contributions are no longer fixed functions of the full scene.

\subsection{Transfer to true rendering under mild transmittance stability}\label{sec:rendered-main}

To transfer the fixed-objective guarantee back to true rendering, we add one more ingredient: a stability requirement saying that the reduced scene preserves the relevant prefix transmittances up to controlled multiplicative distortion.

\begin{definition}[Weighted rendering validity]\label{def:valid-main}
A weight vector $w\in[0,\infty)^N$ is called \emph{valid} on the query set $\cQf$ if $0\le w_i\rho(g_i,q)<1$ for every $q\in\cQf$ and $i\in[N]$. This ensures that every factor in $T_w(G,g_i,q)$ is nonnegative, so the reduced-scene transmittances are well-defined.
\end{definition}

\begin{assumption}[Log-transmittance stability]\label{ass:log-stability-main}
There exists $\gamma\ge 0$ such that for every $q\in\cQf$ and $i\in[N]$ with $w_i>0$, one has $\Abs{\log T(G,g_i,q) - \log T_w(G,g_i,q)}\le \gamma$. This assumption captures the regime in which pruning and reweighting do not drastically alter the occlusion and transmittance structure along the representative queries.
\end{assumption}

\begin{theorem}[Finite-query rendered coreset under mild assumptions]\label{thm:rendered-main}
Let $w$ be the weight vector produced by the sensitivity-based sampling rule, and suppose that $w$ is valid on $\cQf$ and satisfies \cref{ass:log-stability-main}. Under the same sample-size condition as in \cref{thm:finite-th-main}, with probability at least $1-\delta$, $w$ is an $\eps_r$-coreset for the true rendered objective $A$ on $\cQf$, where
$
\eps_r:=\max\Bigl\{1-e^{-\gamma}(1-\eps_c),\ e^{\gamma}(1+\eps_c)-1\Bigr\}.$
Equivalently, $\Abs{A(G,q)-A_w(G,q)}\le \eps_r A(G,q)$ for all $q\in\cQf$.
\end{theorem}

\begin{proof}[Proof overview] \Cref{thm:finite-th-main} controls $A_w^{\mathrm{th}}(G,q)$ relative to $A(G,q)$. The log-transmittance stability assumption then implies that, for each retained Gaussian and representative query, the reduced-scene transmittance differs from the full-scene transmittance by at most a factor $e^{\pm\gamma}$. Summing these termwise bounds yields a multiplicative sandwich between $A_w(G,q)$ and $A_w^{\mathrm{th}}(G,q)$, and combining the two estimates gives the rendered guarantee. The full proof is given in \cref{sec:appendix-transfer}.
\end{proof}
\subsection{Compact-region extension}\label{sec:compact-region-main}

The finite-query theorem controls a prescribed representative $\cQf$. It is useful to ask when such a guarantee transfers from finitely many representatives to a continuous region of nearby queries. We state a standard representative-cover extension: if every query in a compact region is close to some representative query, and both the full and reduced rendered objectives vary smoothly inside the region, then the guarantee extends to the whole region with an additional cover-radius error.

\begin{assumption}[Compact nondegenerate query region]\label{ass:compact-region-main}
Let $\cQ_r\subseteq\cQ$ be a compact query region equipped with a metric $d_{\cQ}$ (think of $\cQ_r$ as a local region of camera poses, pixels, and channels around a representative set). On $\cQ_r$, assume:
(i) $\cQ_r$ is compact, (ii) the full rendered objective $q\mapsto A(G,q)$ is Lipschitz on $\cQ_r$ with constant $L_{\mathrm{full}}$, (iii) the reduced rendered objective $q\mapsto A_w(G,q)$ is Lipschitz on $\cQ_r$ with constant $L_w$, and (iv) the full rendered objective is bounded away from zero:
    $
    A(G,q)\ge A_{\min}>0
    \text{, for every }q\in \cQ_r.
    $
\end{assumption}

\begin{definition}[$\rho$-net]\label{def:rho-net-main}
A finite set $\cQf_r\subseteq \cQ_r$ is a $\rho$-net of $\cQ_r$ if for every $q\in \cQ_r$ there exists $q'\in\cQf_r$ such that
$
d_{\cQ}(q,q')\le \rho .
$
\end{definition}

The Lipschitz requirement is a local regularity condition. It is expected to hold on regions where the active set of Gaussians is finite, the front-to-back order does not change discontinuously, and the camera-to-ray map, color functions, and projected kernel values vary smoothly. In other words, the region should not cross a sharp visibility or ordering discontinuity. The appendix gives the formal smoothness-to-Lipschitz argument; see \cref{sec:appendix-region-proof}.

\begin{theorem}[Compact-region extension from representatives]\label{thm:compact-region-main}
Assume \cref{ass:compact-region-main}. Let $\cQf_r\subseteq \cQ_r$ be a finite $\rho$-net. Suppose that the reduced scene satisfies the multiplicative guarantee on the representative set:
$
\Abs{A(G,q')-A_w(G,q')}
\le
\eps_0 A(G,q')
\text{, for every }q'\in\cQf_r .
$
Then for every $q\in \cQ_r$,
$
\Abs{A(G,q)-A_w(G,q)}
\le
\eps_{\mathrm{reg}} A(G,q),
$
where
$
\eps_{\mathrm{reg}}
:=
\eps_0+
\frac{\bigl((1+\eps_0)L_{\mathrm{full}}+L_w\bigr)\rho}{A_{\min}} .
$
\end{theorem}

\begin{proof}[Proof overview]
For any query $q\in \cQ_r$, choose a representative $q'\in\cQf_r$ within distance $\rho$. The error at $q$ is decomposed into three terms: the change of the full render $A(G,\cdot)$ from $q$ to $q'$, the representative-set approximation error at $q'$, and the change of the reduced render $A_w(G,\cdot)$ from $q'$ back to $q$. The two change terms are controlled by the Lipschitz constants, while the middle term is controlled by the representative-set coreset guarantee. The lower bound $A(G,q)\ge A_{\min}$ converts the resulting additive cover-radius term into relative error. The full proof is given in \cref{sec:appendix-region-proof}.
\end{proof}
When \cref{thm:rendered-main} holds on $\cQf_r$, we may take $\cQf=\cQf_r$ and instantiate $\eps_0$ by the finite-query rendered coreset error $\eps_r$.
Thus, \cref{thm:compact-region-main} extends the finite rendered guarantee from a representative query set to a compact region, at the cost of an additional term proportional to the cover radius $\rho$.
\subsection{Sensitivity at Multiple Resolutions}\label{sec:practical-extensions}
The finite-query coreset theorem is stated for per-channel objectives, but the same argument applies to any finite collection of nonnegative objectives obtained by aggregating these channel contributions. In particular, one may define the query space at different rendering resolutions: individual color channels, pixels, tiles, or scenes. Each choice simply changes what is treated as a single query objective before applying the same sensitivity-sampling argument. The Chernoff analysis is unchanged: the sampled random variables now correspond to the chosen aggregated objective rather than to a single color channel, and the resulting guarantee preserves that objective over the finite query family.

\noindent\textbf{Per-pixel construction.}
This variant treats each camera and image-plane location as a query, aggregates the RGB-channel contributions of each Gaussian into a single nonnegative pixel-level objective (luminance or total color contribution); Sensitivities are computed based on this objective. Rather than preserving each channel separately, the coreset preserves the aggregated pixel objective. 

\noindent\textbf{Per-tile and Per-scene construction.}
The per-tile variant further coarsens the query resolution. Each rendered image is partitioned into non-overlapping tiles, for example matching the tile granularity used by the 3DGS rasterization kernel. For each tile, a Gaussian's contribution is accumulated over the pixels and channels inside that tile, and sensitivities are computed with respect to the resulting tile-level objective. This gives a finite-query coreset guarantee for tile-aggregated rendering objectives. The per-scene variant uses the coarsest aggregation level. Here, each Gaussian's contribution is accumulated over the entire representative rendering set: all selected views, pixels, and channels. Sensitivities are therefore computed with respect to the full scene-level objective used for scoring. This is the variant used in our experiments.

\noindent\textbf{Additional variants.} The coreset framework motivates sensitivity-based selection, but the exact score aggregation can be engineered for practical robustness. In addition to the core color-aware $L_1$-style contribution score, our ablations evaluate $L_2$ variants, namely computing the sensitivities based on $a^{L_2}(G,g_i,q):=c_i^2(q)\rho^2(g_i,q)T^2(G,g_i,q)$, which emphasize large contribution discrepancies, and nocolor variants, namely $a^{nc}(G,g_i,q):=\rho(g_i,q)T(G,g_i,q)$, which remove the color factor to focus the score more strongly on geometry, opacity, visibility, and coverage. $A^{L_2}$ and $A^{nc}$ are also defined accordingly. These variants are empirical engineering extensions of the core method: they preserve the same sensitivity-driven selection pipeline.

\section{Experimental Results}\label{sec:experiments}
\begin{table*}[t]
    \centering
    \caption{Per-dataset average prune-only results at prune ratios 0.90 and 0.99. Rows are prune ratios and methods; columns are grouped by dataset and metric. Bold indicates best per dataset/metric/ratio; colored cells mark the top three methods.}
    \label{tab:prune_only_dataset_averages_090_099}
    \tiny
    \setlength{\tabcolsep}{2.2pt}
    \resizebox{0.9\textwidth}{!}{%
    \begin{tabular}{@{}ll !{\vrule width 0.8pt} c c c c c c c c c}
        \toprule
        \multirow{2}{*}{Prune} & \multirow{2}{*}{Method} & \multicolumn{3}{c}{Mip-NeRF 360} & \multicolumn{3}{c}{Deep Blending} & \multicolumn{3}{c}{Tanks \& Temples} \\
        \cmidrule(lr){3-5}\cmidrule(lr){6-8}\cmidrule(lr){9-11}
        & & PSNR $\uparrow$ & SSIM $\uparrow$ & LPIPS $\downarrow$ & PSNR $\uparrow$ & SSIM $\uparrow$ & LPIPS $\downarrow$ & PSNR $\uparrow$ & SSIM $\uparrow$ & LPIPS $\downarrow$ \\
        \midrule
        \multirow{5}{*}{0.90} & GHAP & \cellcolor{yellow!35}16.76 & \cellcolor{orange!25}0.458 & 0.494 & \cellcolor{yellow!35}20.53 & \cellcolor{orange!25}0.744 & 0.437 & \cellcolor{yellow!35}14.83 & \cellcolor{orange!25}0.488 & 0.483 \\
         & PUP & \cellcolor{orange!25}15.79 & \cellcolor{yellow!35}0.535 & \cellcolor{yellow!35}0.429 & \cellcolor{orange!25}19.52 & \cellcolor{yellow!35}0.765 & \cellcolor{yellow!35}0.393 & \cellcolor{orange!25}13.38 & \cellcolor{yellow!35}0.577 & \textbf{\cellcolor{green!30}0.406} \\
         & Trim & 14.45 & 0.413 & \cellcolor{orange!25}0.480 & 16.52 & 0.642 & \cellcolor{orange!25}0.429 & 12.46 & 0.471 & \cellcolor{orange!25}0.458 \\
         & Unif & 13.16 & 0.391 & 0.516 & 15.69 & 0.654 & 0.501 & 10.22 & 0.437 & 0.482 \\
         & Ours & \textbf{\cellcolor{green!30}18.16} & \textbf{\cellcolor{green!30}0.580} & \textbf{\cellcolor{green!30}0.417} & \textbf{\cellcolor{green!30}22.16} & \textbf{\cellcolor{green!30}0.790} & \textbf{\cellcolor{green!30}0.390} & \textbf{\cellcolor{green!30}16.35} & \textbf{\cellcolor{green!30}0.599} & \cellcolor{yellow!35}0.407 \\
        \cmidrule(lr){1-11}
        \multirow{5}{*}{0.99} & GHAP & 10.19 & \cellcolor{orange!25}0.177 & 0.664 & \cellcolor{yellow!35}10.14 & \cellcolor{orange!25}0.401 & \cellcolor{orange!25}0.653 & \cellcolor{orange!25}7.94 & \cellcolor{orange!25}0.224 & 0.675 \\
         & PUP & \cellcolor{yellow!35}10.67 & \cellcolor{yellow!35}0.227 & \cellcolor{yellow!35}0.624 & \cellcolor{orange!25}10.13 & \cellcolor{yellow!35}0.428 & \cellcolor{yellow!35}0.621 & \cellcolor{yellow!35}8.94 & \cellcolor{yellow!35}0.320 & \cellcolor{yellow!35}0.627 \\
         & Trim & 9.58 & 0.136 & 0.661 & 7.66 & 0.132 & 0.683 & 6.23 & 0.104 & 0.668 \\
         & Unif & \cellcolor{orange!25}10.27 & 0.174 & \cellcolor{orange!25}0.642 & 7.84 & 0.189 & 0.696 & 6.92 & 0.163 & \cellcolor{orange!25}0.658 \\
         & Ours & \textbf{\cellcolor{green!30}13.97} & \textbf{\cellcolor{green!30}0.396} & \textbf{\cellcolor{green!30}0.600} & \textbf{\cellcolor{green!30}15.62} & \textbf{\cellcolor{green!30}0.660} & \textbf{\cellcolor{green!30}0.526} & \textbf{\cellcolor{green!30}11.98} & \textbf{\cellcolor{green!30}0.435} & \textbf{\cellcolor{green!30}0.606} \\
        \bottomrule
    \end{tabular}}
\end{table*}

\begin{figure}[t]
    \centering
    \includegraphics[width=0.9\textwidth]{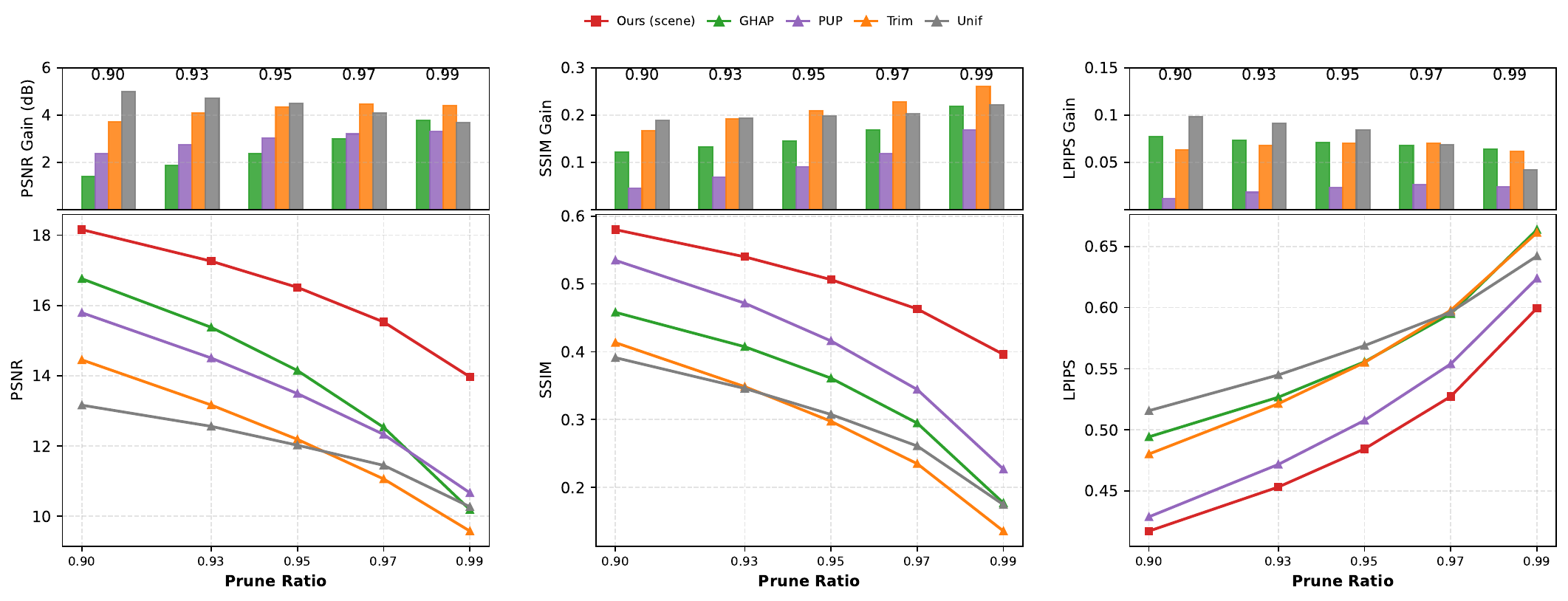}\caption{Combined prune-only curves and deltas for mip-NeRF 360: prune-only deltas (top row), and the prune-only curves (bottom row). For prune-only results, our method is significantly stronger than the competing methods across all metrics.}
\label{fig:combined_prune_delta_mipnerf360}
\end{figure}
Our experimental study targets scarce-compute deployment: aggressive compression at high prune ratios, prune-only evaluation, and short finetuning budgets where long post-pruning recovery is not affordable. The results report how much fidelity each reduction rule preserves immediately after pruning and how the reduced models change under small recovery budgets. The prune-only setting corresponds to zero post-pruning optimization. The short-recovery setting uses low finetuning budgets of 10, 50, 100, and 200 iterations. This emphasizes the regime in which the retained subset must already be strong, rather than relying on a long retraining phase. Experiments were conducted on a PC equipped with one NVIDIA A100-SXM4-40GB GPU and AMD EPYC 7742 64-Core CPUs. %

\noindent\textbf{Datasets.} We use the same 13 scenes used by the 3DGS work~\citep{kerbl2023} and other competing methods: Nine scenes from the Mip-NeRF 360 dataset~\cite{barron2022mip}, two scenes (truck and train) from the Tanks \& Temples dataset~\cite{knapitsch2017tanks}, and two scenes (drjohnson and playroom) from the Deep Blending dataset~\cite{hedman2018deep}.

\textbf{Reported results.} We report PSNR, SSIM, and LPIPS as the primary image-quality metrics. We also report the pruned-Gaussian ratio and the post-pruning optimization budget where applicable. %

\textbf{Competing methods are:} (1) GHAP: Gaussian Herding across Pens \citep{wang2025ghap}; (2) PUP 3D-GS: Principled Uncertainty Pruning for 3D Gaussian Splatting \citep{hanson2025pup}; (3) Trimming the Fat \citep{ali2024trimming}; (4)  Uniform sampling is used as a query-agnostic baseline. Our method is denoted \emph{Ours}. Unless stated otherwise, in all experiments, we use the per-scene sensitivity variant.  We also evaluate all variants later in this section.

\noindent\textbf{Prune-only results at aggressive compression.}
 We evaluate prune-only fidelity at the strongest compression levels where at most $20\%$ of the Gaussians are retained.  As shown in \cref{tab:prune_only_dataset_averages_090_099}, for prune ratios $0.9$ and $0.99$, our method variants consistently outperform the competing  methods across all test datasets, achieving the best PSNR, SSIM, and LPIPS scores in the prune-only setting. The same trend persists across the remaining compression ratios: $\{0.80,0.85,0.90\}$ and $\{0.95,0.97,0.99\}$ which are reported, respectively, in \cref{tab:prune_only_moderate} and \cref{tab:prune_only_aggressive} in the appendix due to space.

\noindent\textbf{Prune-only curves and delta trends.}
\cref{fig:combined_prune_delta_mipnerf360} sweeps the high prune-ratio regime on mip-NeRF 360 and reports PSNR, SSIM, and LPIPS. The upper panels report delta plots that isolate the gap between our method and each competing method: positive deltas indicate an advantage for our method, with PSNR and SSIM computed as ours minus competitor, and LPIPS computed as competitor minus ours. It clearly
shows how each method degrades as compression becomes more aggressive. While our method already shows a large advantage, this gap widens at higher prune ratios for PSNR and SSIM, and remains consistently strong for LPIPS.
The only exception is the delta to uniform pruning, whose gap can shrink at the highest prune ratios because uniform pruning already starts from very poor fidelity. The same trends are observed on all of the other datasets in \cref{fig:combined_prune_delta_tanks_temples} of the appendix.

\begin{wrapfigure}{r}{0.65\textwidth}
    \centering
    \includegraphics[width=0.65\textwidth]{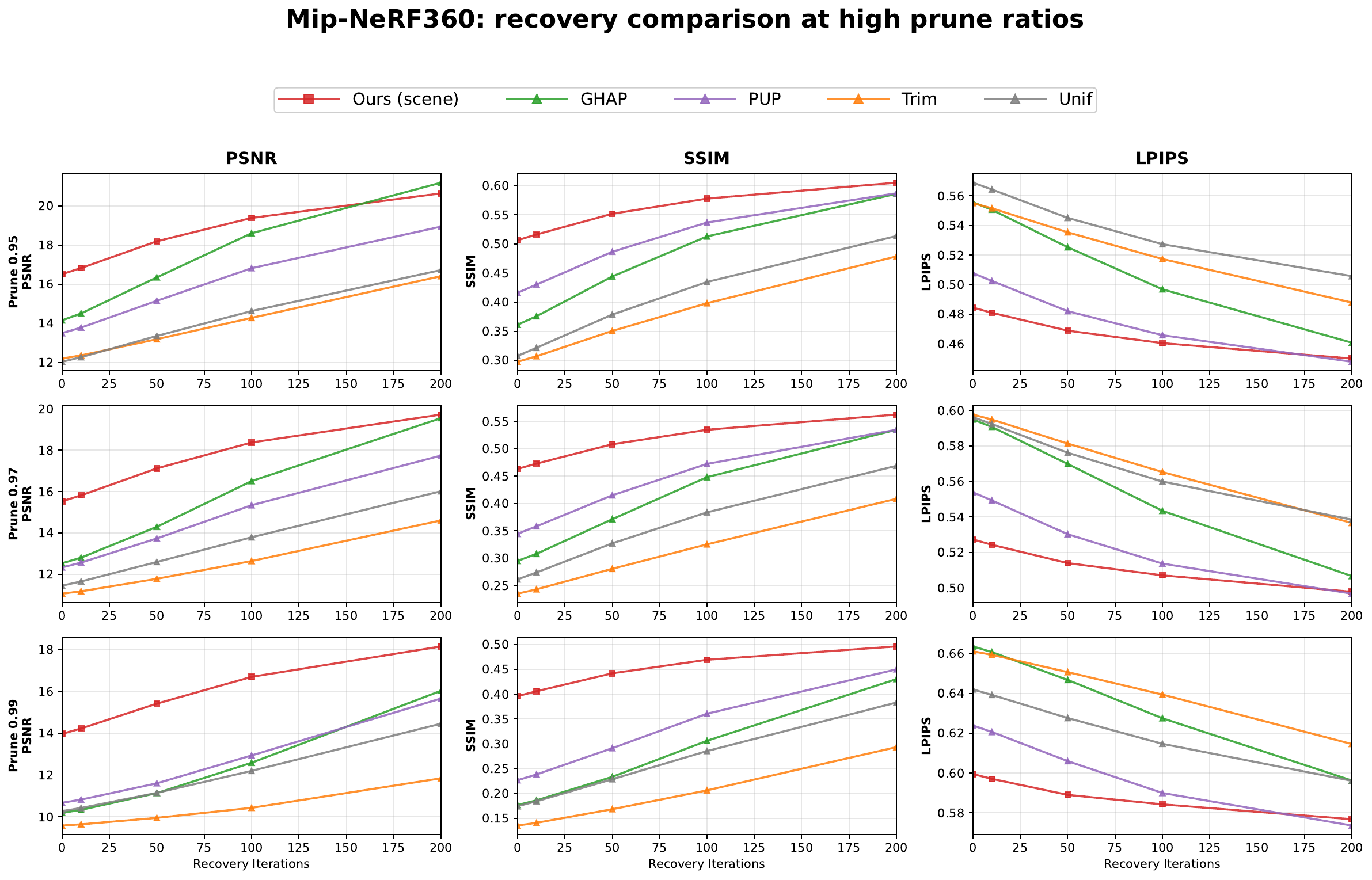}
    \caption{Short-recovery comparison against other methods on mip-NeRF 360 at prune ratios 0.95, 0.97, and 0.99. }%
    \label{fig:short_recovery_competitors_mipnerf360_small}
\end{wrapfigure}
\noindent\textbf{Short-recovery results.}
The short-recovery experiments separate two effects: the fidelity preserved immediately after pruning, and the amount each method can recover under a small finetuning budget. Matched comparisons against competing methods at prune ratios $0.95$, $0.97$, and $0.99$ are shown in \cref{fig:short_recovery_competitors_mipnerf360_small} for mip-NeRF 360. Our methods maintain a clear advantage throughout the low-budget recovery regime, largely because they start from a stronger reduced model. The gap is most pronounced at shorter finetuning budgets and higher prune ratios, precisely where recovery compute is most limited. The main exception is GHAP, which begins to approach our performance after $200$ finetuning iterations. This is expected, since GHAP is a global compaction method rather than standard pruning: by merging and adapting Gaussians, it has more opportunity to recover quality during finetuning. However, this also makes it less suitable for the strict prune-only or very-short-recovery setting, where the goal is to obtain an immediately strong vanilla-compatible subset with minimal post-processing. The same phenomenon is observed on the other tested datasets; see \cref{fig:short_recovery_competitors_deep_blend,fig:short_recovery_competitors_tanks_temples}.

\begin{wrapfigure}{r}{0.65\textwidth}
    \centering
    \includegraphics[width=0.65\textwidth]{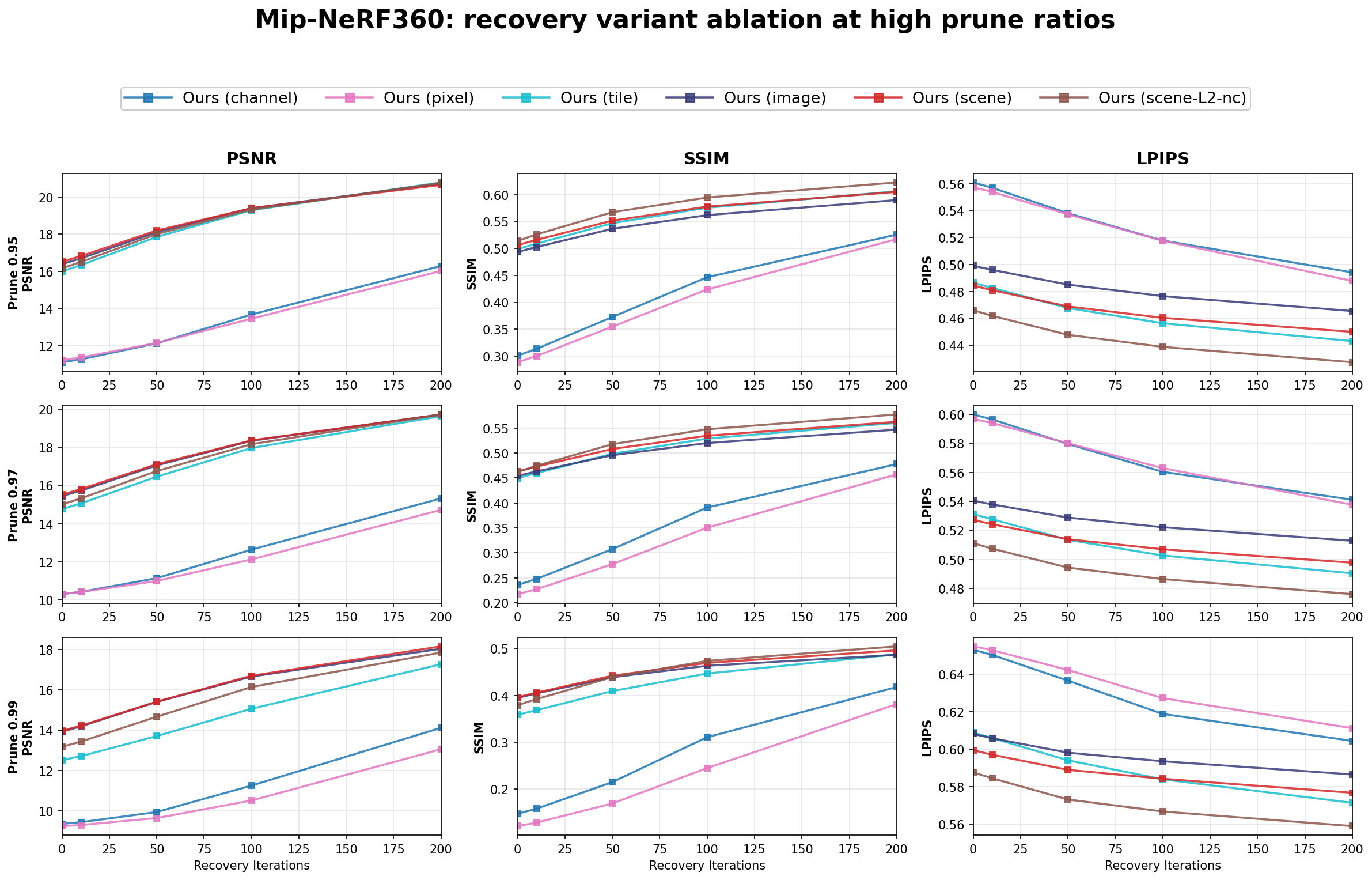}
    \caption{Short-recovery comparison of our methods variants on mip-NeRF 360 at prune ratios 0.95, 0.97, and 0.99.}
    \label{fig:appendix_recovery_variants_mipnerf360}
\end{wrapfigure}
\textbf{Ablations: Effect of sensitivity resolution and engineering variant.} We compare the different sensitivity resolution variants, varying the pruning ratio and reporting PSNR, LPIPS, and SSIM, both in the prune-only setting (\cref{fig:prune-only-fig}) and after the short recovery stage (\cref{fig:appendix_recovery_variants_mipnerf360}). We evaluate the per-channel, per-pixel, per-tile, and per-scene variants introduced in Section~\ref{sec:practical-extensions}, together with simple engineering variants of the scene-level score, such as removing the color term or using an $\ell_2$ aggregation instead of $\ell_1$.
Across both settings, the same trend is observed: per-scene sensitivity and its engineered variant perform best, followed by per-tile, per-pixel, and per-channel scoring. This suggests that aggregating contributions over larger rendering objectives yields a more stable importance signal, especially for identifying Gaussians that are useful across many pixels and views rather than only for isolated channel-level queries. The fact that the trend persists after recovery further shows that the aggregated sensitivities provide a better initialization for finetuning, not only better immediate pruning. Finally, the strong performance of the engineered variant indicates that the sensitivity framework is robust but not saturated: the theoretical principle already gives high-quality subsets, while simple implementation choices in the aggregation rule can still provide additional gains.

\begin{figure}[t]
    \centering
    \includegraphics[width=0.9\textwidth]{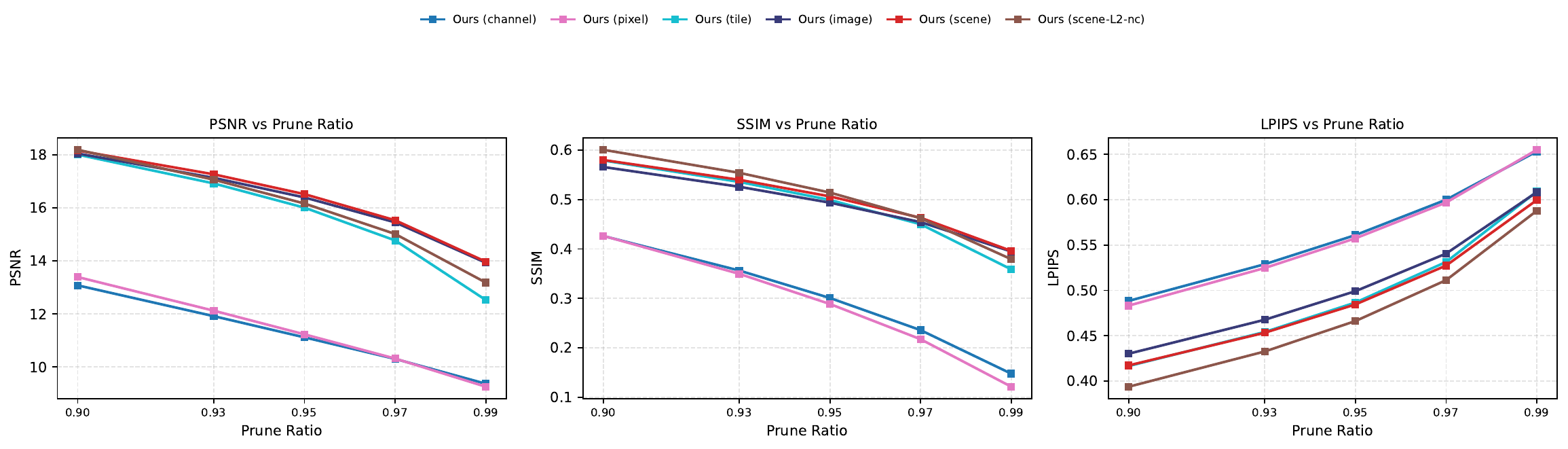}\caption{Ablation: prune-only on Mip-NeRF 360 (prune ratios 0.90--0.99) across our variants.}
    \label{fig:prune-only-fig}
\end{figure}

\section{Conclusion, Limitations, and Scope}
\label{sec:limitations}
We studied 3D Gaussian Splatting pruning from a provable coreset perspective. Our results show that the right question is not whether unrestricted 3DGS admits a universal multiplicative coreset: it does not.
Rather, meaningful guarantees become possible once the target rendering resolution is specified through a representative family of views, rays, pixels, or aggregated objectives.
Under this structured regime, we proved a sensitivity-based coreset construction for 3DGS and showed how its fixed-objective guarantee can be transferred to true re-rendering under explicit validity and log-transmittance stability assumptions. Beyond the theory, the same sensitivity principle leads to a practical pruning rule that is especially effective when recovery compute is scarce.
Across prune-only and very short finetuning regimes, our method preserves rendering quality better than competing pruning heuristics at aggressive compression ratios.
This suggests that coreset-based importance estimation is not only a useful theoretical lens for understanding when 3DGS compression is possible, but also a practical tool for deployment settings where one cannot rely on long post-pruning optimization.
We hope this work provides a foundation for future compression methods that combine the efficiency of vanilla-compatible pruning with explicit approximation guarantees.

\textbf{Limitations. }Our guarantees apply to the structured regime formalized in the paper, not to unrestricted continuous rendering, where we prove that non-trivial multiplicative coresets are impossible. The positive results hold for representative query families induced by a target rendering resolution, such as grids of camera poses, rays, pixels, or tiles, and transfer to rendering under validity and log-transmittance stability assumptions. Empirically, we focus on prune-only evaluation and very short finetuning budgets. This low-compute regime is where subset quality matters most, since little recovery compute is available to repair poor pruning decisions. Methods that rely on longer recovery or specialized compression pipelines may become more competitive under larger optimization budgets, but they are largely orthogonal to our contribution: our pruning rule can provide a strong vanilla-compatible initialization before additional refinement. See \cref{sec:limitations:ext} for the extended discussion.

\noindent\textbf{Code availability.} The open-source implementation and evaluation pipeline are available at \projectlink.

\clearpage
\bibliographystyle{unsrtnat}
\bibliography{references}

\clearpage
\appendix
\suppressfloats[t]

\setcounter{topnumber}{5}
\setcounter{bottomnumber}{5}
\setcounter{totalnumber}{10}
\renewcommand{\topfraction}{0.97}
\renewcommand{\bottomfraction}{0.97}
\renewcommand{\textfraction}{0.02}
\renewcommand{\floatpagefraction}{0.80}
\setlength{\textfloatsep}{9pt plus 2pt minus 2pt}
\setlength{\floatsep}{8pt plus 2pt minus 2pt}
\setlength{\intextsep}{8pt plus 2pt minus 2pt}
\setlength{\abovecaptionskip}{3pt}
\setlength{\belowcaptionskip}{0pt}

\section{Reproducibility and Experimental Protocol}
\label{sec:appendix-protocol}

This section collects the experimental information that is distributed across the main paper and states how the extended results should be read. The experiments are designed to isolate the quality of the retained Gaussian subset in the low-compute regime, first without post-pruning optimization and then under a common short recovery budget.

\subsection{Evaluation setting and datasets}
We evaluate the same 13 scenes described in \cref{sec:experiments}: nine scenes from mip-NeRF 360, two scenes from Tanks \& Temples, and two scenes from Deep Blending. The prune-only experiments evaluate the reduced scenes immediately after pruning. The short-recovery experiments use 10, 50, 100, and 200 post-pruning finetuning iterations. All experiments were run on one NVIDIA A100-SXM4-40GB GPU with AMD EPYC 7742 64-Core CPUs.

\subsection{Metrics and competing methods}
We report PSNR, SSIM, and LPIPS. Higher PSNR and SSIM are better, while lower LPIPS is better. The competing methods are GHAP, PUP 3D-GS, Trimming the Fat, and uniform sampling. Unless stated otherwise, ``Ours'' denotes the per-scene sensitivity variant used in the main comparison.

\subsection{Sensitivity variants}
The ablations compare per-channel, per-pixel, per-tile, per-image, and per-scene aggregation. They also include the implementation-specific $L_1/L_2$ and color/nocolor variants introduced in \cref{sec:practical-extensions}. These variants change the aggregation or score definition while retaining the same sensitivity-driven selection pipeline.

\subsection{How to read the extended results}
Dataset averages summarize the broad trend, while the per-scene plots reveal whether those averages hide scene-specific reversals. Recovery plots separate the quality retained immediately after pruning from the quality recovered through additional optimization. The wider ablations test whether the ordering of the sensitivity variants remains stable across datasets, prune ratios, and recovery budgets.

\section{Extended Experimental Results and Ablations}\label{sec:appendix-main-results}

This section contains the complete experimental record supporting the main results: full prune-only tables and dataset-level curves, short-recovery comparisons, per-scene prune-only plots, and the complete sensitivity-variant ablations. The numerical results and plotted curves are unchanged from the submitted paper; they are reorganized here by experimental question.

\subsection{Complete prune-only tables}
The following tables report the prune-only results over all 13 scenes. Because no recovery optimization is applied, these numbers directly measure the fidelity of the selected reduced scene.

\begin{table*}[t]
    \centering
    \caption{Prune-only results at prune ratios 0.80, 0.85, and 0.90 over all 13 scenes. Rows are scenes split into PSNR / SSIM / LPIPS, with each bracket followed by a per-dataset average row, and columns are grouped by prune ratio and method. Bold indicates best per scene/metric within each prune-ratio group; colored cells mark the top three methods.}
    \label{tab:prune_only_moderate}
    \tiny
    \setlength{\tabcolsep}{1.4pt}
    \resizebox{\textwidth}{!}{%
    \begin{tabular}{@{}c@{\,}ll !{\vrule width 0.8pt} c c c c c !{\vrule width 0.8pt} c c c c c !{\vrule width 0.8pt} c c c c c}
        \toprule
        & \multirow{2}{*}{Scene} & \multirow{2}{*}{Metric} & \multicolumn{5}{c}{Prune 0.80} & \multicolumn{5}{c}{Prune 0.85} & \multicolumn{5}{c}{Prune 0.90} \\
        \cmidrule(lr){4-8}\cmidrule(lr){9-13}\cmidrule(lr){14-18}
        & & & GHAP & PUP & Trim & Unif & Ours & GHAP & PUP & Trim & Unif & Ours & GHAP & PUP & Trim & Unif & Ours \\
        \midrule
        \rdelim\{{27}{*} & \multirow{3}{*}{Flowers} & PSNR $\uparrow$ & \cellcolor{orange!25}16.06 & \cellcolor{yellow!35}16.89 & 14.34 & 10.71 & \textbf{\cellcolor{green!30}17.23} & \cellcolor{orange!25}15.18 & \cellcolor{yellow!35}15.69 & 12.74 & 10.39 & \textbf{\cellcolor{green!30}16.66} & \cellcolor{orange!25}14.01 & \cellcolor{yellow!35}14.31 & 11.12 & 10.03 & \textbf{\cellcolor{green!30}15.91} \\
         &  & SSIM $\uparrow$ & 0.350 & \textbf{\cellcolor{green!30}0.454} & \cellcolor{orange!25}0.369 & 0.224 & \cellcolor{yellow!35}0.428 & \cellcolor{orange!25}0.314 & \textbf{\cellcolor{green!30}0.418} & 0.306 & 0.199 & \cellcolor{yellow!35}0.404 & \cellcolor{orange!25}0.267 & \cellcolor{yellow!35}0.367 & 0.227 & 0.174 & \textbf{\cellcolor{green!30}0.372} \\
         &  & LPIPS $\downarrow$ & 0.525 & \textbf{\cellcolor{green!30}0.456} & \cellcolor{orange!25}0.501 & 0.603 & \cellcolor{yellow!35}0.481 & 0.550 & \textbf{\cellcolor{green!30}0.483} & \cellcolor{orange!25}0.543 & 0.626 & \cellcolor{yellow!35}0.503 & \cellcolor{orange!25}0.584 & \textbf{\cellcolor{green!30}0.521} & 0.600 & 0.651 & \cellcolor{yellow!35}0.532 \\
         & \multirow{3}{*}{Treehill} & PSNR $\uparrow$ & \cellcolor{yellow!35}18.47 & \cellcolor{orange!25}18.03 & 13.00 & 13.23 & \textbf{\cellcolor{green!30}19.21} & \cellcolor{yellow!35}17.70 & \cellcolor{orange!25}17.08 & 11.97 & 12.63 & \textbf{\cellcolor{green!30}18.74} & \cellcolor{yellow!35}16.55 & \cellcolor{orange!25}15.89 & 10.96 & 12.39 & \textbf{\cellcolor{green!30}18.10} \\
         &  & SSIM $\uparrow$ & \cellcolor{orange!25}0.432 & \cellcolor{yellow!35}0.542 & 0.281 & 0.399 & \textbf{\cellcolor{green!30}0.576} & \cellcolor{orange!25}0.405 & \cellcolor{yellow!35}0.509 & 0.227 & 0.368 & \textbf{\cellcolor{green!30}0.553} & \cellcolor{orange!25}0.367 & \cellcolor{yellow!35}0.461 & 0.172 & 0.349 & \textbf{\cellcolor{green!30}0.516} \\
         &  & LPIPS $\downarrow$ & \cellcolor{orange!25}0.512 & \cellcolor{yellow!35}0.448 & 0.553 & 0.537 & \textbf{\cellcolor{green!30}0.417} & \cellcolor{orange!25}0.531 & \cellcolor{yellow!35}0.479 & 0.577 & 0.558 & \textbf{\cellcolor{green!30}0.442} & \cellcolor{orange!25}0.555 & \cellcolor{yellow!35}0.520 & 0.602 & 0.578 & \textbf{\cellcolor{green!30}0.479} \\
         & \multirow{3}{*}{Bicycle} & PSNR $\uparrow$ & \cellcolor{yellow!35}19.07 & \cellcolor{orange!25}18.88 & 15.41 & 15.11 & \textbf{\cellcolor{green!30}19.93} & \cellcolor{yellow!35}17.90 & \cellcolor{orange!25}17.82 & 14.04 & 14.47 & \textbf{\cellcolor{green!30}19.12} & \cellcolor{orange!25}16.19 & \cellcolor{yellow!35}16.61 & 12.52 & 14.07 & \textbf{\cellcolor{green!30}18.01} \\
         &  & SSIM $\uparrow$ & \cellcolor{orange!25}0.477 & \cellcolor{yellow!35}0.579 & 0.444 & 0.378 & \textbf{\cellcolor{green!30}0.588} & \cellcolor{orange!25}0.430 & \cellcolor{yellow!35}0.530 & 0.364 & 0.338 & \textbf{\cellcolor{green!30}0.549} & \cellcolor{orange!25}0.363 & \cellcolor{yellow!35}0.464 & 0.263 & 0.309 & \textbf{\cellcolor{green!30}0.494} \\
         &  & LPIPS $\downarrow$ & 0.433 & \cellcolor{yellow!35}0.360 & \cellcolor{orange!25}0.426 & 0.502 & \textbf{\cellcolor{green!30}0.357} & \cellcolor{orange!25}0.460 & \cellcolor{yellow!35}0.394 & 0.469 & 0.531 & \textbf{\cellcolor{green!30}0.385} & \cellcolor{orange!25}0.498 & \cellcolor{yellow!35}0.440 & 0.522 & 0.557 & \textbf{\cellcolor{green!30}0.427} \\
         & \multirow{3}{*}{Bonsai} & PSNR $\uparrow$ & \cellcolor{yellow!35}21.99 & 19.03 & \textbf{\cellcolor{green!30}22.90} & 12.42 & \cellcolor{orange!25}19.86 & \cellcolor{yellow!35}20.38 & 16.68 & \textbf{\cellcolor{green!30}21.24} & 11.79 & \cellcolor{orange!25}19.16 & \cellcolor{orange!25}18.09 & 14.10 & \textbf{\cellcolor{green!30}19.17} & 11.31 & \cellcolor{yellow!35}18.24 \\
         &  & SSIM $\uparrow$ & 0.726 & \textbf{\cellcolor{green!30}0.774} & \cellcolor{yellow!35}0.764 & 0.473 & \cellcolor{orange!25}0.736 & 0.676 & \cellcolor{yellow!35}0.701 & \textbf{\cellcolor{green!30}0.706} & 0.426 & \cellcolor{orange!25}0.698 & \cellcolor{orange!25}0.602 & 0.587 & \cellcolor{yellow!35}0.624 & 0.359 & \textbf{\cellcolor{green!30}0.651} \\
         &  & LPIPS $\downarrow$ & 0.388 & \cellcolor{yellow!35}0.322 & \textbf{\cellcolor{green!30}0.310} & 0.427 & \cellcolor{orange!25}0.344 & 0.422 & \cellcolor{yellow!35}0.367 & \textbf{\cellcolor{green!30}0.351} & 0.453 & \cellcolor{orange!25}0.377 & 0.466 & \cellcolor{orange!25}0.424 & \textbf{\cellcolor{green!30}0.402} & 0.487 & \cellcolor{yellow!35}0.417 \\
         & \multirow{3}{*}{Counter} & PSNR $\uparrow$ & \cellcolor{orange!25}21.16 & 20.14 & \cellcolor{yellow!35}21.30 & 16.26 & \textbf{\cellcolor{green!30}22.83} & \cellcolor{yellow!35}19.83 & 17.94 & \cellcolor{orange!25}19.34 & 14.85 & \textbf{\cellcolor{green!30}21.58} & \cellcolor{yellow!35}17.93 & 15.19 & \cellcolor{orange!25}17.03 & 13.30 & \textbf{\cellcolor{green!30}19.92} \\
         &  & SSIM $\uparrow$ & 0.691 & \cellcolor{yellow!35}0.780 & \cellcolor{orange!25}0.753 & 0.639 & \textbf{\cellcolor{green!30}0.797} & 0.650 & \cellcolor{yellow!35}0.725 & \cellcolor{orange!25}0.691 & 0.586 & \textbf{\cellcolor{green!30}0.764} & 0.588 & \cellcolor{yellow!35}0.634 & \cellcolor{orange!25}0.604 & 0.504 & \textbf{\cellcolor{green!30}0.717} \\
         &  & LPIPS $\downarrow$ & 0.381 & \cellcolor{yellow!35}0.297 & \textbf{\cellcolor{green!30}0.295} & 0.380 & \cellcolor{orange!25}0.301 & 0.413 & \cellcolor{orange!25}0.339 & \cellcolor{yellow!35}0.336 & 0.409 & \textbf{\cellcolor{green!30}0.334} & 0.455 & \cellcolor{orange!25}0.396 & \cellcolor{yellow!35}0.394 & 0.458 & \textbf{\cellcolor{green!30}0.377} \\
         & \multirow{3}{*}{Garden} & PSNR $\uparrow$ & \cellcolor{yellow!35}18.86 & \cellcolor{orange!25}17.83 & 14.35 & 16.75 & \textbf{\cellcolor{green!30}19.36} & \cellcolor{yellow!35}17.41 & \cellcolor{orange!25}16.44 & 12.77 & 15.94 & \textbf{\cellcolor{green!30}18.26} & \cellcolor{yellow!35}15.33 & 14.75 & 11.09 & \cellcolor{orange!25}15.07 & \textbf{\cellcolor{green!30}16.95} \\
         &  & SSIM $\uparrow$ & 0.491 & \cellcolor{yellow!35}0.597 & 0.468 & \cellcolor{orange!25}0.510 & \textbf{\cellcolor{green!30}0.604} & 0.435 & \cellcolor{yellow!35}0.531 & 0.383 & \cellcolor{orange!25}0.467 & \textbf{\cellcolor{green!30}0.546} & 0.360 & \cellcolor{yellow!35}0.446 & 0.284 & \cellcolor{orange!25}0.418 & \textbf{\cellcolor{green!30}0.476} \\
         &  & LPIPS $\downarrow$ & \cellcolor{orange!25}0.402 & \textbf{\cellcolor{green!30}0.304} & 0.407 & 0.409 & \cellcolor{yellow!35}0.313 & \cellcolor{orange!25}0.435 & \textbf{\cellcolor{green!30}0.352} & 0.462 & 0.443 & \cellcolor{yellow!35}0.356 & 0.483 & \textbf{\cellcolor{green!30}0.415} & 0.526 & \cellcolor{orange!25}0.481 & \cellcolor{yellow!35}0.416 \\
         & \multirow{3}{*}{Kitchen} & PSNR $\uparrow$ & \cellcolor{yellow!35}19.24 & 17.37 & \cellcolor{orange!25}18.34 & 14.27 & \textbf{\cellcolor{green!30}19.31} & \cellcolor{yellow!35}17.69 & 15.19 & \cellcolor{orange!25}16.09 & 13.36 & \textbf{\cellcolor{green!30}17.96} & \cellcolor{yellow!35}15.68 & 12.89 & \cellcolor{orange!25}13.52 & 11.92 & \textbf{\cellcolor{green!30}16.40} \\
         &  & SSIM $\uparrow$ & 0.659 & \cellcolor{yellow!35}0.733 & \cellcolor{orange!25}0.726 & 0.615 & \textbf{\cellcolor{green!30}0.766} & 0.603 & \cellcolor{yellow!35}0.657 & \cellcolor{orange!25}0.651 & 0.569 & \textbf{\cellcolor{green!30}0.717} & 0.525 & \cellcolor{yellow!35}0.550 & \cellcolor{orange!25}0.548 & 0.467 & \textbf{\cellcolor{green!30}0.651} \\
         &  & LPIPS $\downarrow$ & 0.369 & \textbf{\cellcolor{green!30}0.257} & \cellcolor{orange!25}0.267 & 0.327 & \cellcolor{yellow!35}0.258 & 0.412 & \cellcolor{yellow!35}0.310 & \cellcolor{orange!25}0.320 & 0.361 & \textbf{\cellcolor{green!30}0.299} & 0.466 & \cellcolor{yellow!35}0.385 & \cellcolor{orange!25}0.394 & 0.437 & \textbf{\cellcolor{green!30}0.358} \\
         & \multirow{3}{*}{Room} & PSNR $\uparrow$ & \cellcolor{orange!25}23.17 & \cellcolor{yellow!35}24.30 & \textbf{\cellcolor{green!30}24.34} & 18.03 & 21.51 & \cellcolor{orange!25}21.81 & \textbf{\cellcolor{green!30}22.44} & \cellcolor{yellow!35}22.20 & 16.76 & 20.78 & 19.56 & \textbf{\cellcolor{green!30}19.89} & \cellcolor{orange!25}19.63 & 14.73 & \cellcolor{yellow!35}19.78 \\
         &  & SSIM $\uparrow$ & 0.762 & \textbf{\cellcolor{green!30}0.852} & \cellcolor{yellow!35}0.822 & 0.715 & \cellcolor{orange!25}0.796 & 0.729 & \textbf{\cellcolor{green!30}0.819} & \cellcolor{orange!25}0.773 & 0.669 & \cellcolor{yellow!35}0.774 & 0.675 & \textbf{\cellcolor{green!30}0.766} & \cellcolor{orange!25}0.696 & 0.588 & \cellcolor{yellow!35}0.744 \\
         &  & LPIPS $\downarrow$ & 0.367 & \textbf{\cellcolor{green!30}0.272} & \cellcolor{yellow!35}0.284 & 0.386 & \cellcolor{orange!25}0.314 & 0.395 & \textbf{\cellcolor{green!30}0.303} & \cellcolor{yellow!35}0.324 & 0.417 & \cellcolor{orange!25}0.339 & 0.437 & \textbf{\cellcolor{green!30}0.352} & \cellcolor{orange!25}0.377 & 0.465 & \cellcolor{yellow!35}0.376 \\
         & \multirow{3}{*}{Stump} & PSNR $\uparrow$ & \cellcolor{orange!25}20.00 & \cellcolor{yellow!35}22.13 & 18.06 & 16.86 & \textbf{\cellcolor{green!30}22.34} & \cellcolor{orange!25}18.93 & \cellcolor{yellow!35}20.49 & 16.61 & 16.19 & \textbf{\cellcolor{green!30}21.50} & \cellcolor{orange!25}17.50 & \cellcolor{yellow!35}18.51 & 15.05 & 15.65 & \textbf{\cellcolor{green!30}20.12} \\
         &  & SSIM $\uparrow$ & 0.487 & \cellcolor{yellow!35}0.677 & \cellcolor{orange!25}0.492 & 0.430 & \textbf{\cellcolor{green!30}0.683} & \cellcolor{orange!25}0.442 & \cellcolor{yellow!35}0.623 & 0.409 & 0.390 & \textbf{\cellcolor{green!30}0.655} & \cellcolor{orange!25}0.377 & \cellcolor{yellow!35}0.540 & 0.303 & 0.353 & \textbf{\cellcolor{green!30}0.599} \\
         &  & LPIPS $\downarrow$ & 0.446 & \cellcolor{yellow!35}0.314 & \cellcolor{orange!25}0.409 & 0.476 & \textbf{\cellcolor{green!30}0.305} & 0.471 & \cellcolor{yellow!35}0.352 & \cellcolor{orange!25}0.451 & 0.502 & \textbf{\cellcolor{green!30}0.329} & \cellcolor{orange!25}0.505 & \cellcolor{yellow!35}0.406 & 0.505 & 0.526 & \textbf{\cellcolor{green!30}0.374} \\
        \cmidrule(lr){2-18}
         & \multirow{3}{*}{\textbf{Mip-NeRF 360}} & PSNR $\uparrow$ & \cellcolor{yellow!35}19.78 & \cellcolor{orange!25}19.40 & 18.01 & 14.85 & \textbf{\cellcolor{green!30}20.18} & \cellcolor{yellow!35}18.54 & \cellcolor{orange!25}17.75 & 16.33 & 14.04 & \textbf{\cellcolor{green!30}19.30} & \cellcolor{yellow!35}16.76 & \cellcolor{orange!25}15.79 & 14.45 & 13.16 & \textbf{\cellcolor{green!30}18.16} \\
         &  & SSIM $\uparrow$ & 0.564 & \textbf{\cellcolor{green!30}0.665} & \cellcolor{orange!25}0.569 & 0.487 & \cellcolor{yellow!35}0.664 & \cellcolor{orange!25}0.521 & \cellcolor{yellow!35}0.612 & 0.501 & 0.446 & \textbf{\cellcolor{green!30}0.629} & \cellcolor{orange!25}0.458 & \cellcolor{yellow!35}0.535 & 0.413 & 0.391 & \textbf{\cellcolor{green!30}0.580} \\
         &  & LPIPS $\downarrow$ & 0.425 & \textbf{\cellcolor{green!30}0.337} & \cellcolor{orange!25}0.384 & 0.450 & \cellcolor{yellow!35}0.343 & 0.454 & \cellcolor{yellow!35}0.375 & \cellcolor{orange!25}0.426 & 0.478 & \textbf{\cellcolor{green!30}0.374} & 0.494 & \cellcolor{yellow!35}0.429 & \cellcolor{orange!25}0.480 & 0.516 & \textbf{\cellcolor{green!30}0.417} \\
        \specialrule{0.2pt}{1pt}{1pt}
        \rdelim\{{6}{*} & \multirow{3}{*}{DrJohnson} & PSNR $\uparrow$ & \textbf{\cellcolor{green!30}24.48} & \cellcolor{yellow!35}23.88 & 21.68 & 22.17 & \cellcolor{orange!25}23.74 & \cellcolor{yellow!35}22.92 & \cellcolor{orange!25}21.84 & 19.42 & 19.88 & \textbf{\cellcolor{green!30}22.95} & \cellcolor{yellow!35}20.69 & \cellcolor{orange!25}19.14 & 16.59 & 18.00 & \textbf{\cellcolor{green!30}21.92} \\
         &  & SSIM $\uparrow$ & \cellcolor{orange!25}0.812 & \textbf{\cellcolor{green!30}0.832} & 0.770 & 0.776 & \cellcolor{yellow!35}0.813 & \cellcolor{orange!25}0.781 & \textbf{\cellcolor{green!30}0.798} & 0.711 & 0.730 & \cellcolor{yellow!35}0.796 & \cellcolor{orange!25}0.730 & \cellcolor{yellow!35}0.743 & 0.608 & 0.688 & \textbf{\cellcolor{green!30}0.773} \\
         &  & LPIPS $\downarrow$ & 0.371 & \textbf{\cellcolor{green!30}0.320} & \cellcolor{yellow!35}0.333 & 0.399 & \cellcolor{orange!25}0.346 & 0.401 & \textbf{\cellcolor{green!30}0.354} & \cellcolor{orange!25}0.368 & 0.429 & \cellcolor{yellow!35}0.367 & 0.441 & \cellcolor{yellow!35}0.400 & \cellcolor{orange!25}0.418 & 0.461 & \textbf{\cellcolor{green!30}0.396} \\
         & \multirow{3}{*}{Playroom} & PSNR $\uparrow$ & \cellcolor{yellow!35}24.85 & \cellcolor{orange!25}24.70 & 22.36 & 16.85 & \textbf{\cellcolor{green!30}25.08} & \cellcolor{yellow!35}23.18 & \cellcolor{orange!25}22.67 & 19.84 & 15.25 & \textbf{\cellcolor{green!30}23.95} & \cellcolor{yellow!35}20.36 & \cellcolor{orange!25}19.90 & 16.44 & 13.39 & \textbf{\cellcolor{green!30}22.39} \\
         &  & SSIM $\uparrow$ & \cellcolor{orange!25}0.827 & \textbf{\cellcolor{green!30}0.850} & 0.802 & 0.727 & \cellcolor{yellow!35}0.849 & \cellcolor{orange!25}0.801 & \cellcolor{yellow!35}0.825 & 0.755 & 0.686 & \textbf{\cellcolor{green!30}0.831} & \cellcolor{orange!25}0.757 & \cellcolor{yellow!35}0.786 & 0.675 & 0.620 & \textbf{\cellcolor{green!30}0.807} \\
         &  & LPIPS $\downarrow$ & 0.359 & \textbf{\cellcolor{green!30}0.314} & \cellcolor{orange!25}0.337 & 0.454 & \cellcolor{yellow!35}0.325 & 0.388 & \textbf{\cellcolor{green!30}0.342} & \cellcolor{orange!25}0.377 & 0.491 & \cellcolor{yellow!35}0.350 & \cellcolor{orange!25}0.432 & \cellcolor{yellow!35}0.386 & 0.440 & 0.541 & \textbf{\cellcolor{green!30}0.385} \\
        \cmidrule(lr){2-18}
         & \multirow{3}{*}{\textbf{Deep Blending}} & PSNR $\uparrow$ & \textbf{\cellcolor{green!30}24.67} & \cellcolor{orange!25}24.29 & 22.02 & 19.51 & \cellcolor{yellow!35}24.41 & \cellcolor{yellow!35}23.05 & \cellcolor{orange!25}22.25 & 19.63 & 17.56 & \textbf{\cellcolor{green!30}23.45} & \cellcolor{yellow!35}20.53 & \cellcolor{orange!25}19.52 & 16.52 & 15.69 & \textbf{\cellcolor{green!30}22.16} \\
         &  & SSIM $\uparrow$ & \cellcolor{orange!25}0.820 & \textbf{\cellcolor{green!30}0.841} & 0.786 & 0.752 & \cellcolor{yellow!35}0.831 & \cellcolor{orange!25}0.791 & \cellcolor{yellow!35}0.812 & 0.733 & 0.708 & \textbf{\cellcolor{green!30}0.814} & \cellcolor{orange!25}0.744 & \cellcolor{yellow!35}0.765 & 0.642 & 0.654 & \textbf{\cellcolor{green!30}0.790} \\
         &  & LPIPS $\downarrow$ & 0.365 & \textbf{\cellcolor{green!30}0.317} & \cellcolor{yellow!35}0.335 & 0.426 & \cellcolor{orange!25}0.336 & 0.395 & \textbf{\cellcolor{green!30}0.348} & \cellcolor{orange!25}0.372 & 0.460 & \cellcolor{yellow!35}0.359 & 0.437 & \cellcolor{yellow!35}0.393 & \cellcolor{orange!25}0.429 & 0.501 & \textbf{\cellcolor{green!30}0.390} \\
        \specialrule{0.2pt}{1pt}{1pt}
        \rdelim\{{6}{*} & \multirow{3}{*}{Train} & PSNR $\uparrow$ & \cellcolor{yellow!35}16.61 & 15.15 & \cellcolor{orange!25}15.47 & 10.68 & \textbf{\cellcolor{green!30}16.88} & \cellcolor{yellow!35}15.59 & 13.79 & \cellcolor{orange!25}14.12 & 10.01 & \textbf{\cellcolor{green!30}16.21} & \cellcolor{yellow!35}14.20 & 12.05 & \cellcolor{orange!25}12.29 & 8.72 & \textbf{\cellcolor{green!30}15.32} \\
         &  & SSIM $\uparrow$ & 0.573 & \textbf{\cellcolor{green!30}0.629} & \cellcolor{orange!25}0.606 & 0.462 & \cellcolor{yellow!35}0.612 & 0.525 & \cellcolor{yellow!35}0.572 & \cellcolor{orange!25}0.544 & 0.420 & \textbf{\cellcolor{green!30}0.579} & \cellcolor{orange!25}0.463 & \cellcolor{yellow!35}0.494 & 0.458 & 0.338 & \textbf{\cellcolor{green!30}0.537} \\
         &  & LPIPS $\downarrow$ & 0.415 & \cellcolor{yellow!35}0.373 & \textbf{\cellcolor{green!30}0.362} & 0.469 & \cellcolor{orange!25}0.382 & 0.455 & \cellcolor{orange!25}0.419 & \textbf{\cellcolor{green!30}0.408} & 0.497 & \cellcolor{yellow!35}0.417 & 0.505 & \cellcolor{orange!25}0.481 & \cellcolor{yellow!35}0.473 & 0.534 & \textbf{\cellcolor{green!30}0.461} \\
         & \multirow{3}{*}{Truck} & PSNR $\uparrow$ & \cellcolor{yellow!35}18.61 & \cellcolor{orange!25}16.96 & 16.76 & 12.79 & \textbf{\cellcolor{green!30}19.38} & \cellcolor{yellow!35}17.25 & \cellcolor{orange!25}15.98 & 14.94 & 12.32 & \textbf{\cellcolor{green!30}18.52} & \cellcolor{yellow!35}15.47 & \cellcolor{orange!25}14.71 & 12.64 & 11.73 & \textbf{\cellcolor{green!30}17.37} \\
         &  & SSIM $\uparrow$ & 0.638 & \textbf{\cellcolor{green!30}0.746} & \cellcolor{orange!25}0.657 & 0.591 & \cellcolor{yellow!35}0.739 & \cellcolor{orange!25}0.584 & \textbf{\cellcolor{green!30}0.714} & 0.583 & 0.567 & \cellcolor{yellow!35}0.706 & 0.514 & \cellcolor{yellow!35}0.661 & 0.484 & \cellcolor{orange!25}0.536 & \textbf{\cellcolor{green!30}0.662} \\
         &  & LPIPS $\downarrow$ & 0.367 & \textbf{\cellcolor{green!30}0.247} & \cellcolor{orange!25}0.324 & 0.376 & \cellcolor{yellow!35}0.272 & 0.408 & \textbf{\cellcolor{green!30}0.278} & \cellcolor{orange!25}0.376 & 0.399 & \cellcolor{yellow!35}0.307 & 0.461 & \textbf{\cellcolor{green!30}0.331} & 0.443 & \cellcolor{orange!25}0.431 & \cellcolor{yellow!35}0.354 \\
        \cmidrule(lr){2-18}
         & \multirow{3}{*}{\textbf{Tanks \& Temples}} & PSNR $\uparrow$ & \cellcolor{yellow!35}17.61 & 16.05 & \cellcolor{orange!25}16.11 & 11.74 & \textbf{\cellcolor{green!30}18.13} & \cellcolor{yellow!35}16.42 & \cellcolor{orange!25}14.89 & 14.53 & 11.17 & \textbf{\cellcolor{green!30}17.37} & \cellcolor{yellow!35}14.83 & \cellcolor{orange!25}13.38 & 12.46 & 10.22 & \textbf{\cellcolor{green!30}16.35} \\
         &  & SSIM $\uparrow$ & 0.606 & \textbf{\cellcolor{green!30}0.687} & \cellcolor{orange!25}0.632 & 0.526 & \cellcolor{yellow!35}0.675 & 0.555 & \textbf{\cellcolor{green!30}0.643} & \cellcolor{orange!25}0.564 & 0.494 & \cellcolor{yellow!35}0.643 & \cellcolor{orange!25}0.488 & \cellcolor{yellow!35}0.577 & 0.471 & 0.437 & \textbf{\cellcolor{green!30}0.599} \\
         &  & LPIPS $\downarrow$ & 0.391 & \textbf{\cellcolor{green!30}0.310} & \cellcolor{orange!25}0.343 & 0.423 & \cellcolor{yellow!35}0.327 & 0.431 & \textbf{\cellcolor{green!30}0.348} & \cellcolor{orange!25}0.392 & 0.448 & \cellcolor{yellow!35}0.362 & 0.483 & \textbf{\cellcolor{green!30}0.406} & \cellcolor{orange!25}0.458 & 0.482 & \cellcolor{yellow!35}0.407 \\
        \bottomrule
    \end{tabular}}
\end{table*}

\begin{table*}[t]
    \centering
    \caption{Prune-only results at prune ratios 0.95, 0.97, and 0.99 over all 13 scenes. Rows are scenes split into PSNR / SSIM / LPIPS, with each bracket followed by a per-dataset average row, and columns are grouped by prune ratio and method. Bold indicates best per scene/metric within each prune-ratio group; colored cells mark the top three methods.}
    \label{tab:prune_only_aggressive}
    \tiny
    \setlength{\tabcolsep}{1.4pt}
    \resizebox{\textwidth}{!}{%
    \begin{tabular}{@{}c@{\,}ll !{\vrule width 0.8pt} c c c c c !{\vrule width 0.8pt} c c c c c !{\vrule width 0.8pt} c c c c c}
        \toprule
        & \multirow{2}{*}{Scene} & \multirow{2}{*}{Metric} & \multicolumn{5}{c}{Prune 0.95} & \multicolumn{5}{c}{Prune 0.97} & \multicolumn{5}{c}{Prune 0.99} \\
        \cmidrule(lr){4-8}\cmidrule(lr){9-13}\cmidrule(lr){14-18}
        & & & GHAP & PUP & Trim & Unif & Ours & GHAP & PUP & Trim & Unif & Ours & GHAP & PUP & Trim & Unif & Ours \\
        \midrule
        \rdelim\{{27}{*} & \multirow{3}{*}{Flowers} & PSNR $\uparrow$ & \cellcolor{orange!25}12.08 & \cellcolor{yellow!35}12.45 & 9.54 & 9.65 & \textbf{\cellcolor{green!30}14.74} & \cellcolor{orange!25}10.94 & \cellcolor{yellow!35}11.49 & 8.85 & 9.44 & \textbf{\cellcolor{green!30}13.92} & 8.91 & \cellcolor{yellow!35}10.31 & 8.03 & \cellcolor{orange!25}9.15 & \textbf{\cellcolor{green!30}12.60} \\
         &  & SSIM $\uparrow$ & \cellcolor{orange!25}0.198 & \cellcolor{yellow!35}0.278 & 0.140 & 0.146 & \textbf{\cellcolor{green!30}0.322} & \cellcolor{orange!25}0.159 & \cellcolor{yellow!35}0.220 & 0.101 & 0.132 & \textbf{\cellcolor{green!30}0.289} & 0.092 & \cellcolor{yellow!35}0.141 & 0.052 & \cellcolor{orange!25}0.110 & \textbf{\cellcolor{green!30}0.241} \\
         &  & LPIPS $\downarrow$ & \cellcolor{orange!25}0.642 & \cellcolor{yellow!35}0.584 & 0.671 & 0.682 & \textbf{\cellcolor{green!30}0.581} & \cellcolor{orange!25}0.680 & \cellcolor{yellow!35}0.625 & 0.711 & 0.699 & \textbf{\cellcolor{green!30}0.617} & 0.761 & \cellcolor{yellow!35}0.690 & 0.771 & \cellcolor{orange!25}0.727 & \textbf{\cellcolor{green!30}0.687} \\
         & \multirow{3}{*}{Treehill} & PSNR $\uparrow$ & \cellcolor{yellow!35}14.50 & \cellcolor{orange!25}14.18 & 9.98 & 11.99 & \textbf{\cellcolor{green!30}17.13} & \cellcolor{orange!25}13.01 & \cellcolor{yellow!35}13.27 & 9.56 & 11.72 & \textbf{\cellcolor{green!30}16.51} & 10.24 & \cellcolor{yellow!35}11.61 & 9.05 & \cellcolor{orange!25}10.62 & \textbf{\cellcolor{green!30}15.51} \\
         &  & SSIM $\uparrow$ & 0.308 & \cellcolor{yellow!35}0.384 & 0.118 & \cellcolor{orange!25}0.324 & \textbf{\cellcolor{green!30}0.457} & 0.265 & \cellcolor{yellow!35}0.336 & 0.089 & \cellcolor{orange!25}0.306 & \textbf{\cellcolor{green!30}0.422} & 0.159 & \cellcolor{orange!25}0.242 & 0.059 & \cellcolor{yellow!35}0.245 & \textbf{\cellcolor{green!30}0.380} \\
         &  & LPIPS $\downarrow$ & \cellcolor{orange!25}0.594 & \cellcolor{yellow!35}0.577 & 0.634 & 0.604 & \textbf{\cellcolor{green!30}0.537} & 0.622 & \cellcolor{yellow!35}0.608 & 0.655 & \cellcolor{orange!25}0.619 & \textbf{\cellcolor{green!30}0.571} & 0.674 & \cellcolor{yellow!35}0.654 & 0.682 & \cellcolor{orange!25}0.655 & \textbf{\cellcolor{green!30}0.626} \\
         & \multirow{3}{*}{Bicycle} & PSNR $\uparrow$ & 13.66 & \cellcolor{yellow!35}14.91 & 10.97 & \cellcolor{orange!25}13.66 & \textbf{\cellcolor{green!30}16.45} & 12.13 & \cellcolor{yellow!35}13.80 & 10.34 & \cellcolor{orange!25}13.46 & \textbf{\cellcolor{green!30}15.60} & 10.38 & \cellcolor{orange!25}11.78 & 9.70 & \cellcolor{yellow!35}13.07 & \textbf{\cellcolor{green!30}14.40} \\
         &  & SSIM $\uparrow$ & 0.254 & \cellcolor{yellow!35}0.365 & 0.145 & \cellcolor{orange!25}0.278 & \textbf{\cellcolor{green!30}0.416} & 0.178 & \cellcolor{yellow!35}0.304 & 0.097 & \cellcolor{orange!25}0.263 & \textbf{\cellcolor{green!30}0.376} & 0.081 & \cellcolor{orange!25}0.192 & 0.042 & \cellcolor{yellow!35}0.234 & \textbf{\cellcolor{green!30}0.318} \\
         &  & LPIPS $\downarrow$ & \cellcolor{orange!25}0.560 & \cellcolor{yellow!35}0.510 & 0.594 & 0.587 & \textbf{\cellcolor{green!30}0.492} & 0.604 & \cellcolor{yellow!35}0.554 & 0.635 & \cellcolor{orange!25}0.602 & \textbf{\cellcolor{green!30}0.534} & 0.679 & \cellcolor{yellow!35}0.627 & 0.693 & \cellcolor{orange!25}0.633 & \textbf{\cellcolor{green!30}0.611} \\
         & \multirow{3}{*}{Bonsai} & PSNR $\uparrow$ & \cellcolor{orange!25}14.93 & 11.75 & \cellcolor{yellow!35}16.03 & 10.60 & \textbf{\cellcolor{green!30}16.94} & \cellcolor{orange!25}12.94 & 10.73 & \cellcolor{yellow!35}14.08 & 10.11 & \textbf{\cellcolor{green!30}15.98} & \cellcolor{orange!25}10.71 & 9.75 & \cellcolor{yellow!35}11.42 & 9.60 & \textbf{\cellcolor{green!30}14.36} \\
         &  & SSIM $\uparrow$ & \cellcolor{orange!25}0.474 & 0.415 & \cellcolor{yellow!35}0.488 & 0.245 & \textbf{\cellcolor{green!30}0.585} & \cellcolor{orange!25}0.380 & 0.310 & \cellcolor{yellow!35}0.397 & 0.186 & \textbf{\cellcolor{green!30}0.545} & \cellcolor{orange!25}0.219 & 0.171 & \cellcolor{yellow!35}0.231 & 0.093 & \textbf{\cellcolor{green!30}0.474} \\
         &  & LPIPS $\downarrow$ & 0.527 & \cellcolor{orange!25}0.497 & \cellcolor{yellow!35}0.475 & 0.534 & \textbf{\cellcolor{green!30}0.468} & 0.561 & \cellcolor{orange!25}0.533 & \cellcolor{yellow!35}0.514 & 0.552 & \textbf{\cellcolor{green!30}0.499} & 0.605 & \cellcolor{orange!25}0.568 & \cellcolor{yellow!35}0.564 & 0.573 & \textbf{\cellcolor{green!30}0.546} \\
         & \multirow{3}{*}{Counter} & PSNR $\uparrow$ & \cellcolor{yellow!35}15.24 & 12.51 & \cellcolor{orange!25}13.95 & 11.62 & \textbf{\cellcolor{green!30}17.49} & \cellcolor{yellow!35}13.52 & 11.47 & \cellcolor{orange!25}12.42 & 10.94 & \textbf{\cellcolor{green!30}16.07} & \cellcolor{yellow!35}11.00 & 10.09 & \cellcolor{orange!25}10.23 & 9.64 & \textbf{\cellcolor{green!30}13.78} \\
         &  & SSIM $\uparrow$ & \cellcolor{orange!25}0.498 & \cellcolor{yellow!35}0.499 & 0.471 & 0.363 & \textbf{\cellcolor{green!30}0.647} & \cellcolor{yellow!35}0.438 & \cellcolor{orange!25}0.423 & 0.390 & 0.281 & \textbf{\cellcolor{green!30}0.606} & \cellcolor{yellow!35}0.303 & \cellcolor{orange!25}0.281 & 0.238 & 0.148 & \textbf{\cellcolor{green!30}0.536} \\
         &  & LPIPS $\downarrow$ & 0.518 & \cellcolor{yellow!35}0.477 & \cellcolor{orange!25}0.477 & 0.531 & \textbf{\cellcolor{green!30}0.443} & 0.556 & \cellcolor{yellow!35}0.521 & \cellcolor{orange!25}0.523 & 0.563 & \textbf{\cellcolor{green!30}0.485} & 0.616 & \cellcolor{yellow!35}0.589 & \cellcolor{orange!25}0.594 & 0.611 & \textbf{\cellcolor{green!30}0.553} \\
         & \multirow{3}{*}{Garden} & PSNR $\uparrow$ & 12.45 & \cellcolor{orange!25}12.48 & 9.21 & \cellcolor{yellow!35}13.92 & \textbf{\cellcolor{green!30}15.25} & 10.77 & \cellcolor{orange!25}11.26 & 8.42 & \cellcolor{yellow!35}13.28 & \textbf{\cellcolor{green!30}14.29} & 8.48 & \cellcolor{orange!25}9.36 & 7.53 & \cellcolor{yellow!35}10.13 & \textbf{\cellcolor{green!30}12.58} \\
         &  & SSIM $\uparrow$ & 0.253 & \cellcolor{orange!25}0.330 & 0.165 & \cellcolor{yellow!35}0.348 & \textbf{\cellcolor{green!30}0.387} & 0.190 & \cellcolor{orange!25}0.266 & 0.113 & \cellcolor{yellow!35}0.306 & \textbf{\cellcolor{green!30}0.342} & 0.091 & \cellcolor{orange!25}0.163 & 0.053 & \cellcolor{yellow!35}0.182 & \textbf{\cellcolor{green!30}0.276} \\
         &  & LPIPS $\downarrow$ & 0.556 & \cellcolor{yellow!35}0.509 & 0.604 & \cellcolor{orange!25}0.537 & \textbf{\cellcolor{green!30}0.503} & 0.607 & \cellcolor{yellow!35}0.566 & 0.649 & \cellcolor{orange!25}0.573 & \textbf{\cellcolor{green!30}0.555} & 0.700 & \cellcolor{yellow!35}0.663 & 0.724 & \cellcolor{orange!25}0.668 & \textbf{\cellcolor{green!30}0.647} \\
         & \multirow{3}{*}{Kitchen} & PSNR $\uparrow$ & \cellcolor{yellow!35}12.86 & 10.12 & \cellcolor{orange!25}10.60 & 9.37 & \textbf{\cellcolor{green!30}14.44} & \cellcolor{yellow!35}11.06 & 8.81 & \cellcolor{orange!25}9.33 & 7.91 & \textbf{\cellcolor{green!30}13.44} & \cellcolor{yellow!35}8.31 & 7.16 & \cellcolor{orange!25}7.54 & 6.48 & \textbf{\cellcolor{green!30}12.02} \\
         &  & SSIM $\uparrow$ & \cellcolor{yellow!35}0.404 & 0.392 & \cellcolor{orange!25}0.402 & 0.292 & \textbf{\cellcolor{green!30}0.556} & \cellcolor{yellow!35}0.326 & 0.304 & \cellcolor{orange!25}0.322 & 0.198 & \textbf{\cellcolor{green!30}0.502} & \cellcolor{yellow!35}0.205 & 0.168 & \cellcolor{orange!25}0.193 & 0.080 & \textbf{\cellcolor{green!30}0.421} \\
         &  & LPIPS $\downarrow$ & 0.544 & \cellcolor{yellow!35}0.506 & \cellcolor{orange!25}0.512 & 0.556 & \textbf{\cellcolor{green!30}0.448} & 0.589 & \cellcolor{yellow!35}0.576 & \cellcolor{orange!25}0.581 & 0.618 & \textbf{\cellcolor{green!30}0.506} & \cellcolor{yellow!35}0.662 & \cellcolor{orange!25}0.676 & 0.680 & 0.684 & \textbf{\cellcolor{green!30}0.608} \\
         & \multirow{3}{*}{Room} & PSNR $\uparrow$ & \cellcolor{orange!25}16.33 & \cellcolor{yellow!35}16.89 & 15.97 & 12.36 & \textbf{\cellcolor{green!30}18.18} & \cellcolor{orange!25}14.36 & \cellcolor{yellow!35}15.27 & 13.79 & 11.49 & \textbf{\cellcolor{green!30}17.12} & \cellcolor{orange!25}11.29 & \cellcolor{yellow!35}12.71 & 10.74 & 9.76 & \textbf{\cellcolor{green!30}15.29} \\
         &  & SSIM $\uparrow$ & \cellcolor{orange!25}0.585 & \cellcolor{yellow!35}0.672 & 0.568 & 0.459 & \textbf{\cellcolor{green!30}0.695} & \cellcolor{orange!25}0.514 & \cellcolor{yellow!35}0.607 & 0.481 & 0.391 & \textbf{\cellcolor{green!30}0.665} & \cellcolor{orange!25}0.352 & \cellcolor{yellow!35}0.469 & 0.290 & 0.244 & \textbf{\cellcolor{green!30}0.598} \\
         &  & LPIPS $\downarrow$ & 0.501 & \textbf{\cellcolor{green!30}0.428} & \cellcolor{orange!25}0.454 & 0.536 & \cellcolor{yellow!35}0.434 & 0.543 & \cellcolor{yellow!35}0.475 & \cellcolor{orange!25}0.501 & 0.567 & \textbf{\cellcolor{green!30}0.471} & 0.613 & \cellcolor{yellow!35}0.551 & \cellcolor{orange!25}0.583 & 0.618 & \textbf{\cellcolor{green!30}0.523} \\
         & \multirow{3}{*}{Stump} & PSNR $\uparrow$ & \cellcolor{orange!25}15.31 & \cellcolor{yellow!35}16.14 & 13.45 & 15.07 & \textbf{\cellcolor{green!30}18.02} & 14.06 & \cellcolor{yellow!35}14.87 & 12.77 & \cellcolor{orange!25}14.69 & \textbf{\cellcolor{green!30}16.85} & 12.38 & \cellcolor{orange!25}13.27 & 11.99 & \cellcolor{yellow!35}14.00 & \textbf{\cellcolor{green!30}15.19} \\
         &  & SSIM $\uparrow$ & 0.273 & \cellcolor{yellow!35}0.409 & 0.180 & \cellcolor{orange!25}0.313 & \textbf{\cellcolor{green!30}0.491} & 0.205 & \cellcolor{yellow!35}0.330 & 0.125 & \cellcolor{orange!25}0.286 & \textbf{\cellcolor{green!30}0.421} & 0.090 & \cellcolor{orange!25}0.215 & 0.064 & \cellcolor{yellow!35}0.235 & \textbf{\cellcolor{green!30}0.321} \\
         &  & LPIPS $\downarrow$ & 0.558 & \cellcolor{yellow!35}0.482 & 0.575 & \cellcolor{orange!25}0.553 & \textbf{\cellcolor{green!30}0.455} & 0.592 & \cellcolor{yellow!35}0.527 & 0.611 & \cellcolor{orange!25}0.572 & \textbf{\cellcolor{green!30}0.508} & 0.664 & \cellcolor{yellow!35}0.598 & 0.660 & \cellcolor{orange!25}0.610 & \textbf{\cellcolor{green!30}0.594} \\
        \cmidrule(lr){2-18}
         & \multirow{3}{*}{\textbf{Mip-NeRF 360}} & PSNR $\uparrow$ & \cellcolor{yellow!35}14.15 & \cellcolor{orange!25}13.49 & 12.19 & 12.02 & \textbf{\cellcolor{green!30}16.52} & \cellcolor{yellow!35}12.53 & \cellcolor{orange!25}12.33 & 11.06 & 11.45 & \textbf{\cellcolor{green!30}15.53} & 10.19 & \cellcolor{yellow!35}10.67 & 9.58 & \cellcolor{orange!25}10.27 & \textbf{\cellcolor{green!30}13.97} \\
         &  & SSIM $\uparrow$ & \cellcolor{orange!25}0.361 & \cellcolor{yellow!35}0.416 & 0.297 & 0.308 & \textbf{\cellcolor{green!30}0.506} & \cellcolor{orange!25}0.295 & \cellcolor{yellow!35}0.344 & 0.235 & 0.261 & \textbf{\cellcolor{green!30}0.463} & \cellcolor{orange!25}0.177 & \cellcolor{yellow!35}0.227 & 0.136 & 0.174 & \textbf{\cellcolor{green!30}0.396} \\
         &  & LPIPS $\downarrow$ & 0.556 & \cellcolor{yellow!35}0.508 & \cellcolor{orange!25}0.555 & 0.569 & \textbf{\cellcolor{green!30}0.484} & \cellcolor{orange!25}0.595 & \cellcolor{yellow!35}0.554 & 0.598 & 0.596 & \textbf{\cellcolor{green!30}0.527} & 0.664 & \cellcolor{yellow!35}0.624 & 0.661 & \cellcolor{orange!25}0.642 & \textbf{\cellcolor{green!30}0.600} \\
        \specialrule{0.2pt}{1pt}{1pt}
        \rdelim\{{6}{*} & \multirow{3}{*}{DrJohnson} & PSNR $\uparrow$ & \cellcolor{yellow!35}16.95 & \cellcolor{orange!25}15.64 & 13.01 & 15.00 & \textbf{\cellcolor{green!30}20.17} & \cellcolor{yellow!35}14.77 & \cellcolor{orange!25}13.67 & 11.32 & 12.72 & \textbf{\cellcolor{green!30}18.83} & \cellcolor{yellow!35}11.30 & \cellcolor{orange!25}10.94 & 9.39 & 9.47 & \textbf{\cellcolor{green!30}16.20} \\
         &  & SSIM $\uparrow$ & \cellcolor{orange!25}0.630 & \cellcolor{yellow!35}0.645 & 0.413 & 0.597 & \textbf{\cellcolor{green!30}0.733} & \cellcolor{orange!25}0.557 & \cellcolor{yellow!35}0.561 & 0.288 & 0.492 & \textbf{\cellcolor{green!30}0.703} & \cellcolor{orange!25}0.360 & \cellcolor{yellow!35}0.372 & 0.111 & 0.194 & \textbf{\cellcolor{green!30}0.630} \\
         &  & LPIPS $\downarrow$ & 0.502 & \cellcolor{yellow!35}0.467 & \cellcolor{orange!25}0.493 & 0.510 & \textbf{\cellcolor{green!30}0.439} & 0.536 & \cellcolor{yellow!35}0.510 & \cellcolor{orange!25}0.532 & 0.545 & \textbf{\cellcolor{green!30}0.468} & 0.592 & \cellcolor{yellow!35}0.570 & \cellcolor{orange!25}0.580 & 0.602 & \textbf{\cellcolor{green!30}0.516} \\
         & \multirow{3}{*}{Playroom} & PSNR $\uparrow$ & \cellcolor{yellow!35}15.98 & \cellcolor{orange!25}15.77 & 11.66 & 9.93 & \textbf{\cellcolor{green!30}19.94} & \cellcolor{orange!25}13.19 & \cellcolor{yellow!35}13.34 & 8.96 & 8.33 & \textbf{\cellcolor{green!30}18.39} & \cellcolor{orange!25}8.98 & \cellcolor{yellow!35}9.33 & 5.92 & 6.22 & \textbf{\cellcolor{green!30}15.04} \\
         &  & SSIM $\uparrow$ & \cellcolor{orange!25}0.677 & \cellcolor{yellow!35}0.713 & 0.529 & 0.451 & \textbf{\cellcolor{green!30}0.771} & \cellcolor{orange!25}0.609 & \cellcolor{yellow!35}0.656 & 0.405 & 0.359 & \textbf{\cellcolor{green!30}0.748} & \cellcolor{orange!25}0.442 & \cellcolor{yellow!35}0.485 & 0.152 & 0.184 & \textbf{\cellcolor{green!30}0.690} \\
         &  & LPIPS $\downarrow$ & \cellcolor{orange!25}0.510 & \cellcolor{yellow!35}0.468 & 0.557 & 0.644 & \textbf{\cellcolor{green!30}0.439} & \cellcolor{orange!25}0.573 & \cellcolor{yellow!35}0.531 & 0.650 & 0.700 & \textbf{\cellcolor{green!30}0.474} & \cellcolor{orange!25}0.715 & \cellcolor{yellow!35}0.672 & 0.787 & 0.791 & \textbf{\cellcolor{green!30}0.537} \\
        \cmidrule(lr){2-18}
         & \multirow{3}{*}{\textbf{Deep Blending}} & PSNR $\uparrow$ & \cellcolor{yellow!35}16.46 & \cellcolor{orange!25}15.70 & 12.34 & 12.47 & \textbf{\cellcolor{green!30}20.06} & \cellcolor{yellow!35}13.98 & \cellcolor{orange!25}13.51 & 10.14 & 10.53 & \textbf{\cellcolor{green!30}18.61} & \cellcolor{yellow!35}10.14 & \cellcolor{orange!25}10.13 & 7.66 & 7.84 & \textbf{\cellcolor{green!30}15.62} \\
         &  & SSIM $\uparrow$ & \cellcolor{orange!25}0.653 & \cellcolor{yellow!35}0.679 & 0.471 & 0.524 & \textbf{\cellcolor{green!30}0.752} & \cellcolor{orange!25}0.583 & \cellcolor{yellow!35}0.608 & 0.347 & 0.426 & \textbf{\cellcolor{green!30}0.725} & \cellcolor{orange!25}0.401 & \cellcolor{yellow!35}0.428 & 0.132 & 0.189 & \textbf{\cellcolor{green!30}0.660} \\
         &  & LPIPS $\downarrow$ & \cellcolor{orange!25}0.506 & \cellcolor{yellow!35}0.468 & 0.525 & 0.577 & \textbf{\cellcolor{green!30}0.439} & \cellcolor{orange!25}0.554 & \cellcolor{yellow!35}0.520 & 0.591 & 0.623 & \textbf{\cellcolor{green!30}0.471} & \cellcolor{orange!25}0.653 & \cellcolor{yellow!35}0.621 & 0.683 & 0.696 & \textbf{\cellcolor{green!30}0.526} \\
        \specialrule{0.2pt}{1pt}{1pt}
        \rdelim\{{6}{*} & \multirow{3}{*}{Train} & PSNR $\uparrow$ & \cellcolor{yellow!35}11.81 & \cellcolor{orange!25}10.18 & 9.72 & 7.95 & \textbf{\cellcolor{green!30}14.01} & \cellcolor{yellow!35}10.32 & \cellcolor{orange!25}8.99 & 8.40 & 7.38 & \textbf{\cellcolor{green!30}13.11} & \cellcolor{yellow!35}8.00 & \cellcolor{orange!25}7.43 & 6.84 & 6.49 & \textbf{\cellcolor{green!30}11.41} \\
         &  & SSIM $\uparrow$ & \cellcolor{orange!25}0.368 & \cellcolor{yellow!35}0.390 & 0.322 & 0.255 & \textbf{\cellcolor{green!30}0.481} & \cellcolor{orange!25}0.310 & \cellcolor{yellow!35}0.321 & 0.240 & 0.201 & \textbf{\cellcolor{green!30}0.449} & \cellcolor{orange!25}0.212 & \cellcolor{yellow!35}0.221 & 0.135 & 0.120 & \textbf{\cellcolor{green!30}0.400} \\
         &  & LPIPS $\downarrow$ & 0.579 & \cellcolor{yellow!35}0.561 & \cellcolor{orange!25}0.561 & 0.594 & \textbf{\cellcolor{green!30}0.523} & 0.621 & \cellcolor{yellow!35}0.604 & \cellcolor{orange!25}0.609 & 0.624 & \textbf{\cellcolor{green!30}0.561} & 0.683 & \cellcolor{orange!25}0.665 & 0.670 & \cellcolor{yellow!35}0.665 & \textbf{\cellcolor{green!30}0.620} \\
         & \multirow{3}{*}{Truck} & PSNR $\uparrow$ & \cellcolor{orange!25}12.71 & \cellcolor{yellow!35}12.98 & 9.35 & 10.93 & \textbf{\cellcolor{green!30}15.62} & \cellcolor{orange!25}10.94 & \cellcolor{yellow!35}11.96 & 7.74 & 9.83 & \textbf{\cellcolor{green!30}14.54} & \cellcolor{orange!25}7.88 & \cellcolor{yellow!35}10.46 & 5.62 & 7.34 & \textbf{\cellcolor{green!30}12.55} \\
         &  & SSIM $\uparrow$ & 0.414 & \cellcolor{yellow!35}0.566 & 0.323 & \cellcolor{orange!25}0.490 & \textbf{\cellcolor{green!30}0.590} & 0.353 & \cellcolor{yellow!35}0.506 & 0.228 & \cellcolor{orange!25}0.393 & \textbf{\cellcolor{green!30}0.544} & \cellcolor{orange!25}0.237 & \cellcolor{yellow!35}0.420 & 0.074 & 0.207 & \textbf{\cellcolor{green!30}0.470} \\
         &  & LPIPS $\downarrow$ & 0.534 & \textbf{\cellcolor{green!30}0.432} & 0.538 & \cellcolor{orange!25}0.483 & \cellcolor{yellow!35}0.436 & 0.582 & \cellcolor{yellow!35}0.495 & 0.588 & \cellcolor{orange!25}0.555 & \textbf{\cellcolor{green!30}0.490} & 0.667 & \textbf{\cellcolor{green!30}0.588} & 0.666 & \cellcolor{orange!25}0.651 & \cellcolor{yellow!35}0.592 \\
        \cmidrule(lr){2-18}
         & \multirow{3}{*}{\textbf{Tanks \& Temples}} & PSNR $\uparrow$ & \cellcolor{yellow!35}12.26 & \cellcolor{orange!25}11.58 & 9.53 & 9.44 & \textbf{\cellcolor{green!30}14.81} & \cellcolor{yellow!35}10.63 & \cellcolor{orange!25}10.48 & 8.07 & 8.60 & \textbf{\cellcolor{green!30}13.83} & \cellcolor{orange!25}7.94 & \cellcolor{yellow!35}8.94 & 6.23 & 6.92 & \textbf{\cellcolor{green!30}11.98} \\
         &  & SSIM $\uparrow$ & \cellcolor{orange!25}0.391 & \cellcolor{yellow!35}0.478 & 0.323 & 0.372 & \textbf{\cellcolor{green!30}0.535} & \cellcolor{orange!25}0.332 & \cellcolor{yellow!35}0.414 & 0.234 & 0.297 & \textbf{\cellcolor{green!30}0.496} & \cellcolor{orange!25}0.224 & \cellcolor{yellow!35}0.320 & 0.104 & 0.163 & \textbf{\cellcolor{green!30}0.435} \\
         &  & LPIPS $\downarrow$ & 0.557 & \cellcolor{yellow!35}0.496 & 0.550 & \cellcolor{orange!25}0.538 & \textbf{\cellcolor{green!30}0.480} & 0.601 & \cellcolor{yellow!35}0.549 & 0.599 & \cellcolor{orange!25}0.590 & \textbf{\cellcolor{green!30}0.525} & 0.675 & \cellcolor{yellow!35}0.627 & 0.668 & \cellcolor{orange!25}0.658 & \textbf{\cellcolor{green!30}0.606} \\
        \bottomrule
    \end{tabular}}
\end{table*}

\begin{table*}[t]
    \centering
    \caption{Per-dataset average prune-only results at prune ratios 0.80, 0.85, 0.90, 0.95, 0.97, and 0.99. Rows are prune ratios and methods; columns are grouped by dataset and metric. Bold indicates best per dataset/metric/ratio; colored cells mark the top three methods.}
    \label{tab:prune_only_dataset_averages_all_ratios}
    \tiny
    \setlength{\tabcolsep}{2.2pt}
    \resizebox{\textwidth}{!}{%
    \begin{tabular}{@{}ll !{\vrule width 0.8pt} c c c c c c c c c}
        \toprule
        \multirow{2}{*}{Prune} & \multirow{2}{*}{Method} & \multicolumn{3}{c}{Mip-NeRF 360} & \multicolumn{3}{c}{Deep Blending} & \multicolumn{3}{c}{Tanks \& Temples} \\
        \cmidrule(lr){3-5}\cmidrule(lr){6-8}\cmidrule(lr){9-11}
        & & PSNR $\uparrow$ & SSIM $\uparrow$ & LPIPS $\downarrow$ & PSNR $\uparrow$ & SSIM $\uparrow$ & LPIPS $\downarrow$ & PSNR $\uparrow$ & SSIM $\uparrow$ & LPIPS $\downarrow$ \\
        \midrule
        \multirow{5}{*}{0.80} & GHAP & \cellcolor{yellow!35}19.78 & 0.564 & 0.425 & \textbf{\cellcolor{green!30}24.67} & \cellcolor{orange!25}0.820 & 0.365 & \cellcolor{yellow!35}17.61 & 0.606 & 0.391 \\
         & PUP & \cellcolor{orange!25}19.40 & \textbf{\cellcolor{green!30}0.665} & \textbf{\cellcolor{green!30}0.337} & \cellcolor{orange!25}24.29 & \textbf{\cellcolor{green!30}0.841} & \textbf{\cellcolor{green!30}0.317} & 16.05 & \textbf{\cellcolor{green!30}0.687} & \textbf{\cellcolor{green!30}0.310} \\
         & Trim & 18.01 & \cellcolor{orange!25}0.569 & \cellcolor{orange!25}0.384 & 22.02 & 0.786 & \cellcolor{yellow!35}0.335 & \cellcolor{orange!25}16.11 & \cellcolor{orange!25}0.632 & \cellcolor{orange!25}0.343 \\
         & Unif & 14.85 & 0.487 & 0.450 & 19.51 & 0.752 & 0.426 & 11.74 & 0.526 & 0.423 \\
         & Ours & \textbf{\cellcolor{green!30}20.18} & \cellcolor{yellow!35}0.664 & \cellcolor{yellow!35}0.343 & \cellcolor{yellow!35}24.41 & \cellcolor{yellow!35}0.831 & \cellcolor{orange!25}0.336 & \textbf{\cellcolor{green!30}18.13} & \cellcolor{yellow!35}0.675 & \cellcolor{yellow!35}0.327 \\
        \cmidrule(lr){1-11}
        \multirow{5}{*}{0.85} & GHAP & \cellcolor{yellow!35}18.54 & \cellcolor{orange!25}0.521 & 0.454 & \cellcolor{yellow!35}23.05 & \cellcolor{orange!25}0.791 & 0.395 & \cellcolor{yellow!35}16.42 & 0.555 & 0.431 \\
         & PUP & \cellcolor{orange!25}17.75 & \cellcolor{yellow!35}0.612 & \cellcolor{yellow!35}0.375 & \cellcolor{orange!25}22.25 & \cellcolor{yellow!35}0.812 & \textbf{\cellcolor{green!30}0.348} & \cellcolor{orange!25}14.89 & \textbf{\cellcolor{green!30}0.643} & \textbf{\cellcolor{green!30}0.348} \\
         & Trim & 16.33 & 0.501 & \cellcolor{orange!25}0.426 & 19.63 & 0.733 & \cellcolor{orange!25}0.372 & 14.53 & \cellcolor{orange!25}0.564 & \cellcolor{orange!25}0.392 \\
         & Unif & 14.04 & 0.446 & 0.478 & 17.56 & 0.708 & 0.460 & 11.17 & 0.494 & 0.448 \\
         & Ours & \textbf{\cellcolor{green!30}19.30} & \textbf{\cellcolor{green!30}0.629} & \textbf{\cellcolor{green!30}0.374} & \textbf{\cellcolor{green!30}23.45} & \textbf{\cellcolor{green!30}0.814} & \cellcolor{yellow!35}0.359 & \textbf{\cellcolor{green!30}17.37} & \cellcolor{yellow!35}0.643 & \cellcolor{yellow!35}0.362 \\
        \cmidrule(lr){1-11}
        \multirow{5}{*}{0.90} & GHAP & \cellcolor{yellow!35}16.76 & \cellcolor{orange!25}0.458 & 0.494 & \cellcolor{yellow!35}20.53 & \cellcolor{orange!25}0.744 & 0.437 & \cellcolor{yellow!35}14.83 & \cellcolor{orange!25}0.488 & 0.483 \\
         & PUP & \cellcolor{orange!25}15.79 & \cellcolor{yellow!35}0.535 & \cellcolor{yellow!35}0.429 & \cellcolor{orange!25}19.52 & \cellcolor{yellow!35}0.765 & \cellcolor{yellow!35}0.393 & \cellcolor{orange!25}13.38 & \cellcolor{yellow!35}0.577 & \textbf{\cellcolor{green!30}0.406} \\
         & Trim & 14.45 & 0.413 & \cellcolor{orange!25}0.480 & 16.52 & 0.642 & \cellcolor{orange!25}0.429 & 12.46 & 0.471 & \cellcolor{orange!25}0.458 \\
         & Unif & 13.16 & 0.391 & 0.516 & 15.69 & 0.654 & 0.501 & 10.22 & 0.437 & 0.482 \\
         & Ours & \textbf{\cellcolor{green!30}18.16} & \textbf{\cellcolor{green!30}0.580} & \textbf{\cellcolor{green!30}0.417} & \textbf{\cellcolor{green!30}22.16} & \textbf{\cellcolor{green!30}0.790} & \textbf{\cellcolor{green!30}0.390} & \textbf{\cellcolor{green!30}16.35} & \textbf{\cellcolor{green!30}0.599} & \cellcolor{yellow!35}0.407 \\
        \cmidrule(lr){1-11}
        \multirow{5}{*}{0.95} & GHAP & \cellcolor{yellow!35}14.15 & \cellcolor{orange!25}0.361 & 0.556 & \cellcolor{yellow!35}16.46 & \cellcolor{orange!25}0.653 & \cellcolor{orange!25}0.506 & \cellcolor{yellow!35}12.26 & \cellcolor{orange!25}0.391 & 0.557 \\
         & PUP & \cellcolor{orange!25}13.49 & \cellcolor{yellow!35}0.416 & \cellcolor{yellow!35}0.508 & \cellcolor{orange!25}15.70 & \cellcolor{yellow!35}0.679 & \cellcolor{yellow!35}0.468 & \cellcolor{orange!25}11.58 & \cellcolor{yellow!35}0.478 & \cellcolor{yellow!35}0.496 \\
         & Trim & 12.19 & 0.297 & \cellcolor{orange!25}0.555 & 12.34 & 0.471 & 0.525 & 9.53 & 0.323 & 0.550 \\
         & Unif & 12.02 & 0.308 & 0.569 & 12.47 & 0.524 & 0.577 & 9.44 & 0.372 & \cellcolor{orange!25}0.538 \\
         & Ours & \textbf{\cellcolor{green!30}16.52} & \textbf{\cellcolor{green!30}0.506} & \textbf{\cellcolor{green!30}0.484} & \textbf{\cellcolor{green!30}20.06} & \textbf{\cellcolor{green!30}0.752} & \textbf{\cellcolor{green!30}0.439} & \textbf{\cellcolor{green!30}14.81} & \textbf{\cellcolor{green!30}0.535} & \textbf{\cellcolor{green!30}0.480} \\
        \cmidrule(lr){1-11}
        \multirow{5}{*}{0.97} & GHAP & \cellcolor{yellow!35}12.53 & \cellcolor{orange!25}0.295 & \cellcolor{orange!25}0.595 & \cellcolor{yellow!35}13.98 & \cellcolor{orange!25}0.583 & \cellcolor{orange!25}0.554 & \cellcolor{yellow!35}10.63 & \cellcolor{orange!25}0.332 & 0.601 \\
         & PUP & \cellcolor{orange!25}12.33 & \cellcolor{yellow!35}0.344 & \cellcolor{yellow!35}0.554 & \cellcolor{orange!25}13.51 & \cellcolor{yellow!35}0.608 & \cellcolor{yellow!35}0.520 & \cellcolor{orange!25}10.48 & \cellcolor{yellow!35}0.414 & \cellcolor{yellow!35}0.549 \\
         & Trim & 11.06 & 0.235 & 0.598 & 10.14 & 0.347 & 0.591 & 8.07 & 0.234 & 0.599 \\
         & Unif & 11.45 & 0.261 & 0.596 & 10.53 & 0.426 & 0.623 & 8.60 & 0.297 & \cellcolor{orange!25}0.590 \\
         & Ours & \textbf{\cellcolor{green!30}15.53} & \textbf{\cellcolor{green!30}0.463} & \textbf{\cellcolor{green!30}0.527} & \textbf{\cellcolor{green!30}18.61} & \textbf{\cellcolor{green!30}0.725} & \textbf{\cellcolor{green!30}0.471} & \textbf{\cellcolor{green!30}13.83} & \textbf{\cellcolor{green!30}0.496} & \textbf{\cellcolor{green!30}0.525} \\
        \cmidrule(lr){1-11}
        \multirow{5}{*}{0.99} & GHAP & 10.19 & \cellcolor{orange!25}0.177 & 0.664 & \cellcolor{yellow!35}10.14 & \cellcolor{orange!25}0.401 & \cellcolor{orange!25}0.653 & \cellcolor{orange!25}7.94 & \cellcolor{orange!25}0.224 & 0.675 \\
         & PUP & \cellcolor{yellow!35}10.67 & \cellcolor{yellow!35}0.227 & \cellcolor{yellow!35}0.624 & \cellcolor{orange!25}10.13 & \cellcolor{yellow!35}0.428 & \cellcolor{yellow!35}0.621 & \cellcolor{yellow!35}8.94 & \cellcolor{yellow!35}0.320 & \cellcolor{yellow!35}0.627 \\
         & Trim & 9.58 & 0.136 & 0.661 & 7.66 & 0.132 & 0.683 & 6.23 & 0.104 & 0.668 \\
         & Unif & \cellcolor{orange!25}10.27 & 0.174 & \cellcolor{orange!25}0.642 & 7.84 & 0.189 & 0.696 & 6.92 & 0.163 & \cellcolor{orange!25}0.658 \\
         & Ours & \textbf{\cellcolor{green!30}13.97} & \textbf{\cellcolor{green!30}0.396} & \textbf{\cellcolor{green!30}0.600} & \textbf{\cellcolor{green!30}15.62} & \textbf{\cellcolor{green!30}0.660} & \textbf{\cellcolor{green!30}0.526} & \textbf{\cellcolor{green!30}11.98} & \textbf{\cellcolor{green!30}0.435} & \textbf{\cellcolor{green!30}0.606} \\
        \bottomrule
    \end{tabular}}
\end{table*}

\FloatBarrier

\subsection{Dataset-level prune-only curves and delta trends}
The following plots show how each method degrades as the prune ratio increases. The delta panels use a sign convention in which positive values indicate an advantage for our method for all three metrics.

\begin{figure}[!htbp]
    \centering
    \includegraphics[width=\textwidth]{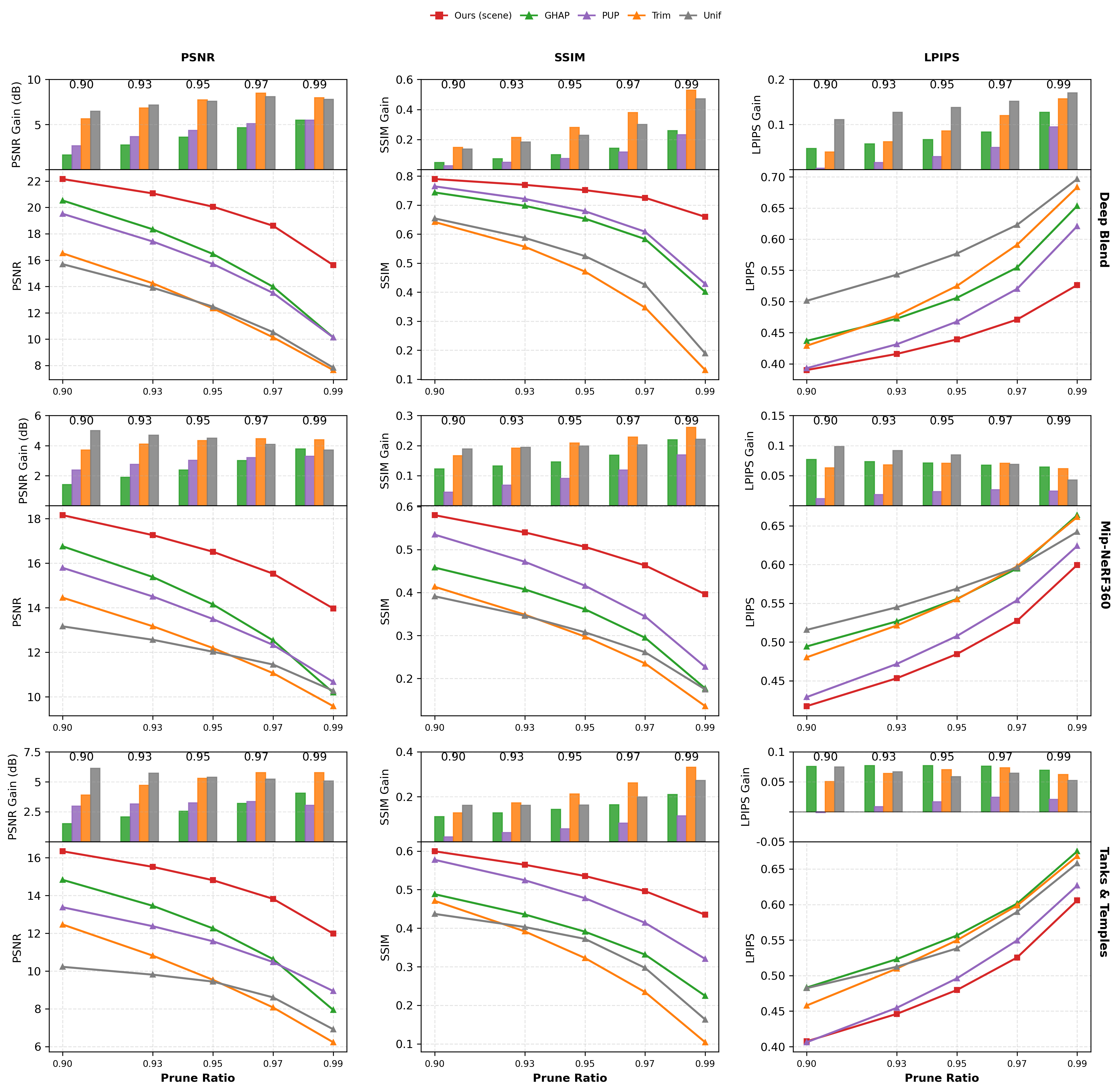}
    \caption{Prune-only averages per dataset.}
    \label{fig:combined_prune_delta_tanks_temples}
\end{figure}

\FloatBarrier

\subsection{Short-recovery results}
The first group fixes our per-scene variant and shows how its quality changes with prune ratio and recovery budget. The second group compares our method with all competing methods at the most aggressive prune ratios.

\begin{figure}[!htbp]
    \centering
    \includegraphics[width=\textwidth]{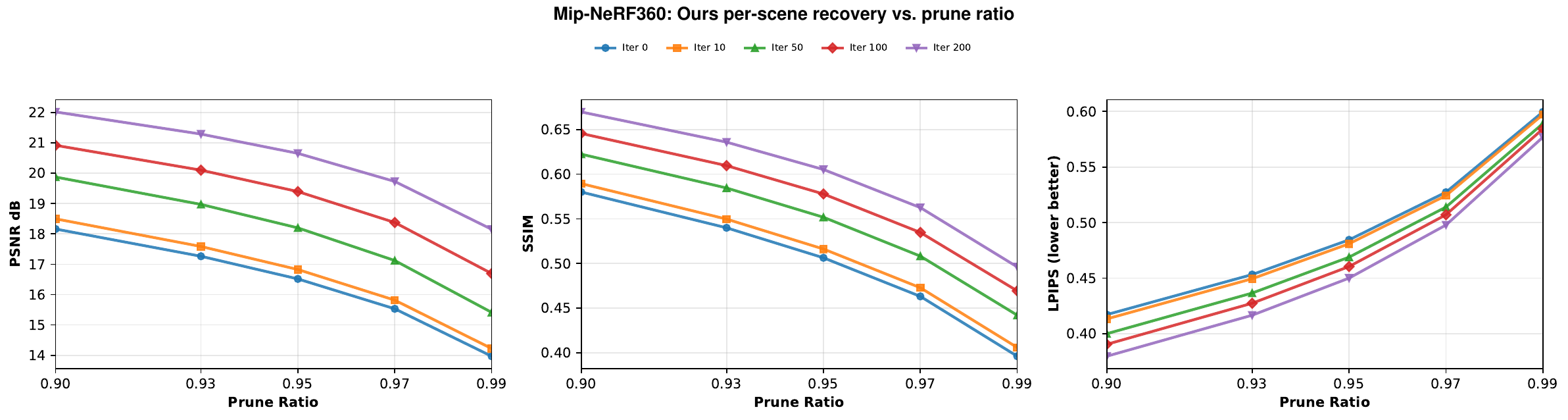}
    \caption{Transposed short-recovery curves for our scene-variant on mip-NeRF 360. The x-axis is prune ratio and each curve corresponds to a recovery iteration.}
    \label{fig:short_recovery_core_prune_axis_mipnerf360}
\end{figure}

\begin{figure}[!htbp]
    \centering
    \includegraphics[width=\textwidth]{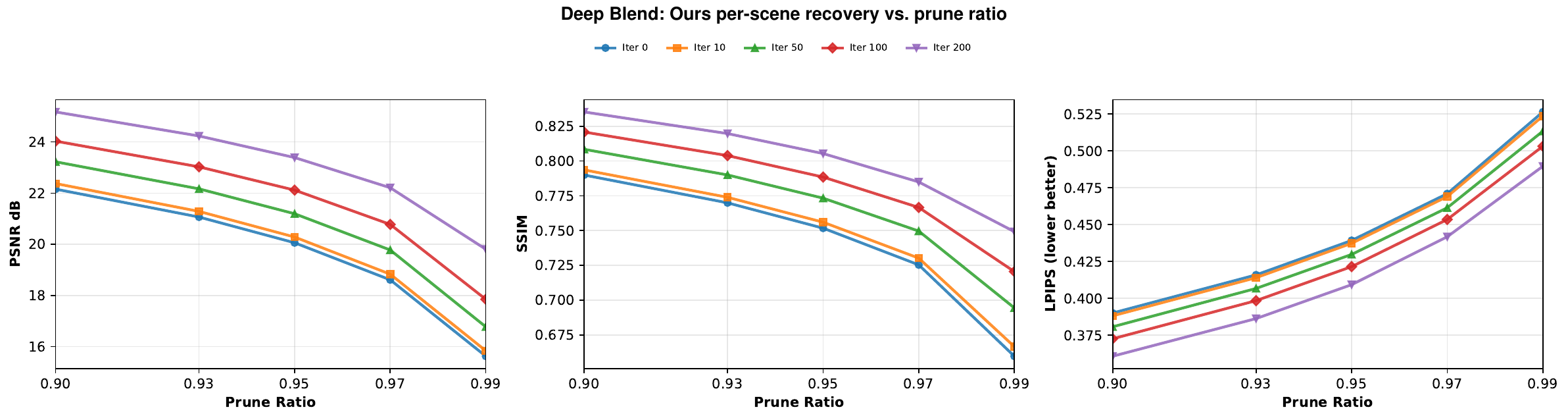}
    \caption{Transposed short-recovery curves for our scene-variant on Deep Blending, using the same layout as \cref{fig:short_recovery_core_prune_axis_mipnerf360}.}
    \label{fig:short_recovery_core_prune_axis_deep_blend}
\end{figure}

\begin{figure}[!htbp]
    \centering
    \includegraphics[width=\textwidth]{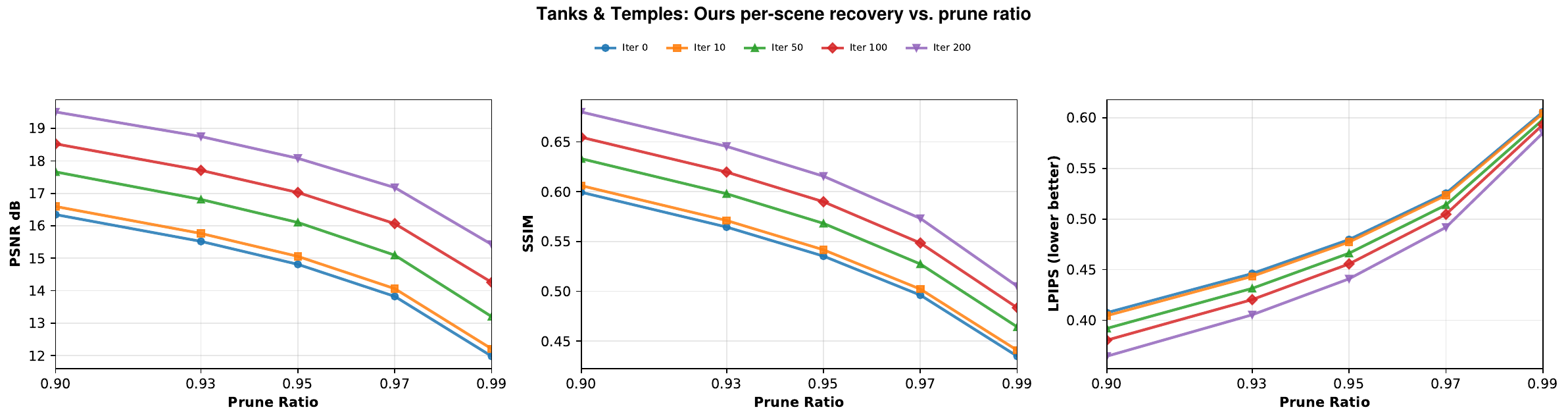}
    \caption{Transposed short-recovery curves for our scene-variant on Tanks \& Temples, using the same layout as \cref{fig:short_recovery_core_prune_axis_mipnerf360}.}
    \label{fig:short_recovery_core_prune_axis_tanks_temples}
\end{figure}

\FloatBarrier

\begin{figure}[p]
    \centering
    \includegraphics[width=0.82\textwidth]{paper_fill/figures/recovery_competitors_mipnerf360.pdf}
    \caption{Short-recovery comparison on mip-NeRF 360 at prune ratios 0.95, 0.97, and 0.99. Curves compare our scene-variant against GHAP, PUP 3D-GS, Trimming the Fat, and Uniform.}
    \label{fig:short_recovery_competitors_mipnerf360}

    \vspace{0.55em}
    \includegraphics[width=0.82\textwidth]{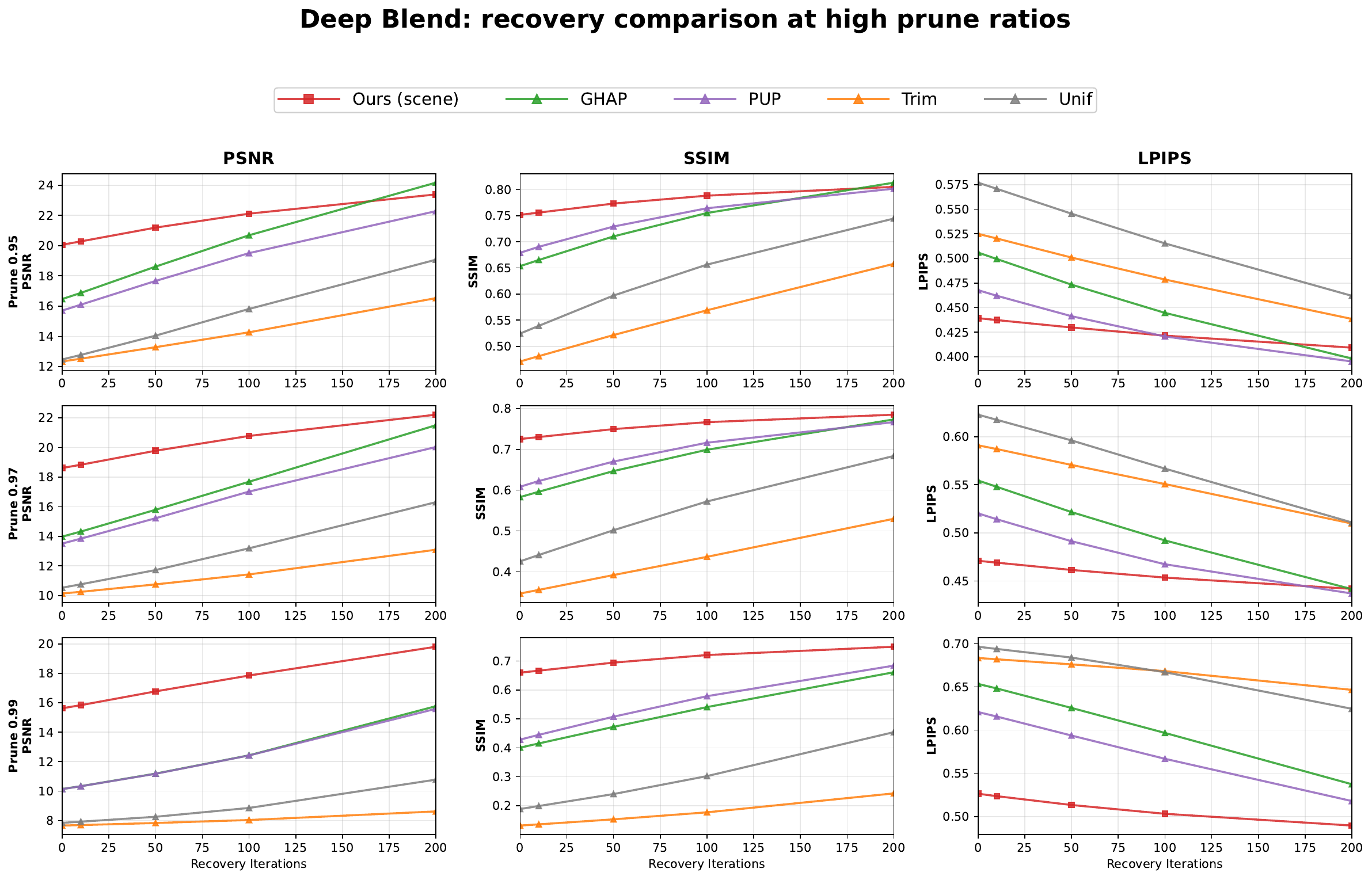}
    \caption{Short-recovery comparison on Deep Blending at prune ratios 0.95, 0.97, and 0.99.}
    \label{fig:short_recovery_competitors_deep_blend}
\end{figure}
\clearpage

\begin{center}
\begin{minipage}{\textwidth}
    \centering
    \includegraphics[width=0.78\textwidth]{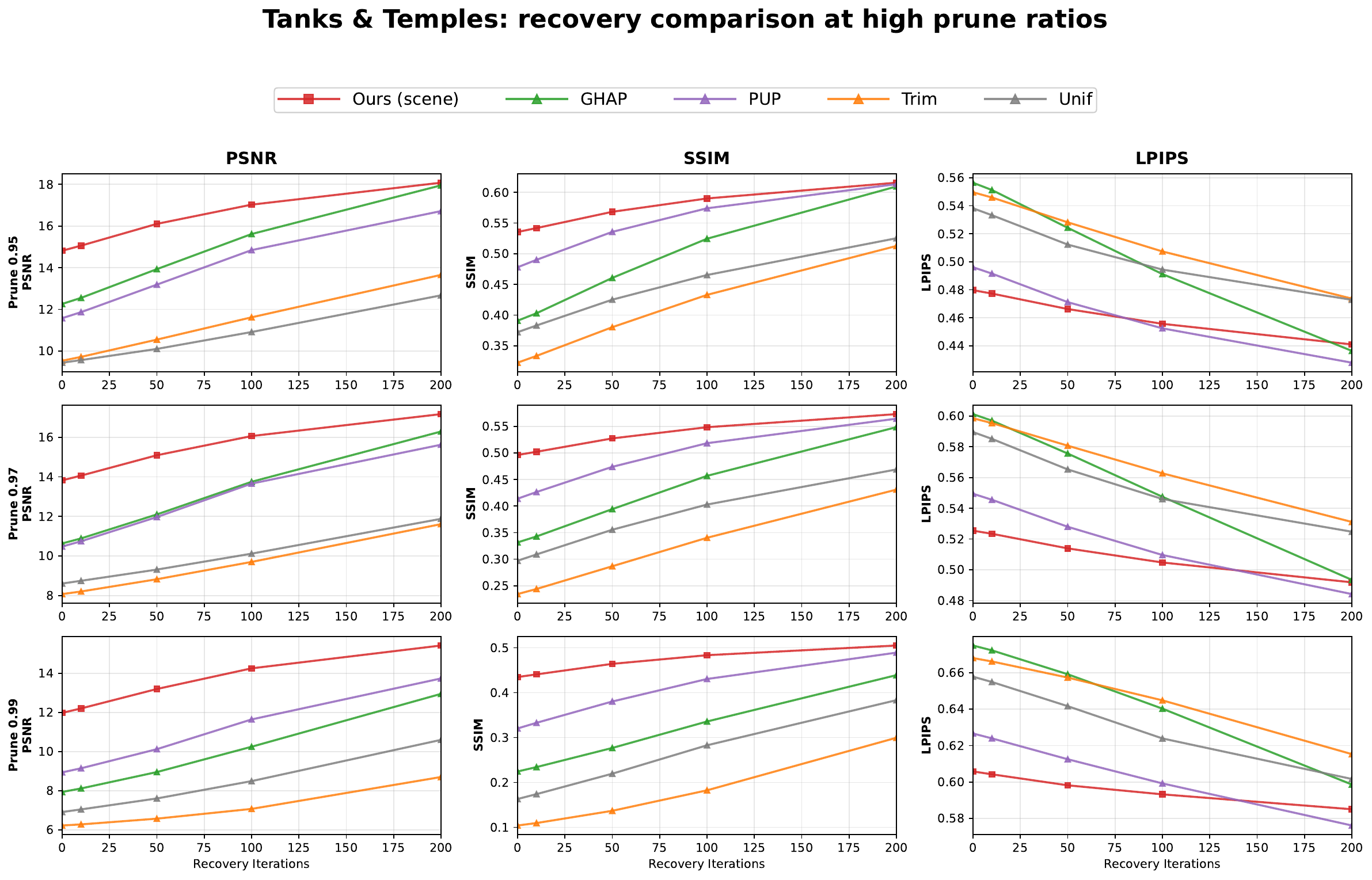}
    \captionof{figure}{Short-recovery comparison on Tanks \& Temples at prune ratios 0.95, 0.97, and 0.99.}
    \label{fig:short_recovery_competitors_tanks_temples}
\end{minipage}
\end{center}

\subsection{Prune-only per-scene graphs}
The dataset averages are expanded below into 39 scene-level panels. These plots show where the average trend is uniform and where individual scenes behave differently.

This appendix section expands the main average-over-scenes prune-only graphs into per-scene plots. The three figures below cover 13 scenes $\times$ 3 metrics, for a total of 39 prune-only panels. Each panel plots metric value on the y-axis against prune ratio on the x-axis, with curves for \method variants and the competing methods. These plots support per-scene variability analysis and show whether averaged trends hide scene-specific failures or strengths.

\begin{center}
\begin{minipage}{\textwidth}
    \centering
    \includegraphics[width=0.78\textwidth]{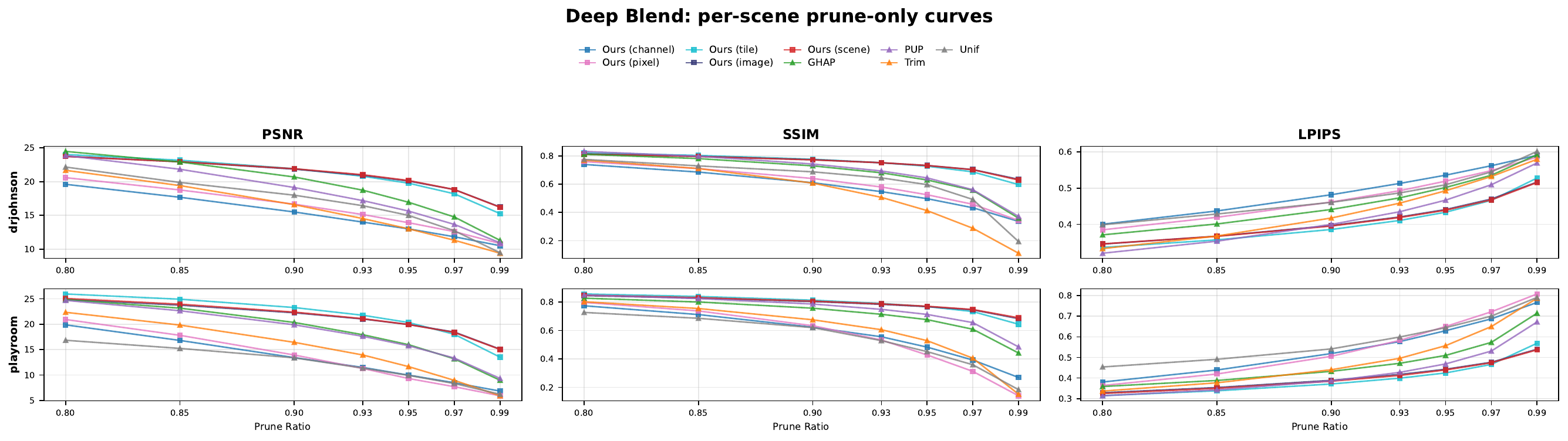}
    \captionof{figure}{Per-scene prune-only curves for Deep Blending.}
    \label{fig:appendix_per_scene_prune_only_deep_blend}

    \vspace{0.35em}
    \includegraphics[width=0.78\textwidth]{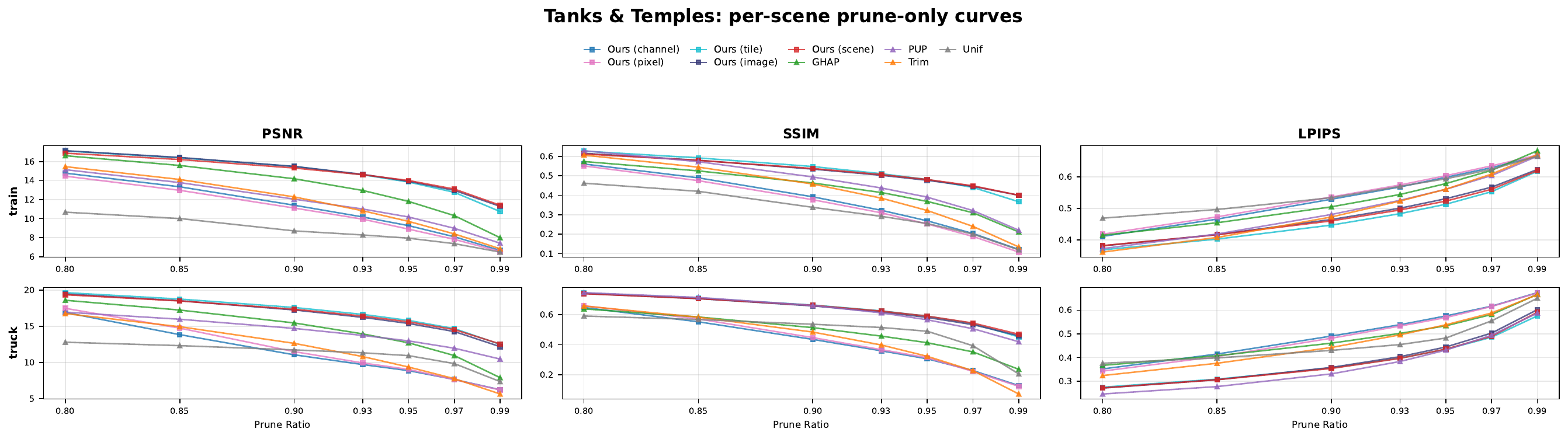}
    \captionof{figure}{Per-scene prune-only curves for Tanks \& Temples.}
    \label{fig:appendix_per_scene_prune_only_tanks_temples}
\end{minipage}
\end{center}

\begin{figure}[!htbp]
    \centering
    \includegraphics[width=0.96\textwidth]{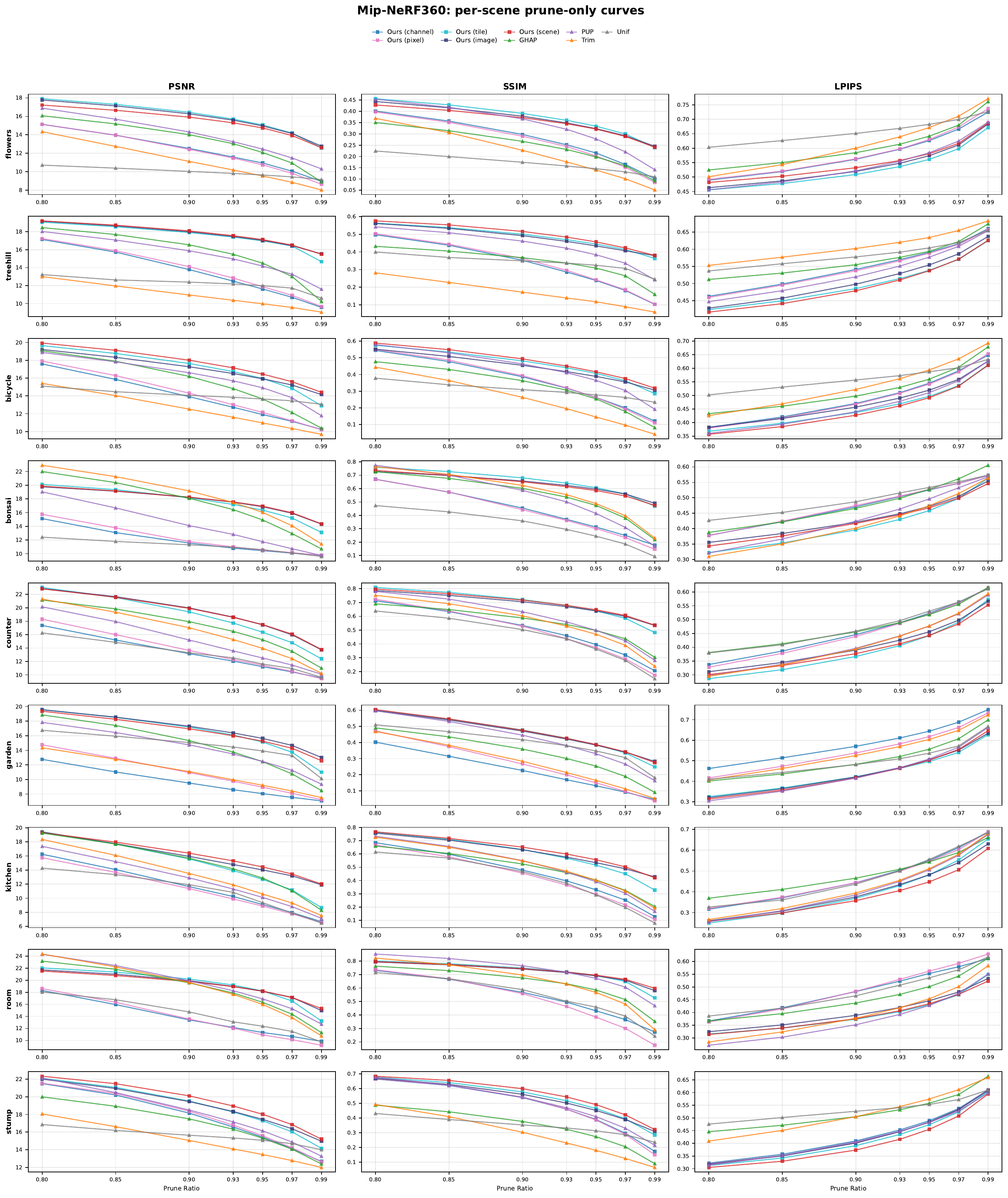}
    \caption{Per-scene prune-only curves for mip-NeRF 360.}
    \label{fig:appendix_per_scene_prune_only_mipnerf360}
\end{figure}

\FloatBarrier

\subsection{Sensitivity ablations and wider recovery curves}
The following figures report the complete ablations of our method and the wider short-recovery comparison. Together they test whether the relative ordering of the aggregation levels persists after limited optimization and across all three datasets.

The following figures report the full ablations of our method, including per-channel, per-pixel, per-tile, per-image, and per-scene aggregation levels, as well as implementation-specific L1/L2 and color/nocolor variants where needed for reproducibility. The final three figures give the wider short-recovery comparison across prune ratios 0.90, 0.93, 0.95, 0.97, and 0.99.

\begin{figure}[!htbp]
    \centering
    \includegraphics[width=\textwidth]{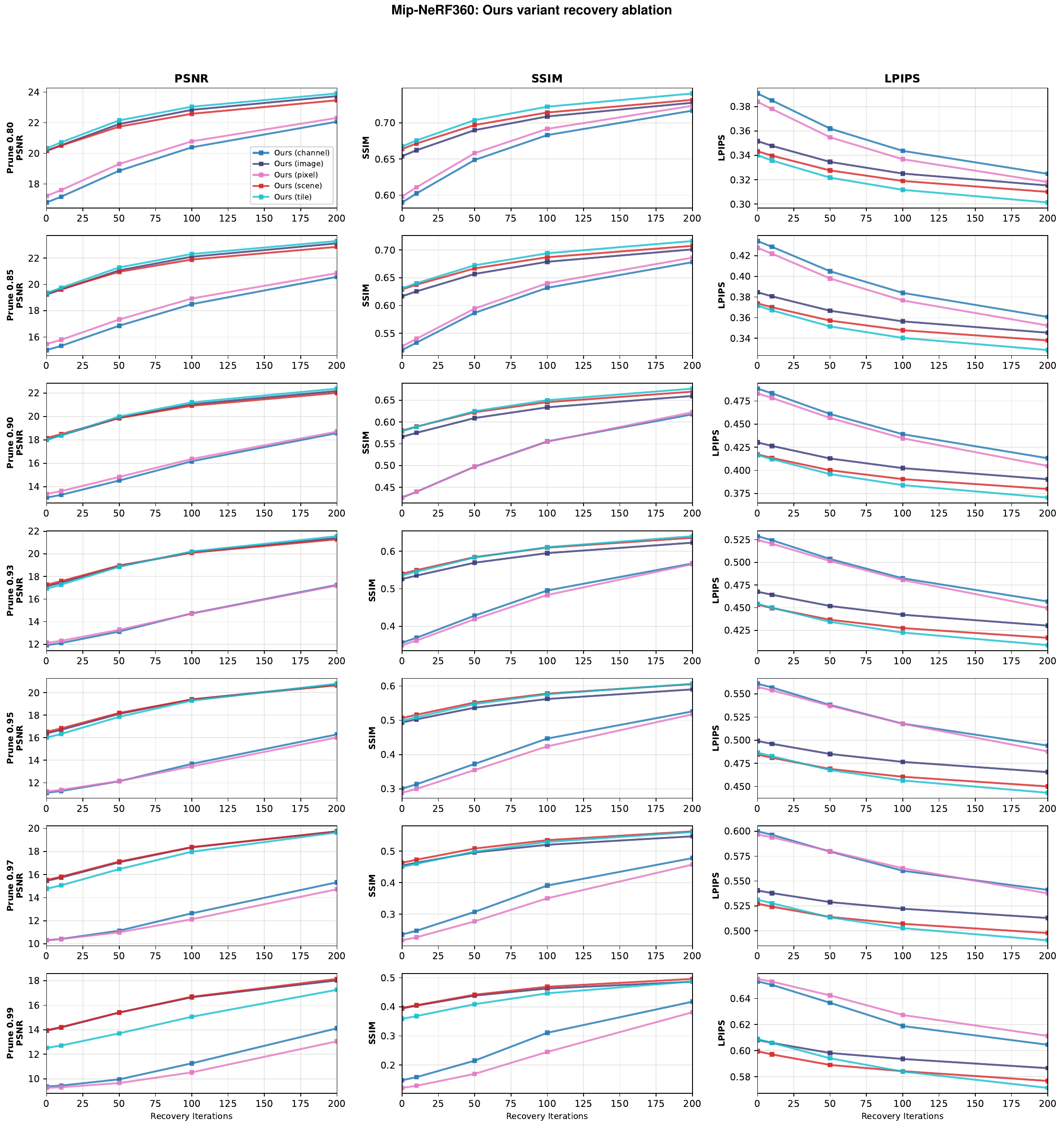}
    \caption{Recovery ablation across our variants for mip-NeRF 360.}
    \label{fig:appendix_ablations_mipnerf360}
\end{figure}

\begin{figure}[!htbp]
    \centering
    \includegraphics[width=\textwidth]{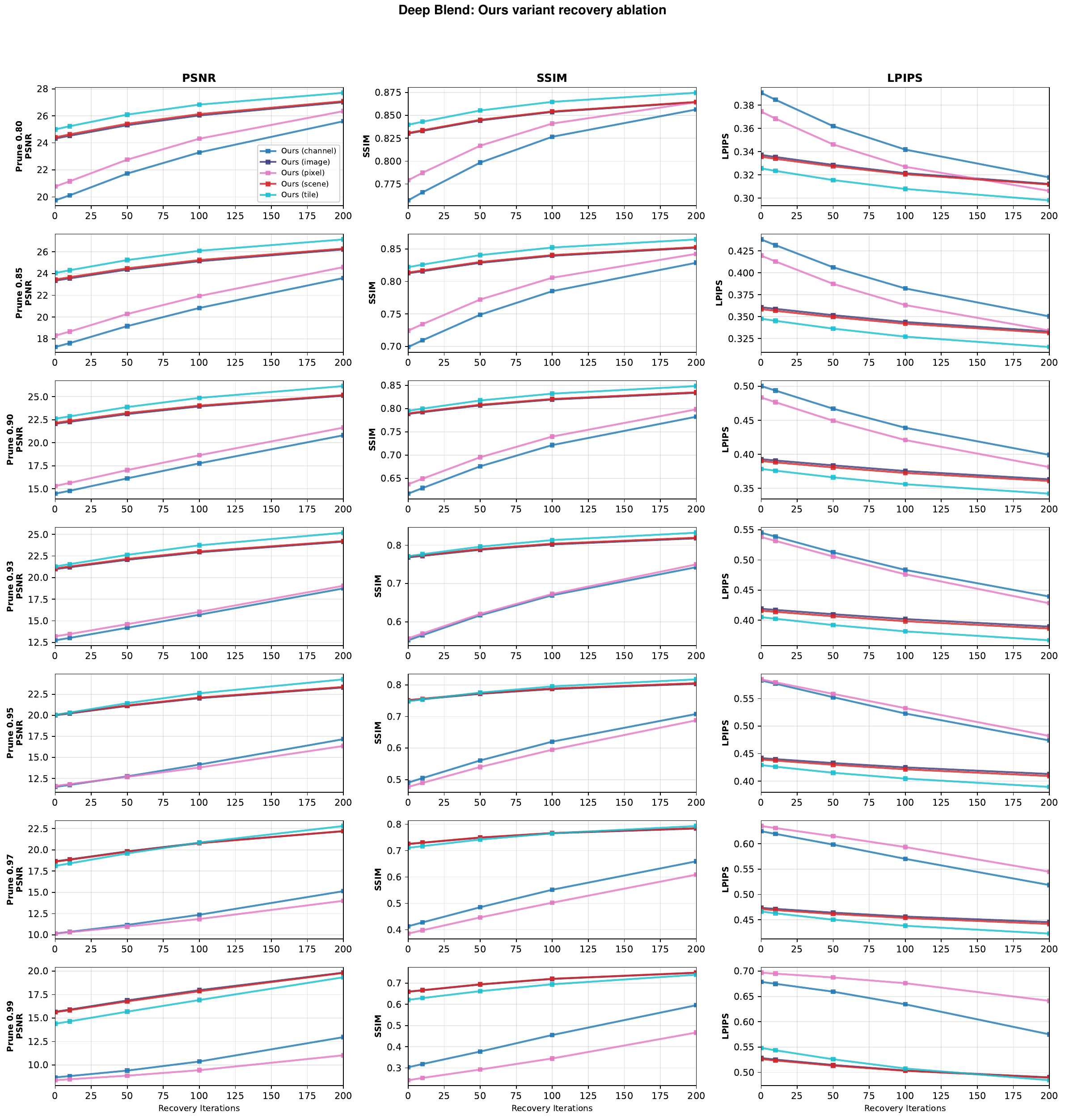}
    \caption{Recovery ablation across our variants for Deep Blending.}
    \label{fig:appendix_ablations_deep_blend}
\end{figure}

\begin{figure}[!htbp]
    \centering
    \includegraphics[width=\textwidth]{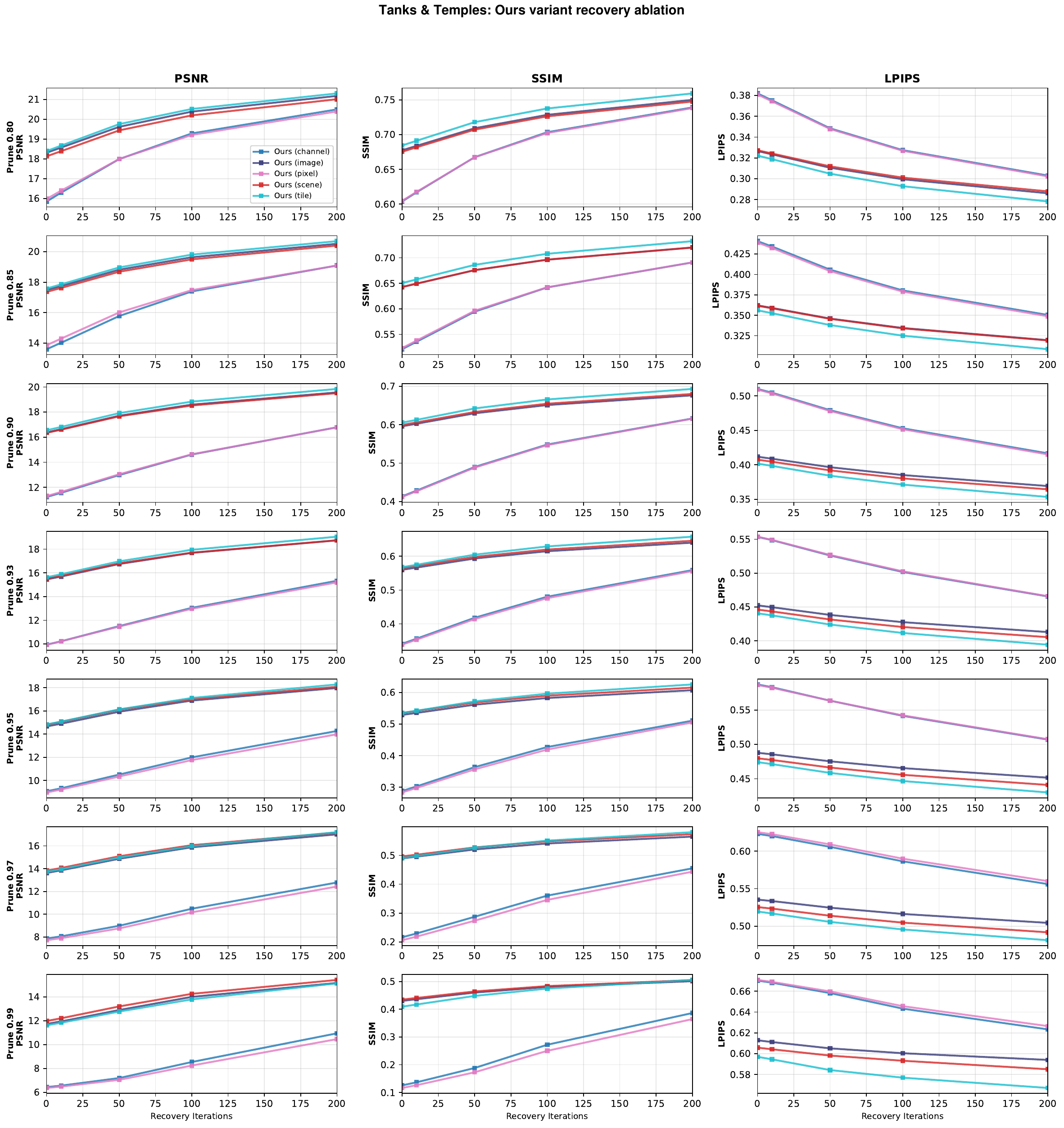}
    \caption{Recovery ablation across our variants for Tanks \& Temples.}
    \label{fig:appendix_ablations_tanks_temples}
\end{figure}

\begin{figure}[p]
    \centering
    \includegraphics[width=0.68\textwidth]{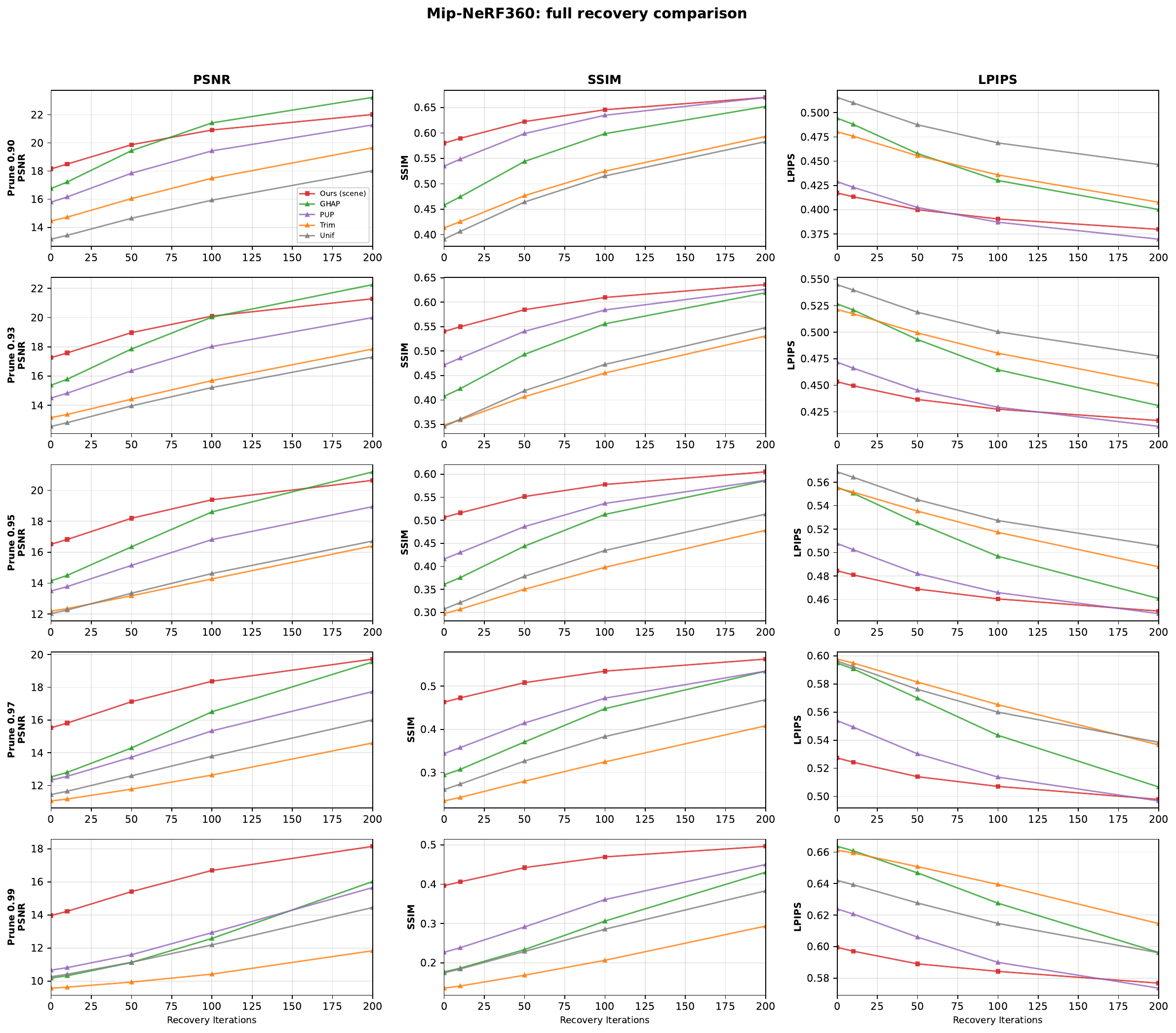}
    \caption{Full short-recovery comparison for mip-NeRF 360 across prune ratios 0.90--0.99.}
    \label{fig:appendix_recovery_competitors_full_mipnerf360}

    \vspace{0.45em}
    \includegraphics[width=0.68\textwidth]{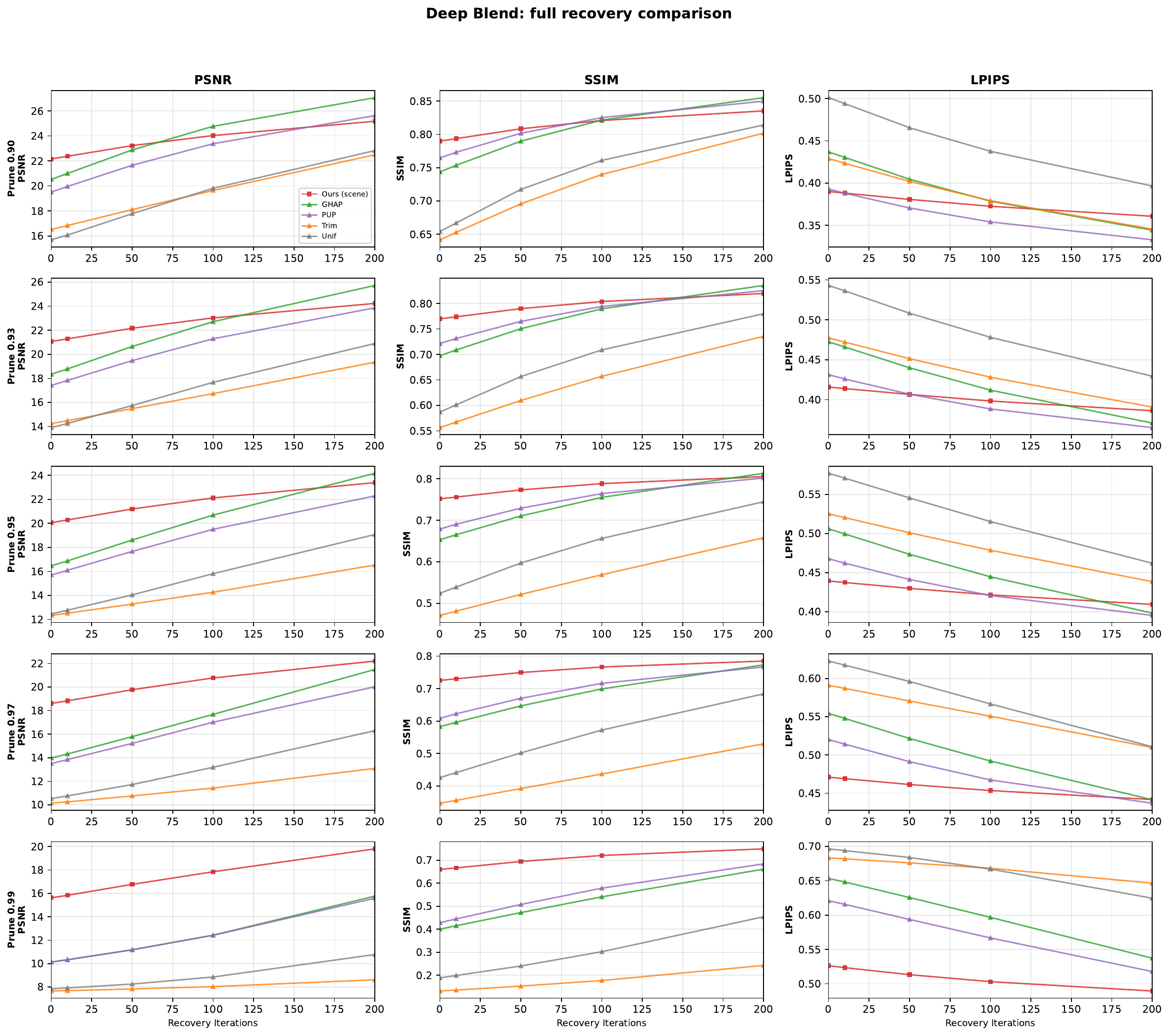}
    \caption{Full short-recovery comparison for Deep Blending across prune ratios 0.90--0.99.}
    \label{fig:appendix_recovery_competitors_full_deep_blend}
\end{figure}

\begin{figure}[p]
    \centering
    \includegraphics[width=0.68\textwidth]{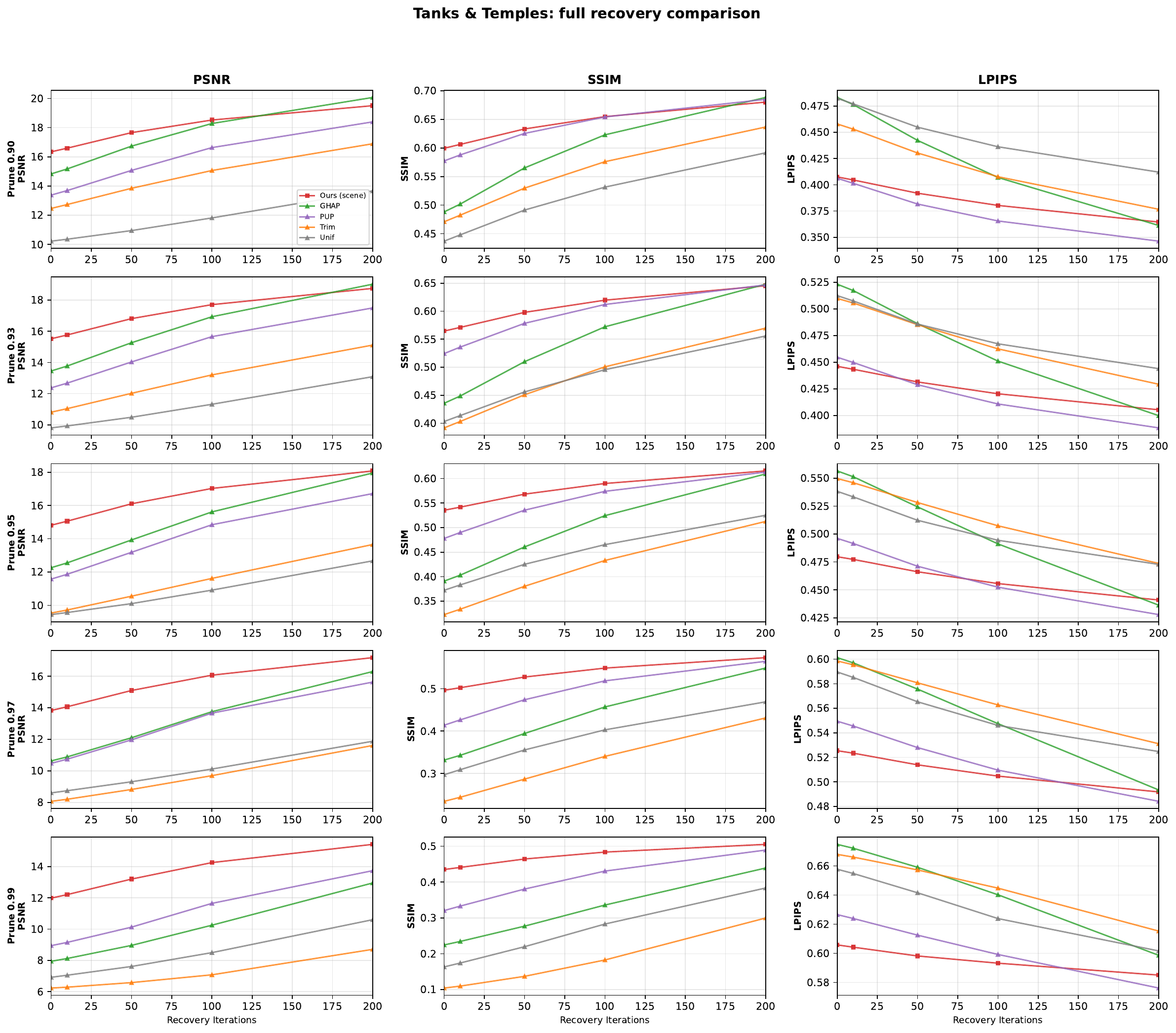}
    \caption{Full short-recovery comparison for Tanks \& Temples across prune ratios 0.90--0.99.}
    \label{fig:appendix_recovery_competitors_full_tanks_temples}

    \vspace{0.65em}
    \includegraphics[width=0.76\textwidth]{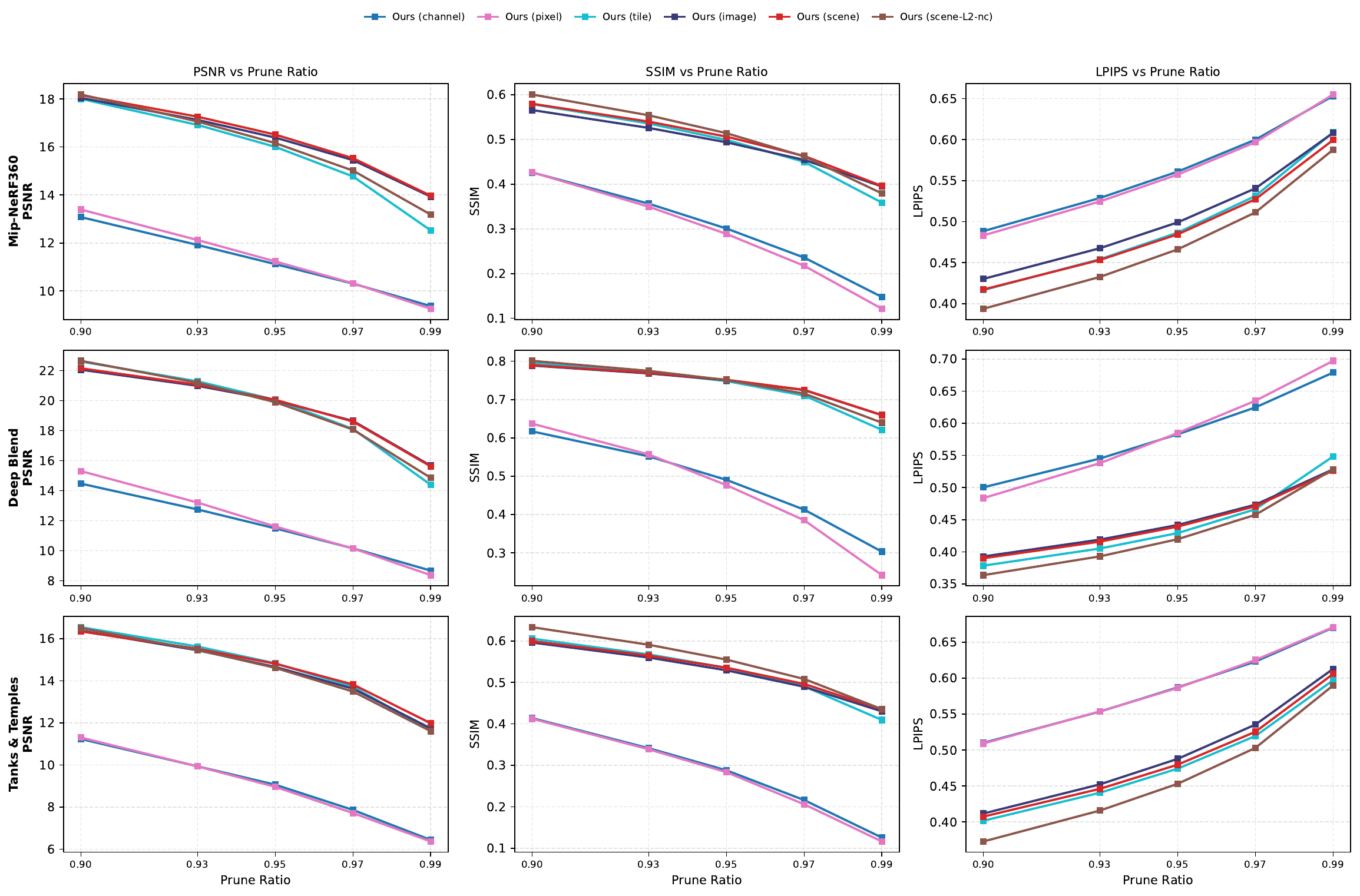}
    \caption{Ablation study.}
    \label{fig:appendix_prune_variants_all_datasets}
\end{figure}
\FloatBarrier

\begin{figure}[!htbp]
    \centering
    \includegraphics[height=0.82\textheight,keepaspectratio]{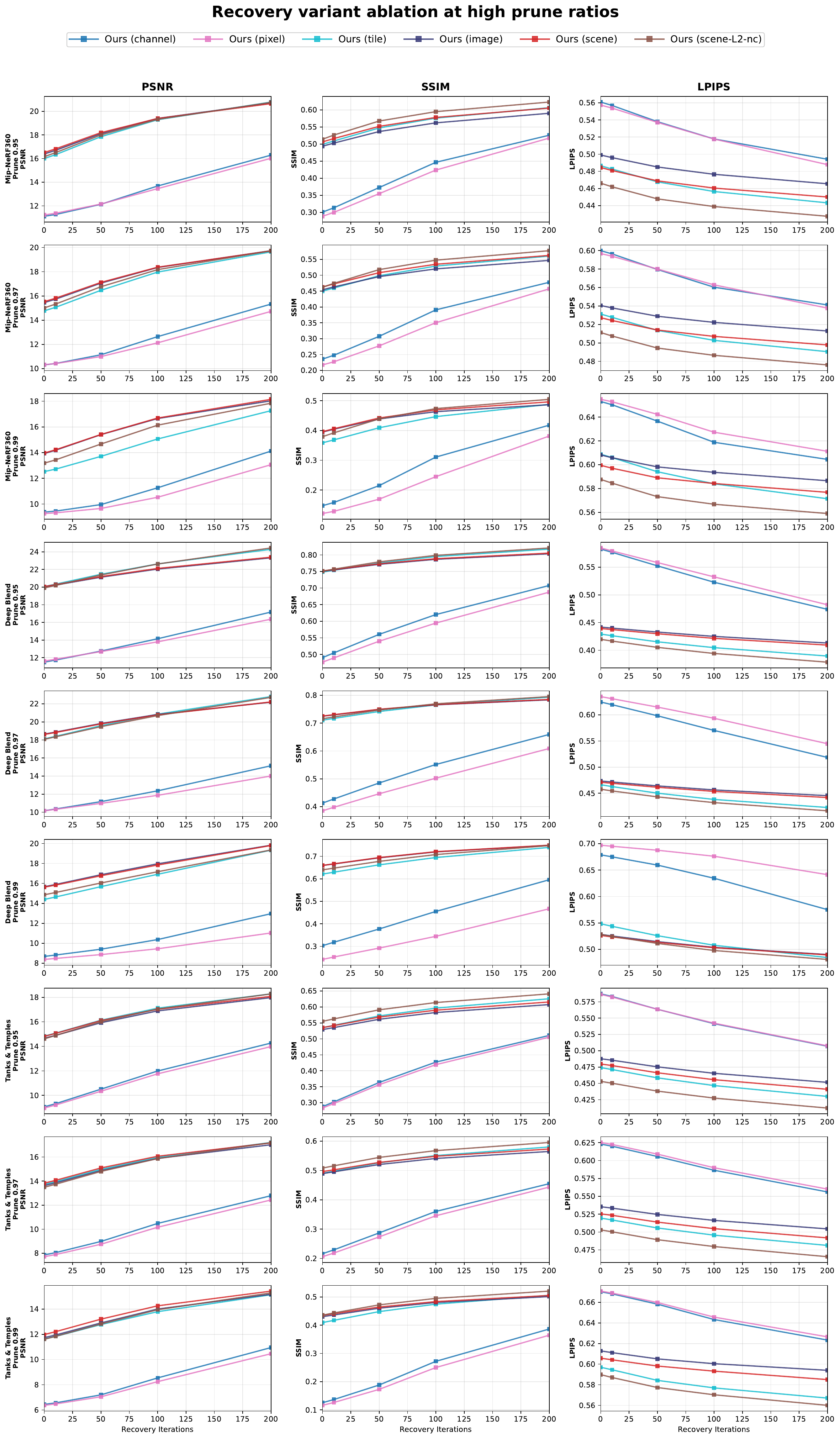}
    \caption{Ablation study.}
    \label{fig:appendix_recovery_variants_all_datasets}
\end{figure}
\FloatBarrier

\section{Theoretical Details and Complete Proofs}
This section gives the complete formal development in the same logical order as the main paper: setup and notation, unrestricted impossibility, the finite-query fixed-objective theorem, transfer to true rendering, and extension from representatives to compact query regions.

\subsection{Full setup and notation}\label{sec:appendix-setup}

For readability, we restate the two core definitions used throughout the appendix.

\begin{definition}[Multiplicative $\eps$-coreset for 3DGS, restated]\label{def:eps-coreset-app}
Let $N>0$, let $G=\{g_1,\dots,g_N\}$ be a 3DGS scene, and let $\mathcal U\subseteq\cQ$ be a family of rendering queries. Let $w\in[0,\infty)^N$ represent a weighted pruned scene. Recall $A$, the nonnegative scalar objective of the full scene, and let $A_w$ be the corresponding objective induced by $w$. We call $w$ an \emph{$\eps$-coreset} for $G$ on $\mathcal U$ if
\[
\Abs{A(G,q)-A_w(G,q)}\le \eps A(G,q)
\qquad\text{for every }q\in\mathcal U.
\]
\end{definition}

\begin{definition}[Weighted rendering validity, restated]\label{def:valid-app}
A weight vector $w\in[0,\infty)^N$ is called \emph{valid} on $\cQf$ if
\[
0\le w_i\rho(g_i,q)<1
\qquad\text{for every }q\in\cQf\text{ and every }i\in[N].
\]
\end{definition}

\begin{remark}
The theoretical objective $A_w^{\mathrm{th}}(G,q)$ is a classical coreset-style weighted sum of fixed itemwise terms $a(G,g_i,q)$. The true rendered objective $A_w(G,q)$ is the physically meaningful quantity of interest but is coupled through the reduced-scene transmittances $T_w(G,g_i,q)$.
\end{remark}

\subsection{General-case impossibility: restated theorem and proof}
\label{sec:appendix-no-general}

We now give the full construction behind \cref{thm:no-general-main}. The
construction uses the finite projected footprints used by rasterized 3DGS
implementations. Mathematically, an untruncated Gaussian has infinite support,
so exact footprint isolation is not available for the ideal infinite-support
kernel. In practical splatting, however, each projected Gaussian is evaluated
only on a bounded screen-space footprint, such as a constant-radius
Mahalanobis ellipse, commonly described as a $3\sigma$ footprint. The
following proof is stated for this truncated-footprint rendered objective.

\begin{theorem}[General-case impossibility, restated]
\label{thm:no-general-app}
For every $N\ge 1$ and every $0<\eps<1$, there exist a 3DGS scene $G=\{g_1,\dots,g_N\}$ and a query family $\cQ$ such that no weighted reduced scene specified $w\in[0,\infty)^N$ with $\|w\|_0<N$ is an $\eps$-coreset for the true rendered objective $A_w$ on $\cQ$.
\end{theorem}

\begin{proof}
Fix $N\ge 1$ and $0<\eps<1$. We construct a scene and a finite query family
for which each query isolates exactly one Gaussian.

Fix one image-plane location $x_0$ and one color channel $\lambda\in\Lambda$.
Choose $N$ camera poses $C_1,\dots,C_N$
and construct $N$ Gaussians $g_1,\dots,g_N$ such that, for every
$i\in[N]$, under camera pose $C_i$, the truncated projected footprint of
$g_i$ contains $x_0$, while the truncated projected footprint of every
$g_j$ with $j\neq i$ does not contain $x_0$. This is possible in the
unrestricted setting: place the Gaussian centers far apart in 3D, choose
camera pose $C_i$ to view $g_i$ at the selected pixel $x_0$, and choose the
projected covariances sufficiently small so that all other Gaussians lie
outside the finite rasterized footprint at that pixel.

For each Gaussian $g_i=(\mu_i,\Sigma_i,\alpha_i,c_i)$, choose
$\Sigma_i$ positive definite and choose $\alpha_i\in(0,1)$. Decrease the
opacities if necessary so that
\[
0\le \rho(g_j,q_i)<1
\qquad
\text{for every } i,j\in[N].
\]
This is possible because the query family constructed below is finite.
Choose the scalar appearance in channel $\lambda$ so that the isolated
Gaussian has positive color contribution at its corresponding query.

Define the query family
\[
\cQ:=\{q_1,\dots,q_N\},
\qquad
q_i:=(C_i,x_0,\lambda).
\]
By construction of the truncated projected footprints, for every
$i\in[N]$,
\[
k(g_i,q_i)>0,
\]
whereas for every $j\neq i$,
\[
k(g_j,q_i)=0.
\]
Therefore
\[
\rho(g_i,q_i)=\alpha_i k(g_i,q_i)>0,
\qquad
\rho(g_j,q_i)=0
\quad\text{for every }j\neq i.
\]

Since $0\le \rho(g_j,q_i)<1$ for all $i,j$, every prefix transmittance
factor $1-\rho(g_j,q_i)$ is positive. In particular,
\[
T(G,g_i,q_i)>0.
\]
Thus, at query $q_i$, Gaussian $g_i$ is the unique positive contributor:
\[
a(G,g_i,q_i)
=
c_i(q_i)\rho(g_i,q_i)T(G,g_i,q_i)>0,
\]
while for every $j\neq i$,
\[
a(G,g_j,q_i)
=
c_j(q_i)\rho(g_j,q_i)T(G,g_j,q_i)=0,
\]
because $\rho(g_j,q_i)=0$. Consequently,
\[
A(G,q_i)
=
\sum_{j=1}^N a(G,g_j,q_i)
=
a(G,g_i,q_i)
>
0.
\]

Now let $w\in[0,\infty)^N$ be any weighted pruned scene with
$\|w\|_0<N$. Then there exists an index $i\in[N]$ such that $w_i=0$.
Consider the corresponding query $q_i$. For every $j\neq i$, Gaussian
$g_j$ has zero truncated-footprint contribution at $q_i$. The only Gaussian
with positive full-scene contribution at $q_i$ is $g_i$, but its reduced
weight is zero. Therefore
\[
A_w(G,q_i)=0.
\]
On the other hand,
\[
A(G,q_i)>0.
\]
Hence
\[
\Abs{A(G,q_i)-A_w(G,q_i)}
=
A(G,q_i)
>
\eps A(G,q_i),
\]
because $0<\eps<1$. Thus $w$ fails the multiplicative coreset guarantee at
query $q_i$.

Since the argument holds for every $w$ with $\|w\|_0<N$, no weighted
reduced scene supported on fewer than $N$ Gaussians is an $\eps$-coreset for
this scene and the query family $\cQ$.
\end{proof}

\subsection{Full proof of the finite-query fixed-objective theorem}\label{sec:appendix-finite-proof}

\begin{lemma}[Multiplicative Chernoff for bounded variables]\label{lem:chernoff-app}
Let $Z_1,\dots,Z_m$ be i.i.d. random variables with $0\le Z_t\le 1$ and $\mathbb E[Z_t]=\mu$. Then for every $0<\eps\le 1$,
\[
\Pr\!\left[\left|\frac1m\sum_{t=1}^m Z_t-\mu\right|>\eps\mu\right]\le 2\exp\!\left(-\frac{\eps^2\mu m}{3}\right).
\]
\end{lemma}

\begin{theorem}[Query size dependent 3DGS coreset, restated]\label{thm:finite-th-app}
Let $G$ be a 3DGS scene, and let $\cQf\subseteq\cQ$ be a finite set of representative queries after discarding queries with $A(G,q)=0$. Let $\eps_c,\delta\in(0,1)$, and let $w$ be the random weight vector produced by the sensitivity-based sampling rule of \cref{sec:finite-th-main}. If
\[
m\ge \frac{3S}{\eps_c^2}\log\!\left(\frac{2|\cQf|}{\delta}\right),
\]
then, with probability at least $1-\delta$, $w$ is an $\eps_c$-coreset for the fixed-objective relaxation $A^{\mathrm{th}}$ on $\cQf$, i.e.,
\[
\Abs{A(G,q)-A_w^{\mathrm{th}}(G,q)}\le \eps_c A(G,q)
\qquad\text{for all }q\in\cQf.
\]
Moreover, $\|w\|_0\le m$.
\end{theorem}

\begin{proof}
Let $I_+:=\{i\in[N]: s(g_i)>0\}$ and, for $i\in I_+$, let $p_i:=s(g_i)/S$. Sample indices $J_1,\dots,J_m$ independently from $I_+$ with $\Pr[J_t=i]=p_i$. Let $n_i$ be the number of times index $i\in I_+$ is sampled, and set
\[
w_i:=
\begin{cases}
\dfrac{n_i}{mp_i}, & i\in I_+,\\[0.8ex]
0, & i\notin I_+.
\end{cases}
\]
Fix one query $q\in\cQf$. Since $A(G,q)>0$ by construction of $\cQf$, define
\[
Y_t(q):=\frac{a(G,g_{J_t},q)}{p_{J_t}A(G,q)}.
\]
Then
\[
\mathbb E[Y_t(q)] = \sum_{i\in I_+} p_i\frac{a(G,g_i,q)}{p_iA(G,q)}=\frac{1}{A(G,q)}\sum_{i\in I_+} a(G,g_i,q)=1,
\]
because $a(G,g_i,q)=0$ whenever $s(g_i)=0$. By the definition of sensitivity,
\[
\frac{a(G,g_i,q)}{A(G,q)}\le s(g_i) \qquad\Rightarrow\qquad Y_t(q)\le \frac{s(g_i)}{p_i}=S\qquad\text{for every } i\in I_+.
\]
Hence $0\le Y_t(q)/S\le 1$ and $\mathbb E[Y_t(q)/S]=1/S$. Applying \cref{lem:chernoff-app} to $Z_t(q):=Y_t(q)/S$ yields
\[
\Pr\!\left[\left|\frac1m\sum_{t=1}^m Y_t(q)-1\right|>\eps_c\right]\le 2\exp\!\left(-\frac{\eps_c^2m}{3S}\right).
\]
But
\[
\frac1m\sum_{t=1}^m Y_t(q)=\frac1{A(G,q)}\sum_{i\in I_+}\frac{n_i}{mp_i}a(G,g_i,q)=\frac{A_w^{\mathrm{th}}(G,q)}{A(G,q)}.
\]
Therefore
\[
\Pr\!\left[\Abs{A_w^{\mathrm{th}}(G,q)-A(G,q)}>\eps_cA(G,q)\right]\le 2\exp\!\left(-\frac{\eps_c^2m}{3S}\right).
\]
Choosing
\[
m\ge \frac{3S}{\eps_c^2}\log\!\left(\frac{2|\cQf|}{\delta}\right)
\]
makes the right-hand side at most $\delta/|\cQf|$. A union bound over all $q\in\cQf$ proves the simultaneous claim. The support bound $\|w\|_0\le m$ is immediate from the construction.
\end{proof}

\subsection{Transfer to true rendering under log-transmittance stability}\label{sec:appendix-transfer}

\begin{assumption}[Log-transmittance stability, restated]\label{ass:log-stability-app}
There exists $\gamma\ge 0$ such that for every $q\in\cQf$ and every $i\in[N]$ with $w_i>0$, one has
\[
\Abs{\log T(G,g_i,q)-\log T_w(G,g_i,q)}\le \gamma .
\]
\end{assumption}

\begin{lemma}[Exact transmittance-ratio consequence]\label{lem:ratio-from-log-app}
Under \cref{ass:log-stability-app}, for every $q\in\cQf$ and every $i\in[N]$ with $w_i>0$,
\[
e^{-\gamma}T(G,g_i,q)\le T_w(G,g_i,q)\le e^{\gamma}T(G,g_i,q).
\]
\end{lemma}

\begin{proof}
By \cref{ass:log-stability-app},
\[
-\gamma\le \log T_w(G,g_i,q)-\log T(G,g_i,q)\le \gamma .
\]
Exponentiating gives
\[
e^{-\gamma}\le \frac{T_w(G,g_i,q)}{T(G,g_i,q)}\le e^{\gamma}.
\]
Multiplying by $T(G,g_i,q)$ proves the claim.
\end{proof}

\begin{lemma}[Transfer from the theoretical objective to true rendering]\label{lem:transfer-app}
Assume $w$ is valid on $\cQf$ and satisfies \cref{ass:log-stability-app}. Then for every $q\in\cQf$,
\[
e^{-\gamma}A_w^{\mathrm{th}}(G,q)\le A_w(G,q)\le e^{\gamma}A_w^{\mathrm{th}}(G,q).
\]
\end{lemma}

\begin{proof}
By definition,
\[
A_w(G,q)=\sum_{g_i\in G} w_i c_i(q)\rho(g_i,q)T_w(G,g_i,q).
\]
Applying \cref{lem:ratio-from-log-app} termwise to indices with $w_i>0$ and using nonnegativity gives
\[
e^{-\gamma}w_ic_i(q)\rho(g_i,q)T(G,g_i,q)
\le
w_ic_i(q)\rho(g_i,q)T_w(G,g_i,q)
\le
e^{\gamma}w_ic_i(q)\rho(g_i,q)T(G,g_i,q).
\]
For indices with $w_i=0$, all terms are zero. Summing over $g_i\in G$ yields the result.
\end{proof}

\begin{theorem}[Finite-query rendered coreset under mild assumptions, restated]\label{thm:finite-rendered-app}
Let $w$ be the random weight vector produced by the sensitivity-based sampling rule of \cref{sec:finite-th-main}, and suppose that $w$ is valid on $\cQf$ and satisfies \cref{ass:log-stability-app}. Under the same sample-size condition as in \cref{thm:finite-th-app}, with probability at least $1-\delta$, $w$ is an $\eps_r$-coreset for the true rendered objective on $\cQf$, where
\[
\eps_r:=\max\Bigl\{1-e^{-\gamma}(1-\eps_c),\ e^{\gamma}(1+\eps_c)-1\Bigr\}.
\]
Equivalently,
\[
\Abs{A(G,q)-A_w(G,q)}\le \eps_r A(G,q)
\qquad\text{for all }q\in\cQf.
\]
Moreover, $\|w\|_0\le m$.
\end{theorem}

\begin{proof}
By \cref{thm:finite-th-app}, with probability at least $1-\delta$,
\[
(1-\eps_c)A(G,q)\le A_w^{\mathrm{th}}(G,q)\le (1+\eps_c)A(G,q)
\qquad\text{for every }q\in\cQf.
\]
On the same event, \cref{lem:transfer-app} gives
\[
e^{-\gamma}A_w^{\mathrm{th}}(G,q)\le A_w(G,q)\le e^{\gamma}A_w^{\mathrm{th}}(G,q).
\]
Combining them yields
\[
e^{-\gamma}(1-\eps_c)A(G,q)\le A_w(G,q)\le e^{\gamma}(1+\eps_c)A(G,q).
\]
Therefore
\[
A(G,q)-A_w(G,q)\le \bigl(1-e^{-\gamma}(1-\eps_c)\bigr)A(G,q),
\]
and
\[
A_w(G,q)-A(G,q)\le \bigl(e^{\gamma}(1+\eps_c)-1\bigr)A(G,q).
\]
Taking the maximum of the two one-sided bounds proves the stated $\eps_r$-coreset guarantee. The support bound follows from \cref{thm:finite-th-app}.
\end{proof}

\subsection{Compact-region extension from representatives: restated statements and proof}
\label{sec:appendix-region-proof}

\begin{assumption}[Compact nondegenerate query region, restated]\label{ass:region-app}
Let $\cQ_r\subseteq\cQ$ be a query region endowed with a metric $d_{\cQ}$. Assume:
\begin{enumerate}
    \item $\cQ_r$ is compact;
    \item the full rendered objective $q\mapsto A(G,q)$ is Lipschitz on $\cQ_r$ with constant $L_{\mathrm{full}}$;
    \item the reduced rendered objective $q\mapsto A_w(G,q)$ is Lipschitz on $\cQ_r$ with constant $L_w$;
    \item the full rendered objective is bounded away from zero:
    \[
    A(G,q)\ge A_{\min}>0
    \qquad\text{for every }q\in \cQ_r.
    \]
\end{enumerate}
\end{assumption}

\begin{proposition}[Why the Lipschitz assumption is reasonable]\label{prop:smooth-app}
Suppose on a region $\cQ_r$: (i) the relevant set of Gaussians is finite; (ii) the front-to-back order does not jump discontinuously inside $\cQ_r$; (iii) the camera-to-ray map is smooth; (iv) each color function $q\mapsto c_i(q)$ is smooth; and (v) each projected kernel contribution $q\mapsto k(g_i,q)$ is smooth. Then the full rendered objective $q\mapsto A(G,q)$ and the reduced rendered objective $q\mapsto A_w(G,q)$ are smooth on $\cQ_r$ and hence Lipschitz on $\cQ_r$ when $\cQ_r$ is compact.
\end{proposition}

\begin{proof}
Each per-Gaussian contribution is a finite product and composition of smooth functions of $q$: a color factor, a projected-kernel factor, and a transmittance factor built from finitely many smooth attenuation terms. A finite sum of smooth functions is smooth. Every smooth function on a compact set is Lipschitz.
\end{proof}

\begin{definition}[$\rho$-net, restated]\label{def:rho-net-app}
A finite set $\cQf_r\subseteq \cQ_r$ is a $\rho$-net of $\cQ_r$ if for every $q\in \cQ_r$ there exists $q'\in\cQf_r$ such that
\[
d_{\cQ}(q,q')\le \rho .
\]
\end{definition}

\begin{theorem}[Compact-region extension from representatives, restated]\label{thm:region-app}
Assume \cref{ass:region-app}. Let $\cQf_r\subseteq \cQ_r$ be a finite $\rho$-net. Suppose that the reduced scene satisfies the multiplicative guarantee on the representative set:
\[
\Abs{A(G,q')-A_w(G,q')}
\le
\eps_0 A(G,q')
\qquad\text{for every }q'\in\cQf_r .
\]
Then for every $q\in \cQ_r$,
\[
\Abs{A(G,q)-A_w(G,q)}
\le
\left(
\eps_0+
\frac{\bigl((1+\eps_0)L_{\mathrm{full}}+L_w\bigr)\rho}{A_{\min}}
\right)A(G,q).
\]
\end{theorem}

\begin{proof}
Fix $q\in \cQ_r$ and choose $q'\in\cQf_r$ such that $d_{\cQ}(q,q')\le \rho$. By the triangle inequality,
\[
\Abs{A(G,q)-A_w(G,q)}
\le
\Abs{A(G,q)-A(G,q')}
+
\Abs{A(G,q')-A_w(G,q')}
+
\Abs{A_w(G,q')-A_w(G,q)}.
\]
By Lipschitz continuity,
\[
\Abs{A(G,q)-A(G,q')}\le L_{\mathrm{full}}\rho,
\qquad
\Abs{A_w(G,q')-A_w(G,q)}\le L_w\rho.
\]
By the representative-set guarantee,
\[
\Abs{A(G,q')-A_w(G,q')}\le \eps_0 A(G,q').
\]
Using Lipschitz continuity once more,
\[
A(G,q')\le A(G,q)+L_{\mathrm{full}}\rho.
\]
Substituting gives
\[
\begin{aligned}
\Abs{A(G,q)-A_w(G,q)}
&\le
L_{\mathrm{full}}\rho
+
\eps_0\bigl(A(G,q)+L_{\mathrm{full}}\rho\bigr)
+
L_w\rho \\
&=
\eps_0 A(G,q)
+
\bigl((1+\eps_0)L_{\mathrm{full}}+L_w\bigr)\rho .
\end{aligned}
\]
Finally, since $A(G,q)\ge A_{\min}$,
\[
\bigl((1+\eps_0)L_{\mathrm{full}}+L_w\bigr)\rho
\le
\frac{\bigl((1+\eps_0)L_{\mathrm{full}}+L_w\bigr)\rho}{A_{\min}}A(G,q).
\]
Combining the two displays proves the claim.
\end{proof}

\begin{corollary}[Regionwise target-error form]\label{cor:region-target-app}
Fix $\bar\eps>\eps_0$. If
\[
\rho
\le
\frac{(\bar\eps-\eps_0)A_{\min}}
{(1+\eps_0)L_{\mathrm{full}}+L_w},
\]
then
\[
\Abs{A(G,q)-A_w(G,q)}
\le
\bar\eps A(G,q)
\qquad\text{for every }q\in \cQ_r.
\]
\end{corollary}

\begin{proof}
Substitute the chosen upper bound on $\rho$ into \cref{thm:region-app}.
\end{proof}

\begin{corollary}[Representative-to-region theorem with explicit coreset parameters]\label{cor:rep-to-region-app}
Assume \cref{ass:region-app} on $\cQ_r$, let $\cQf_r\subseteq \cQ_r$ be a finite $\rho$-net, and assume the hypotheses of \cref{thm:finite-rendered-app} hold on $\cQf_r$ with parameters $(\eps_c,\gamma_r,\delta_r)$ and sensitivity sum $S_r$. If
\[
m_r\ge \frac{3S_r}{\eps_c^2}\log\frac{2|\cQf_r|}{\delta_r},
\]
then with probability at least $1-\delta_r$,
\[
\Abs{A(G,q)-A_w(G,q)}
\le
\left(
\eps_{0,r}
+
\frac{\bigl((1+\eps_{0,r})L_{\mathrm{full}}+L_w\bigr)\rho}{A_{\min}}
\right)A(G,q)
\qquad\text{for all }q\in \cQ_r,
\]
where
\[
\eps_{0,r}
:=
\max\Bigl\{
1-e^{-\gamma_r}(1-\eps_c),\
e^{\gamma_r}(1+\eps_c)-1
\Bigr\}.
\]
\end{corollary}

\begin{proof}
Apply \cref{thm:finite-rendered-app} on the representative set $\cQf_r$ and then apply \cref{thm:region-app} with $\eps_0=\eps_{0,r}$.
\end{proof}

\begin{theorem}[Finite union of compact regions]\label{thm:union-regions-app}
Let $\cQ_{\mathrm{reg}}=\bigcup_{r=1}^R \cQ_r$ be a finite union of compact query regions satisfying \cref{ass:region-app}, and for each $r$ let $\cQf_r\subseteq \cQ_r$ be a representative net. Suppose the hypotheses of \cref{cor:rep-to-region-app} hold on each region with failure probability $\delta_r$, and that
\[
\sum_{r=1}^R\delta_r\le \delta.
\]
Then with probability at least $1-\delta$, every regional conclusion holds simultaneously on all $\cQ_r$, hence on all of $\cQ_{\mathrm{reg}}$.
\end{theorem}

\begin{proof}
Apply \cref{cor:rep-to-region-app} on each region and then take a union bound over $r=1,\dots,R$.
\end{proof}

\section{Additional Discussion, Limitations, and Societal Impact}
The final section expands the scope discussion from the main paper and records the broader implications of making 3DGS models cheaper to store, transmit, and render.

\subsection{Limitations and scope: extended}
\label{sec:limitations:ext}
Our results are designed for the structured regime formalized in this paper. 
The guarantees are not meant to cover unrestricted continuous rendering, where we prove that non-trivial multiplicative coresets are impossible. 
Instead, our positive results apply to representative query families induced by a target rendering resolution, e.g., a grid of camera poses, rays, pixels, or tiles, and extend to compact query regions under Lipschitz regularity and sufficient query coverage. This matches the practical setting in which 3DGS models are evaluated on finite sets of desired views, or with finite jumps from camera pose to another when deployed on a robot. The rendered guarantee also requires validity and log-transmittance stability. These assumptions are not merely technical: they identify the regime in which a fixed-objective coreset approximation can be faithfully transferred to true re-rendering after pruning. They make explicit how much visibility and occlusion structure must be preserved for provable rendering control.

Empirically, our focus is the low-compute regime: prune-only evaluation and very short finetuning budgets. 
This is precisely where principled subset selection is most important, since there is little recovery compute available to repair a poor pruning decision. 
Methods that rely on longer recovery or specialized compression pipelines may become more competitive when large post-pruning optimization budgets are available. 
These directions are largely orthogonal to our contribution: our pruning rule can serve as a strong initialization before applying post-pruning refinement or other recovery procedures. 
Our goal is therefore not to replace all 3DGS compression techniques, but to provide a principled, vanilla-compatible pruning approach with strong performance when recovery compute is scarce.

\subsection{Societal impacts}
\label{sec:societal-impacts}

This work may have positive societal impact by making 3D Gaussian Splatting models more efficient to store, transmit, and render.
By reducing memory and compute requirements, provable pruning can help deploy high-quality 3D representations on resource-constrained devices, lower energy consumption, and make real-time 3D applications more accessible in areas such as robotics, telepresence, education, cultural heritage, and scientific visualization.
The theoretical guarantees may also improve reliability by making compression behavior more predictable than purely heuristic pruning.

At the same time, more efficient 3D scene representations may also have negative societal impacts.
Easier deployment of compact, high-quality 3D models could support misuse in surveillance, unauthorized reconstruction of private spaces, or the creation and distribution of realistic synthetic or misleading 3D content.
Compression may also make such models easier to share at scale, increasing the risk of misuse.
These risks are not introduced uniquely by our method, but our work could lower the computational barriers for both beneficial and harmful applications.
Therefore, deployment should respect privacy, consent, dataset licensing, and application-specific safety constraints.
\subsection{Broader perspective}

The theory developed here suggests that 3DGS coresets are best understood not as a blanket theorem over arbitrary continuous query spaces, but as a layered statement:
\begin{enumerate}
    \item impossibility in the unrestricted general case;
    \item existence on a finite representative family for a fixed objective;
    \item transfer to true rendering under explicit transmittance stability; and
    \item extension to compact regions by representative covers.
\end{enumerate}
This formulation is intentionally conservative, but it makes the logical structure of the proof transparent and leaves no hidden gap between the negative and positive parts of the claim.

\clearpage

\end{document}